\documentclass[twoside,11pt]{article}

\usepackage{blindtext}

\usepackage[abbrvbib, preprint]{jmlr2e}

\usepackage{jmlr2e}

\usepackage{lastpage}
\usepackage[utf8]{inputenc} % allow utf-8 input
\usepackage[T1]{fontenc}    % use 8-bit T1 fonts
\usepackage{hyperref}       % hyperlinks
\usepackage{url}            % simple URL typesetting
\usepackage{booktabs}       % professional-quality tables
\usepackage{amsfonts}       % blackboard math symbols
\usepackage{nicefrac}       % compact symbols for 1/2, etc.
\usepackage{microtype}      % microtypography
\usepackage{color}         % colors
\usepackage{graphicx}
\usepackage{subfigure}
\usepackage{amssymb}
\usepackage{enumerate}
\usepackage{graphicx}
\usepackage{latexsym}
\usepackage{mathrsfs}
\usepackage{mathtools}
\usepackage{empheq}
\usepackage{cases}
\usepackage[]{xcolor}
\usepackage{verbatim}
\usepackage{bm}
\usepackage{caption}
\usepackage{authblk}
\newtheorem{assumption}[theorem]{Assumption}
\def\bR{{\mathbb{R}}}
\def\bN{{\mathbb{N}}}
\def\cC{{\mathcal{C}}}
\def\cG{{\mathcal{G}}}
\def\cF{{\mathcal{F}}}
\def\cP{{\mathcal{P}}}
\def\cM{{\mathcal{M}}}
\def\cK{{\mathcal{K}}}
\def\hmu{\widehat{\mu}}
\def\hpsi{\widehat{\psi}}
\def\hvphi{{\widehat{\varphi}}}
\newcommand{\exR}{\overline{\bR}}
\newcommand{\KL}{\mathrm{KL}}

\newcommand{\JS}{{\mathrm{JS}}}

\newcommand{\Lip}{{\mathrm{Lip}}}
\def\IPM{{\mathrm{IPM}}}
\newcommand{\abs}[1]{\left|#1\right|}
\def\gatSum{{\widetilde{C}^N_\cK}}
\def\gatSumvar{{C^N_\cK}}
\def\Spsi{S_{1}}
\def\Smu{S^{\prime}}
\def\Svar{S_{2}}
\def\Sf{S^{\prime \prime}}

\def\normpsi{\cC(X), 1}
\def\normmu{\cM(X)}
\def\normvar{\cC(X), 2}

\def\Spsic{S_{1,c}}
\def\Svarc{S_{2,c}}
\def\Sfc{S^{\prime \prime}_{c}}

\newcommand{\Add}[1]{\textcolor{red}{#1}}

\newcommand{\Addsue}[1]{\textcolor{magenta}{#1}}
\newcommand{\Erase}{\bgroup\markoverwith{\textcolor{red}{\rule[.5ex]{2pt}{0.4pt}}}\ULon}
\newcommand{\Erasefur}{\bgroup\markoverwith{\textcolor{red}{\rule[.5ex]{2pt}{0.4pt}}}\ULon}
\newcommand{\Eraseoku}{\bgroup\markoverwith{\textcolor{orange}{\rule[.5ex]{2pt}{0.4pt}}}\ULon}
\newcommand{\Erasesaw}{\bgroup\markoverwith{\textcolor{blue}{\rule[.5ex]{2pt}{0.4pt}}}\ULon}
\newcommand{\Erasesue}{\bgroup\markoverwith{\textcolor{magenta}{\rule[.5ex]{2pt}{0.4pt}}}\ULon}

\begin{document}

\title{Convergences for Minimax Optimization Problems over Infinite-Dimensional Spaces Towards Stability in Adversarial Training}

\author[1,a, *]{\rm Takashi Furuya}
\author[2,b, *]{\rm Satoshi Okuda}
\author[3,c]{\rm Kazuma Suetake}
\author[2,d]{\rm Yoshihide Sawada}

\affil[1]{{\small Shimane University}}
\affil[a]{{\small Email: takashi.furuya0101@gmail.com}\vspace{2mm}}

\affil[2]{{\small Tokyo Research Center, Aisin Corporation}}
\affil[b]{{\small Email: satoshi.okuda@aisin.co.jp}\vspace{2mm}}

\affil[3]{{\small AISIN SOFTWARE, Japan}}
\affil[c]{{\small Email: kazuma.suetake@aisin-software.com}\vspace{2mm}}

\affil[d]{{\small Email: yoshihide.sawada@aisin.co.jp}\vspace{2mm}}

\affil[*]{{\footnotesize These two authors contributed equally to this work}}

\editor{}

\maketitle

\begin{abstract}
Training neural networks that require adversarial optimization, such as generative adversarial networks (GANs) and unsupervised domain adaptations (UDAs), suffers from instability. This instability problem comes from the difficulty of the minimax optimization, and there have been various  approaches in GANs and UDAs to overcome this problem. In this study, we tackle this problem theoretically through a functional analysis. Specifically, we show the convergence property of the minimax problem by the gradient descent over the infinite-dimensional spaces of continuous functions and probability measures under certain conditions.
Using this setting, we can discuss GANs and UDAs comprehensively, which have been studied independently.
In addition, we show that the conditions necessary for the convergence property are interpreted as stabilization techniques of adversarial training such as the spectral normalization and the gradient penalty.
\end{abstract}


% 主要
% Optimization 

% キーワード
% Minimax, Non-convex Optimization, Convergence Analysis,  Adversarial Training, Functional Analysis

% TLDR (up to 250 characters)
% We provide the convergence theorem for the convex-concave minimax problem in spaces of measures and continuous functions under the relative smoothness. We also show that our setting allows for certain GANs and unsupervised domain adaptations.

%% 5/17以降の変更点%%

% Analysis Section S_{C}, S_{C} から S_{C}, S_{C}^{\prime} へ

% Joint convex, Proposition, (ii), input-output で C(X)_1-C(X)_2 違うnorm

% 2023.09.15 discussion (Suetake, Furuya), (Suetake, Sawada)
% # Section
% 1. Introduction
% 2. Related Work
% 3. Method (Problem settings and minimax analysis)
% 4. Experiment (Verification----)
% 5. Conclusion
% 
% Verification----が机上実験に相当する。それと分かるように改題：
% * Verification -> Example---
% * Experiment はやりすぎ
% 
% Problem settings を抑え目に？ Dual form が効くのは analysis ではなく experiment.
% * Analysis の冒頭＋途中に導線を張る。
% 
% 投稿先
% * とりあえず AISTATS に向けて頑張ってみる
% * ページきつそう or reject なら journal へシフトかな
% * （AISTATS ダメだったとして COLT 行けんのかな？むずそうけど分からん？）
% 
% AISTATS 提出に向けて削るなら？
% * 方針：面白い部分を本文に残して、厳密さのための検証は app. 送り
% * Preliminary は AISTATS だろうが journal だろうが問答無用で app.
% 
% 感想：App. すごい量だなあ。

%%%%%%%%%%%%%%%%%%%%%%%%%%%%%%%%%%%%%%%%%%%%%%%%%%%%%%%%%%%%%%%%%
\section{Introduction}\label{Introduction}

With the increased computational resources and available data, neural networks (NNs) trained by adversarial training have emerged prominently in various fields.
An example is the application of generative adversarial networks (GANs) in generative tasks. GANs train the generator to captures the data distribution in an adversarial manner against the discriminator, which distinguishes between data generated by the generator and the dataset ~\citep{goodfellow2014generative}.
Another example is the utilization of adversarial training in unsupervised domain adaptations (UDAs) as generalization techniques. 
UDAs transfer knowledge from source domains to the target domain by extracting domain-invariant features against the domain critic that distinguish between data from source and target domains in an adversarial manner~\citep{ganin2015unsupervised}.
Despite the effectiveness of GANs and UDAs, both pose challenges as nonconvex-nonconcave minimax problems, leading to inherent instability~\citep{salimans2016improved}. This instability, though insufficiently explored theoretically, complicates the widespread deployment of these models and hinders their practical application.
To address and pave the way for more robust applications, we analyze the instability problem from a functional analysis perspective.

As instability is related to the convergence properties of the gradient descent algorithm~\citep{chu2020smoothness}, we aim to clarify the convergence conditions for adversarial optimization problems.
To facilitate the derivation of these conditions from the functional analysis perspective, we begin by considering the ideal setting.
In our study, the ideal setting is derived from the dual formula of the minimization of a functional over probability distributions, leading to the \textit{minimax} problem over \textit{infinite-dimensional} spaces of continuous functions or probability measures. 
By exploring this minimax problem over infinite-dimensional spaces, we can prove the convergence to a minimax solution for a convex-concave setting~(Section~\ref{Convex-concave setting}) and a stationary point for a nonconvex-concave setting~(Section~\ref{Minimax analysis-nonconvex}) under appropriate assumptions.

Throughout the convergence analyses, we maintain the assumption that the discrepancy measure, appearing in both GANs and UDAs, is strongly convex and $L$-smooth for the convergence. 
Achieving strong convexity involves confining the discriminator to a suitable subset within Lipschitz continuous function spaces. 
This concept aligns with the spectral normalization \citep{miyato2018spectral}.
To ensure $L$-smoothness, we utilize the inf-convolution with a regularizer, such as the squared maximum mean discrepancy (MMD) in the reproducing kernel Hilbert spaces (RKHS) with the Gaussian kernel. This process corresponds to the gradient penalty \citep{gulrajani2017improved}.
Therefore, we can theoretically interpret widely-used stabilization techniques in adversarial training as the desired condition for achieving convergence properties.

\begin{comment}
{\color{red}
[Reviewer h133. The motivation for analyzing infinite-dimensional optimization problems could be
made clearer. Specifically, line 60 says that previous works argue that finite-dimensional GAN
training is theoretically unstable, while the infinite-dimensional case analyzed in this paper provably converges. What insights can this result provide to other theorists and practitioners, and how has previous infinite-dimensional analysis contributed?]}
% 無限次元のモチベを明確にせよ。
% この論文の目的: stability テクニックを minimax 収束条件として解釈する.
% (infinite dimensional) ideal setting (= convex dual of functional over probability measure) => global or local convergenceを示せる. 
\end{comment}
%%%%%%%%%%%%%%%%%%%%%%%%%%%%%%%%%%%%%%%%%%%%%%%%%%%%%%%%%%%%%%%%%
%%%%%%%%%%%%%%%%%%%%%%%%%%%%%%%%%%%%%%%%%%%%%%%%%%%%%%%%%%%%%%%%%
%\subsection{Contributions}
{\bf Contributions}
\begin{itemize}
\item[(A)]
We show the convergence of the minimax solution for a convex-concave setting and the stationary point for a nonconvex-concave setting over infinite-dimensional spaces of continuous functions or probability measures. 
This analysis is motivated by adversarial training in the scheme of the gradient descent (Section~\ref{sec:Minimax analysis}).

% We show that the relative smoothness induces convergence properties of the minimax problem over function spaces in the scheme of the mirror descent (Section~\ref{Minimax analysis}).
%
%%%%%%%%%%%%%%%%%%%%%%%%%%%%%%%
\item[(B)] 
We verify %the conditions necessary 
the fulfillment of sufficient conditions for the convergence properties in certain GANs and UDAs settings (Section~\ref{Application}), providing a theoretical interpretation of existing techniques such as the spectral normalization and gradient penalty.
%
%%%%%%%%%%%%%%%%%%%%%%%%%%%%%%
\end{itemize}
%
% theoretical insights/interpretations/take-home messages by considering ifnite-dimensional minimax analysis
%

%
% organization : 
%
%This paper organizes as follows:
%In Section~\ref{Preliminary}, we prepare mathematical concepts, including G\^{a}teaux differentials, Bregman divergences, and relative smoothness, in order to define mirror descent scheme.
%In Section~\ref{Minimax analysis}, we analyze the convex-concave minimax problem for the general objective function over measure and continuous functions spaces under the relative smoothness assumption.
%Finally in Section~\ref{Application}, we show the general objective function discussed in Section~\ref{Minimax analysis} can be applied to many types of objective function objective functions used in GANs and UDAs.

\section{Related Work}

GAN training often exhibits an unstable trajectory, resulting in poor solutions~\citep{goodfellow2014generative,metz2016unrolled}.
To address this instability, various stabilization techniques have been proposed, including the Wasserstein GAN~\citep{arjovsky2017wasserstein}, gradient penalty~\citep{gulrajani2017improved}, and spectral normalization~\citep{miyato2018spectral}. 
The effectiveness of these techniques in stabilizing GAN training has been theoretically demonstrated~\citep{chu2020smoothness}.
%In particular, the sufficient conditions were identified to guarantee the stationarity of the generator of GANs. 
%In addition, 
This theoretical result implies that the instability of GANs is due to adversarial training. 
Thus, UDAs with adversarial training are expected to encounter similar instability during training. 
Notably, \citet{chu2020smoothness} provides theoretical insight into GANs, interpreting stabilization techniques as conditions from the perspective of minimization problem over finite-dimensional spaces. 
On the other hand, our work provides similar theoretical insight from the viewpoint of the {\it {minimax}} problem in the {\it{infinite}}-dimensional spaces. 
Considering the minimax problem offers a setting closer to adversarial training than the minimization problem, and analyzing infinite-dimensional spaces provides comprehensive framework for both GANs and UDAs settings.

%, which is closer to the actual setting of adversarial training than minimization, 
% sawa: 文の最後にfinite-dim. spacesを追加しました。流し読みすると、これがないとChuらが無限次元考えているけどGANしかやっていない、とも読み取れると思ったためです。間違っている、不要なら削除してください
%\Erasesaw{The consideration of minimax problem is closer setting in adversarial training rather than minimization problem, and the consideration of infinite-dimensional spaces enables the analysis for both GANs and UDAs settings, while \citet{chu2020smoothness} only discussed the GANs settings over finite-dimensional spaces.}
% On the other hand, it is not clear whether the setting of \citet{chu2020smoothness} over finite-dimensional spaces for GANs can provide a comprehensive discussion as in this paper.

Numerous references delve into the minimax optimization problem over finite-dimensional spaces, often treated as specific cases of Hilbert spaces. For instance, \citet{cherukuri2017saddle, mokhtari2020unified, du2019linear} explore the convex-concave setting, while \citet{huang2021efficient, thekumparampil2019efficient, lin2020gradient} focus on the nonconvex-concave setting.
Although the minimax problem over Hilbert spaces has received extensive attention, with works such as \citet{bauschke2017correction,boct2022accelerated,bot2023relaxed}, the exploration of the minimax problem over spaces of probability measures or continuous functions—distinct from Hilbert spaces—remains relatively limited. 
On the other hand, our work delves into the minimax optimization problem for infinite-dimensional spaces of probability measures or continuous functions.
% sawa: 以下の参考文献の論文は全てヒルベルト空間での話ですか？
% sawa: 今回こういう設定を考えることによる利点は何ですか？特に上の方で参照した論文と比べて。なんでこの設定で考えたのか、別に利点がないなら今までの設定でいいじゃん、と思います。
%In particular, our work can treat three variables: the discriminator, generator, and predictor, which allows for UDA/GAN-specific theoretical insights.

\begin{comment}
{\color{red}
[Reviewer tcXx. 
1.There are so many papers studying the convergence of GAN / min-max optimization, and
so many related to convex-concave optimization. Unfortunately, I don’t see much related work
discussion. While I’m not an expert in this specific topic, there is some literature that seems
relevant after a simple search [1,2,3]. Convex-concave min-max optimization is a very standard
problem. A detailed comparison with related work is needed
]
}
% Related work にminimax optimizationの先行研究比較を充実せよ. 

% 1. Chu et al との比較. 
% (Chu) minimization-finite dimensional の観点で, GANのフレームワークで, 安定テクニックを収束条件として解釈した.
% (Ours) minimax infinite dimensional の観点で, UDAを含めたフレームワークで, ...同上... 

% 2. many references との比較.
% (many) 多くのminimaxは finite dimensional case. また, Hilber spaces ではよくやられている.
% (Ours) probability measures or continuous functions spaces=> infinite dimensional and not Hilbert space. three variable setting => UDAを含めたフレームワークであり、やられていない
\end{comment}

%%%%%%%%%%%%%%%%%%%%%%%%%%%%%%%%%%%%%%
%%%%%%%%%%%%%%%%%%%%%%%%%%%%%%%%%%%%%%
\section{Preliminary}\label{Notation}
This section describes the mathematical tools required in this paper.
%article.

Let $\mathbb{N}_{0}$  be the set of natural numbers including zero, $X \subset \mathbb{R}^d$ be a compact set, and $\overline{\mathbb{R}}=\mathbb{R}\cup \{ - \infty, +\infty \}$ be the extended real number.
We denote by $\mathcal{M}(X)$, $\mathcal{M}^{+}(X)$, and $\cP(X)$ the set of all finite signed measures on $X$, the set of all non-negative finite measure on $X$, and the set of Borel probability measures on $X$, respectively.
Let $\mathcal{C}(X)$ be the set of all continuous functions $X \to \mathbb{R}$.
As shown in \citet[Section 5.14]{aliprantis2006infinite}, $\left<\mathcal{M}(X), \mathcal{C}(X)\right>$ is a dual pair equipped with the bilinear functional
\[
\left<\mu, \varphi \right>:= \int \varphi d\mu, \quad \mu \in \mathcal{M}(X), \ \varphi \in \mathcal{C}(X),
\]
and the topological dual of $\mathcal{M}(X)$ with respect to weak topology is $\mathcal{C}(X)$ 
\citep[Theorem 5.93]{aliprantis2006infinite}.
In the context of machine learning, %it is important to restrict $\mathcal{M}(X)$ to $\mathcal{P}(X)$, i.e., $\mathcal{P}(X) \subset \mathcal{M}(X)$.
we restrict $\mathcal{M}(X)$ to $\mathcal{P}(X)$.
As $X$ is a compact subset of $\mathbb{R}^d$, $\mathcal{P}(X)$ is a compact subset in $\mathcal{M}(X)$ (see e.g., \citet[Theorem 15.11]{aliprantis2006infinite}). 

Let $\left\| \cdot \right\|_{\mathcal{M}(X)}$ and $\left\| \cdot \right\|_{\mathcal{C}(X)}$ be norms induced by inner products in $\mathcal{M}(X)$ and $\mathcal{C}(X)$, respectively. 
Then, we first define the dual norms, convex conjugation, and strong convexity as follows:

\paragraph{Dual Norms}
%Let $\left\| \cdot \right\|_{\mathcal{M}(X)}$ and $\left\| \cdot \right\|_{\mathcal{C}(X)}$ be norms in $\mathcal{M}(X)$ and $\mathcal{C}(X)$, respectively. 
We denote dual norms $\left\| \cdot \right\|_{\mathcal{M}(X)}^{\star}$ and $\left\| \cdot \right\|_{\mathcal{C}(X)}^{\star}$ of $\left\| \cdot \right\|_{\mathcal{M}(X)}$ and $\left\| \cdot \right\|_{\mathcal{C}(X)}$ by, respectively, 
\begin{align}
\left\| \varphi \right\|_{\mathcal{M}(X)}^{\star}
&= \sup \left\{ \left| \int \varphi d\mu \right| \ : \ \left\| \mu \right\|_{\mathcal{M}(X)} \leq 1, \  \mu \in \mathcal{M}(X)\right\},
\quad
\varphi \in \mathcal{C}(X),\label{dual-norm-M}
\\
\left\| \mu \right\|_{\mathcal{C}(X)}^{\star}
&= \sup \left\{ \left| \int \varphi d\mu \right| \ : \ \left\| \varphi \right\|_{\mathcal{C}(X)} \leq 1, \  \varphi \in \mathcal{C}(X)\right\},
\quad
\mu \in \mathcal{M}(X).
\label{dual-norm-C}
\end{align}
%respectively.

%
\paragraph{Convex Conjugation}
%\Eraseoku{Let $\phi : \mathcal{M}(X) \to \overline{\mathbb{R}}$ and $\eta : \mathcal{C}(X) \to \overline{\mathbb{R}}$ be convex.
%We denote convex conjugates $\phi^{\star}$ and $\eta^{\star}$ of $\phi$ and $\eta$ by}
The convex conjugates $F^\star$ and $G^\star$ of each functionals $F:\cC(X)\to\exR$ and $G: \cM(X)\to\exR$ are defined by, respectively, 
\begin{align}
F^{\star}(\mu)
&=\sup
_{\varphi \in \mathcal{C}(X)}
\int \varphi d\mu - F(\varphi), 
\quad\mu \in \mathcal{M}(X),
\label{dual-eta}\\
G^{\star}(\varphi)
&= \sup_{\mu \in \mathcal{M}(X)} \int \varphi d\mu - G(\mu), 
\quad 
\varphi \in \mathcal{C}(X).
\label{dual-phi}
\end{align}
%respectively.

%In the following, norms $\left\| \cdot \right\|_{\mathcal{M}(X)}$ and $\left\| \cdot \right\|_{\mathcal{C}(X)}$ in $\mathcal{M}(X)$ and $\mathcal{C}(X)$, induced by inner products, respectively. 

%
\paragraph{Strong Convexity}
Let $S_{\cC} \subset \cC(X)$ and $S_{\cM} \subset \cM(X)$.
%Let $S_{\cC} \subset \cC(X)$ and $S_{\cM} \subset \cM(X)$ be subsets, and let be $\beta>0$.
We say that $F: \cC(X)\to\exR$ and $G: \cM(X)\to\exR$ are $\beta$-strongly convex~($\beta > 0$) with respect to $\left\| \cdot \right\|_{\mathcal{C}(X)}$ and $\left\| \cdot \right\|_{\mathcal{M}(X)}$ over $S_{\cC}$ and $S_{\cM}$, respectively, if it holds that
for any $\alpha\in[0, 1]$
\begin{align}
F(\alpha \psi + (1-\alpha) \varphi ) 
&\leq \alpha F(\psi) + (1-\alpha) F(\varphi) 
-\frac{\alpha (1-\alpha)\beta}{2} \left\| \psi - \varphi \right\|_{\mathcal{C}(X)}^{2}, \quad \psi, \varphi \in S_{\cC},
\label{SC-psi}\\
G(\alpha \mu + (1-\alpha) \nu ) 
&\leq \alpha G(\mu) + (1-\alpha) G(\nu) 
-\frac{\alpha (1-\alpha)\beta}{2} \left\| \mu - \nu \right\|_{\mathcal{M}(X)}^{2}, \quad \mu, \nu \in S_{\cM}.
\label{SC-mu}
\end{align}
%
%\paragraph{G\^ateaux Differentials}

Next, we review %the 
G\^{a}teaux differentials, Bregman divergences, and $L$-smoothness in order.

The G\^{a}teaux differential is a generalization of the concept of directional derivative in finite-dimensional differential calculus.
%We review the G\^{a}teaux differentials, which is a generalization of the concept of directional derivative in finite-dimensional differential calculus.
Let
$F : \mathcal{C}(X) \to \overline{\mathbb{R}}$ and
$G : \mathcal{M}(X) \to \overline{\mathbb{R}}$ 
, then %the
G\^{a}teaux differentials are defined as follows.
%%%%%%%%%%%%%%%%%%%%%%%%%%%%%%%%%%%%%%%%%%%%%%%%%%%%
\begin{definition}
%[G\^{a}teaux differentials over the spaces of measures and continuous functions]
\label{G\^{a}teaux differentials over measure and continuous function spaces}
We define G\^{a}teaux differentials 
$dF_{\varphi} : \mathcal{C}(X) \to \overline{\mathbb{R}}$ 
and 
$dG_{\mu} : \mathcal{M}(X) \to \overline{\mathbb{R}}$ of the functionals $F$ and $G$ 
at $\varphi \in \mathcal{C}(X)$ and $\mu \in \mathcal{M}(X)$ in the direction $\lambda \in \mathcal{C}(X)$ and $\chi \in \mathcal{M}(X)$ by, respectively,
\begin{align*}
    dF_{\varphi}(\lambda) &:= \lim_{\epsilon \to +0} \frac{F(\varphi+\epsilon \lambda) - F(\varphi)}{\epsilon},\\
    dG_{\mu}(\chi) &:= \lim_{\epsilon \to +0} \frac{G(\mu+\epsilon \chi) - G(\mu)}{\epsilon}.
\end{align*}
\end{definition}
%%%%%%%%%%%%%%%%%%%%%%%%%%%%%%%%%%%%%%%%%%%%%%%%%
We note that if $F$ and $G$ are proper convex functionals, then for $\varphi \in \mathcal{C}(X)$ and $\mu \in \mathcal{M}(X)$ there exist G\^{a}teaux differentials
$dF_{\varphi} : \mathcal{C}(X) \to \overline{\mathbb{R}}$ and
$dG_{\mu} : \mathcal{M}(X) \to \overline{\mathbb{R}}$, respectively~\citep[Lemma 7.14]{aliprantis2006infinite}.

Then, we review the Bregman divergences. The Bregman divergences over spaces of measures and continuous functions measure between two points defined in terms of convex functions.
%We review the Bregman divergence over spaces of measures and continuous functions, which is a measure between two points defined in terms of convex functions.
%%%%%%%%%%%%%%%%%%%%%%%%%%%%%%%%%%%%%%%%%%%%%%%%%%
\begin{definition}%[Bregman divergences over spaces of measures and continuous functions]
\label{Bregman divergences over measure and continuous function spaces}
Let
$F : \mathcal{C}(X) \to \overline{\mathbb{R}}$ and
$G : \mathcal{M}(X) \to \overline{\mathbb{R}}$ be proper, lower semi-continuous, and convex functionals.
Then, $F$-Bregman divergence $D_{F} :\mathcal{M}(X) \times \mathcal{M}(X) \to \mathbb{R}_{+}$ and $G$-Bregman divergence $D_{G} :\mathcal{C}(X) \times \mathcal{C}(X) \to \mathbb{R}_{+}$ are defined by, respectively,
\[
D_{F}(\nu | \mu):=F(\nu) - F(\mu) -dF_{\mu}(\nu-\mu), \ \mu, \nu \in \mathcal{M}(X),
\]
\[
D_{G}(\psi | \varphi):=G(\psi) - G(\varphi) -dG_{\varphi}(\psi-\varphi), \ \varphi, \psi \in \mathcal{C}(X).
\]
\end{definition}
%%%%%%%%%%%%%%%%%%%%%%%%%%%%%%%%%%%%%%%%%%%%%%%%%%
Finally, we review the $L$-smoothness. The $L$-smoothness over spaces of measures and continuous functions are defined using the Bregman divergence as follows.
%We also review the $L$-smoothness over spaces of measures and continuous functions, which are defined using the Bregman divergence.
%%%%%%%%%%%%%%%%%%%%%%%%%%%%%%%%%%%%%%%%%%%%%%%%%%
\begin{definition}
%[$L$-smoothness for spaces of measures and continuous functions]
\label{Relative smoothness for measure and continuous function spaces}
Let $S_{\cC} \subset \cC(X)$ and $S_{\cM} \subset \cM(X)$ be subsets, and $F : \mathcal{C}(X) \to \overline{\mathbb{R}}$ and $G : \mathcal{M}(X) \to \overline{\mathbb{R}}$ be proper, lower semi-continuous, and convex.
Then, we say that $F$ and $G$ are $L$-smooth ($L > 0$) with respect to $\left\| \cdot \right\|_{\mathcal{C}(X)}$ and $\left\| \cdot \right\|_{\mathcal{M}(X)}$ over $S_{\cC}$ and $S_{\cM}$ if it holds that, respectively,
\begin{align*}
D_{F}(\psi | \varphi) &\leq \frac{L}{2} \left\|\psi - \varphi\right\|^2_{\cC(X)}, \quad \varphi, \psi \in S_{\cC}, \\
D_{G}(\nu | \mu) &\leq \frac{L}{2} \left\|\nu - \mu \right\|^2_{\cM(X)}, \quad \mu, \nu \in S_{\cM}.
\end{align*}
\end{definition}
%%%%%%%%%%%%%%%%%%%%%%%%%%%%%%%%%%%%%%%%%%%%%%%%
\begin{comment}
If $\phi=\frac{1}{2}\left\| \cdot \right\|_{\mathcal{M}(X)}^{2}$ and $\eta=\frac{1}{2}\left\| \cdot \right\|_{\mathcal{C}(X)}^{2}$ that both are norms induced by inner products, then $D_{\phi}(\nu | \mu)=\frac{1}{2}\left\| \mu - \nu\right\|_{\mathcal{M}(X)}^{2}$ and $D_{\eta}(\psi | \varphi)=\frac{1}{2}\left\| \psi - \varphi \right\|_{\mathcal{C}(X)}^{2}$ (see \citet[Example 1]{aubin2022mirror}), that is the relative smoothness is a generalization of the smoothness based on norms induced by inner products.
\end{comment}
%
%%%%%%%%%%%%%%%%%%%%%%%%%%%%%%%%%%%%%%%%%%%%%%%%%%%%%%%%%%%%%%%%%%%%%%%%%%%%%%%%%%%%%%%%%%%%%%%%%%
\section{Problem Setting}
%Let us summarize our formulation.
This section describes the problem setup of GAN and UDA training, %First, 
building upon the reformulation introduced by \citet{chu2019probability} as the foundation for our theoretical framework.

In their work, \citet{chu2019probability} reformulated GAN training as a minimization problem with an objective function $J_{\nu_0}(\mu)$ over the set of probability measures, which represents a discrepancy measure between a generated distribution $\mu$ and an unknown true distribution $\nu_0$.
Moreover, the adversarial loss can be obtained through the Fenchel-Moreau theorem.
Consequently, they showed that various GAN models can be constructed by identifying particular discrepancy measures on an infinite dimensional space, such as the ordinal GAN~\citep{goodfellow2014generative}, maximal mean discrepancy (MMD) GAN~\citep{li2015generative}, $f$-GAN~\citep{nowozin2016f}, and Wasserstein GAN~\citep{arjovsky2017wasserstein}.
Building upon this formulation, we extend it to unsupervised domain adaptation by adversarial training.

\begin{comment}
{\color{red}
[Reviewer tcXx. 3.Moreover in sec 5, the author mainly discusses UDA while treating GAN as a special setting.
In that case, a bit more introduction to UDA is needed, and a separate section can only discuss
GAN as a special case. Currently all the discussion is mixed together.

4.For UDA formula, the author should explain right after introducing equation (7).
]
}

{\color{red}
[Reviewer h133. 
Line 173-181 are a little confusing to me. How do this work’s theoretical results relate
to the results of [1]? Does [1] provide insights into whether the assumptions made in this work
are satisfied?
]
}
% UDAの導入・説明をもうちょい追加ってこと？
\end{comment}

The UDA can be regarded as a simultaneous optimization problem for a source risk $R: \cC(X) \times \cP(X) \to \bR$ and a discrepancy measure $J_{\nu_0}(\mu): \cP(X) \to \bR$ between a source distribution $\mu$ and a fixed target distribution $\nu_0$.
Then, the optimization problem for the UDA can be expressed as:
\begin{equation}
    \min_{(\psi, \mu) \in \mathcal{C}(X) \times \mathcal{P}(X)} R(\psi, \mu) + J_{\nu_0}(\mu).
    \label{min-GANs-UDAs}
\end{equation}
Here, the first variable $\psi$ in $R$ corresponds to the predictor.
A typical example of $R$ is that $R(\psi, \mu)=\int |\psi(x)- \psi_0(x)|^2d\mu(x)$ where $\psi_0$ is the true predictor. 
The particular discrepancy measures lead to the well-known models of domain adversarial neural networks (DANNs)~\citep{ganin2015unsupervised}, such as DANNs with its extensions with Wasserstein-1 distance~\citep{shen2018wasserstein}, $f$-divergence~\citep{acuna2021a}, and MMD~\citep{wu2022distribution}.
As in the case of GANs~\citep{chu2019probability}, the Fenchel-Moreau theorem yields the following formulation equal to \eqref{min-GANs-UDAs}:
\begin{equation}
    \min_{(\psi, \mu) \in \mathcal{C}(X) \times \mathcal{P}(X)} \max_{\varphi \in \mathcal{C}(X)} R(\psi, \mu) + \int \varphi d\mu - J_{\nu_0}^{\star}(\varphi).
    \label{intro-minimax-eq}
\end{equation}
This objective function is convex for $\psi$ and $\mu$, and concave for $\varphi$, where $\varphi$ corresponds to the domain classifier in the UDA.
In Section~\ref{Minimax analysis}, we delve into the convergence of this objective function in the general setting.

%This objective function is convex-concave. Thus, under appropriate assumptions, we can prove that some gradient descent algorithm converges to the optimal minimax solution.
%In Section~\ref{Minimax analysis}, we consider the convex-concave minimax problem, which is the general setting of (\ref{intro-minimax-eq}).
%This formulation can be regarded as a specific example of the minimax problem for which we will prove the convergence theorem in Section~\ref{Minimax analysis}.

By omitting the source risk $R$, the formulation (\ref{intro-minimax-eq}) reduces to that of GAN :
% , as in~\eqref{intro-minimax-eq-gan-special}. %%[REF=Chu??].
\begin{equation}
    \min_{\mu \in \mathcal{P}(X)} \max_{\varphi \in \mathcal{C}(X)} \int \varphi d\mu - J_{\nu_0}^{\star}(\varphi),
    \label{intro-minimax-eq-gan-special}
\end{equation}
where $\varphi$ corresponds to the discriminator in the GAN.
This allows us to analyze the convergence properties in GANs and UDAs in a unified manner.
In other words, the findings of GANs, which have been extensively studied for stability, %can
could be used for UDAs. In fact, the assumptions used in this paper are related to the constraints of the GANs (see Section~\ref{Application}).

However, the formulation of (\ref{intro-minimax-eq}), which extends the reformulation of \citet{chu2019probability}, deviates from minimax optimization in actual GANs and UDAs such as \citet{goodfellow2014generative,ganin2015unsupervised}, as it does not directly optimize the distribution $\mu$. %of the generative model.
To get more practical situations, we consider the source distribution $\mu$ as pushforward measure $f_{\sharp} \xi_{0}$ of fixed probability measure $\xi_{0} \in \cP(Z)$ by continuous function $f \in \cC(Z; X)$% corresponding to a generator. 
, which corresponds to a generator in GANs, or a feature extractor in UDAs.
Then, the problem (\ref{intro-minimax-eq}) is reformulated as 
\begin{equation}
    \min_{\psi \in \mathcal{C}(X)} 
    \min_{f \in \mathcal{C}(Z;X)} 
    \max_{\varphi \in \mathcal{C}(X)} R(\psi, f_{\sharp}\xi_0) + \int \varphi d(f_{\sharp}\xi_0) - J_{\nu_0}^{\star}(\varphi).
    \label{intro-minimax-eq-nonconvex}
\end{equation}
This objective function is generally nonconvex for $\psi$ and $f$. In Section~\ref{Minimax analysis-nonconvex}, we explore the convergence of this objective function in the general setting.

\begin{comment}
\begin{figure}[h]
    \centering
    \includegraphics[width=0.6\linewidth]{fig.png}
    \caption{Diagram of our setting}
    \label{fig:fig}
\end{figure}
\end{comment}

%This objective function is nonconvex-concave. Thus, in general, there does not exist a minimax solution so that we can show the convergence to not an optimal minimax solution but a stationary point. In Section~\ref{Minimax analysis-nonconvex}, we consider nonconvex-concave minimax problem which is the general setting of (\ref{intro-minimax-eq-nonconvex}).

%Since the objective function in (\ref{intro-minimax-eq-nonconvex}) is generally nonconvex with respect to $f$, it serves as an example of the minimax problem for which we will prove the convergence theorem in Section~\ref{Minimax analysis-nonconvex}.

%%%%%%%%%%%%%%%%%%%%%%%%%%%%%%%%%%%%%%%%%%%%%%%%%%%%%%%%%%%%%%%%
%%%%%%%%%%%%%%%%%%%%%%%%%%%%%%%%%%%%%%%%%%%%%%%%%%%%%%%%%%%%%%%%
\section{Minimax Analysis}\label{sec:Minimax analysis}
Our goal in this section is to prove the convergence of the minimax optimization problem in the scheme of the gradient descent under appropriate assumptions. 
In Section~\ref{Minimax analysis}, we will consider the convex-concave problem over spaces of continuous functions and probability measures, and prove that the sequence obtained by % some 
a certain gradient descent converges to the optimal minimax solution. 
While, in Section~\ref{Minimax analysis-nonconvex}, we will consider the nonconvex-concave problem over spaces of continuous functions, and show 
%the convergence to a stationary point. 
that the sequence obtained by % some 
a certain gradient descent converges to a stationary point.

%\Erasesaw{The objective functions in Sections~\ref{Minimax analysis} and ~\ref{Minimax analysis-nonconvex} are general forms of (\ref{intro-minimax-eq}) and (\ref{intro-minimax-eq-nonconvex}), and we here do not use dual properties, which will be adapted in Section~\ref{Application} that will exam the assumptions for the convergence in the case of objective function in (\ref{intro-minimax-eq}) and (\ref{intro-minimax-eq-nonconvex}).}
Note that the objective functions in Sections~\ref{Minimax analysis} and ~\ref{Minimax analysis-nonconvex} are general forms of (\ref{intro-minimax-eq}) and (\ref{intro-minimax-eq-nonconvex}), respectively.

\begin{comment}
{\color{red}
[Reviewer tcXx. 
2.For each of the assumptions, it’d be good to explain why it’s needed, and any examples that
such assumptions are realistic. I understand that the author put an abstract framework in sec4 and
put the examples in sec5. I think it’d be good to give some pointers in sec4, otherwise the reader
might be confused and unsure where these assumptions come from.
]
}
Done
% Section 4の仮定の後に, pointers(どういう仮定なのか, どうつながるか？)を追加, 
\end{comment}

%
\subsection{Convex-concave setting}\label{Minimax analysis}
\label{Convex-concave setting}
This section considers the following minimax problem: 
%We consider the following minimax problem: 
\begin{equation}
\min_{(\psi, \mu) \in \Spsi\times \Smu}
\max_{\varphi \in \Svar} \
\mathcal{K}(\psi, \mu, \varphi),
\label{Minimax-K-eq}
\end{equation}
%
% sawa: 上でthis sectionに主語を変えたので、わざわざこのセクションではcompact convex subsetsと言わなくてもwhere後でcompact convex subsetsと言えばいいと思い、削除しました。
where $\Smu \subset \mathcal{P}(X)$ and $\Spsi, \Svar  \subset \mathcal{C}(X)$ are compact convex subsets and $\mathcal{K}:\mathcal{C}(X) \times \mathcal{M}(X) \times \mathcal{C}(X) \to \overline{\mathbb{R}}$ is supposed to be an objective function of GANs or UDAs.
The typical example of $\mathcal{K}$ is the objective function in (\ref{intro-minimax-eq}), that is,
\[
\mathcal{K}(\psi, \mu, \varphi)=R(\psi, \mu) + \int \varphi d\mu - J_{\nu_0}^{\star}(\varphi).
\]

%where $\Smu \subset \mathcal{P}(X)$ and $\Spsi, \Svar  \subset \mathcal{C}(X)$ are convex subsets and $\mathcal{K}:\mathcal{C}(X) \times \mathcal{M}(X) \times \mathcal{C}(X) \to \overline{\mathbb{R}}$ that is supposed to be an objective function of GANs or UDAs.
%Throughout this section, we assume that $\Smu \subset \mathcal{P}(X)$ and $\Spsi, \Svar \subset \mathcal{C}(X)$ are compact convex subsets. 

We show that \textit{the sequence obtained by the gradient descent converges to the optimal solution of \eqref{Minimax-K-eq} under appropriate assumptions.} 
%\textit{Our goal in this section is to show that the sequence obtained by the gradient descent converges to the optimal solution of \eqref{Minimax-K-eq} under appropriate assumptions.}
% sawa: もし、joint convexityを課している先行文献があるなら、ここでひとこと、彼らと同じ仮定を考える、と入れた方がいいと思いました。
To do this, we first consider the joint convexity as follows:

%%%%%%%%%%%%%%%%%%%%%%%%%%%%%%%%%%%%%%%%%%
\begin{assumption}\label{convex-concave-K}
Assume the following:
\begin{itemize}
\item[(i)] $\mathcal{K}(\cdot, \cdot, \varphi)$ is proper, lower semi-continuous, and convex over $\Spsi \times \Smu$ for each $\varphi \in  \Svar $.
\item[(ii)] $\mathcal{K}(\psi, \mu, \cdot)$ is proper, upper semi-continuous, and concave over $\Svar$ for each $\psi \in \Spsi$ and $\mu \in \Smu$.
\end{itemize}
\end{assumption}
%%%%%%%%%%%%%%%%%%%%%%%%%%%%%%%%%%%%%%%%%
This assumption means that the problem (\ref{Minimax-K-eq}) is a convex-concave problem.
Under this assumption, Sion's minimax theorem~\citep{sion1958general} guarantees that 
%Under Assumption~\ref{convex-concave-K}, Sion's minimax theorem~\citep{sion1958general} guarantees that 
\begin{equation*}
    \min_{(\psi, \mu) \in \Spsi\times \Smu}
    \max_{\varphi \in \Svar} 
    \mathcal{K}(\psi, \mu, \varphi)
    =
    \max_{\varphi \in \Svar}\min_{(\psi, \mu) \in \Spsi\times \Smu}
    \mathcal{K}(\psi, \mu, \varphi).
\end{equation*}
Moreover, there exists at least one minimax solution $(\psi_\ast, \mu_\ast, \varphi_\ast)$ in our minimax problem (\ref{Minimax-K-eq}),
\begin{align}
    \begin{split}
    &\mathcal{K}(\psi_{\ast}, \mu_{\ast}, \varphi_{\ast}) 
    \geq \mathcal{K}(\psi_{\ast}, \mu_{\ast}, \varphi), 
    \quad
    \varphi \in \Svar,\\
    &\mathcal{K}(\psi_{\ast}, \mu_{\ast}, \varphi_{\ast}) 
    \leq \mathcal{K}(\psi, \mu, \varphi_{\ast}), 
    \quad
    (\psi, \mu) \in \Spsi \times \Smu.
    \label{eq:nash-minimax}
    \end{split}
\end{align}

\begin{comment}
The definition of a Nash equilibrium for a three player game
\begin{align}
    \begin{split}
        &\mathcal{K}(\psi_{\ast}, \mu_{\ast}, \varphi_{\ast}) 
        \geq \mathcal{K}(\psi_{\ast}, \mu_{\ast}, \varphi), 
        \quad
        \varphi \in \Spsi,\\
        &\mathcal{K}(\psi_{\ast}, \mu_{\ast}, \varphi_{\ast}) 
        \leq \mathcal{K}(\psi, \mu_\ast, \varphi_{\ast}), 
        \quad
        \psi \in \Spsi,\\
        &\mathcal{K}(\psi_{\ast}, \mu_{\ast}, \varphi_{\ast}) 
        \leq \mathcal{K}(\psi_\ast, \mu, \varphi_{\ast}), 
        \quad
        \mu \in \Smu.
        \label{eq:nash-minminmax}
    \end{split}
    \end{align}
\end{comment}

\begin{comment}
We remark that (\ref{Minimax-K-eq}) can be treated as convex-convex-concave miniminimax problem for $\psi \in \Spsi, \mu \in \Smu, \varphi \in \Svar $.
However, here we consider the convex-concave minimax problem for $(\psi, \mu) \in \Spsi \times \Smu$ and $\varphi \in \Svar $ by assuming joint convexity in Assumption~\ref{convex-concave-K} because it is generally difficult to show the existence of a Nash equilibrium point for convex-convex-concave problem.
%because there does not exist a Nash equilibrium point for convex-convex-concave problem in general.
\end{comment}
% sawa: 5章との関連説明を追加。機械学習で良く使われるRKHSとMMDとの関連をここで示しておくと見通しが良い（査読者対応の一つ）
% Note that this assumption is in line with reality.
Note that this assumption is in line with practical settings.
Indeed, as shown in Section~\ref{subsec:joint_convexity}, 
%the optimization function 
the source risk $R$ of UDAs can be joint convex by adding both reproducing kernel Hilbert space (RKHS) \citep{alvarez2012kernels} and maximal mean discrepancy (MMD) \citep{gretton2012kernel} constraints.

%Let $\left\| \cdot \right\|_{\normpsi}$ and $\left\| \cdot \right\|_{\normvar}$ be norms induced by inner products in $\cC(X)$. Note that both the first variable $\psi$ and the third variable $\varphi$ in $\mathcal{K}$ are continuous functions, but the inner product space $(\mathcal{C}(X), \left\| \cdot \right\|_{\normpsi})$ for $\psi$ is different from the inner product space $(\mathcal{C}(X), \left\| \cdot \right\|_{\normvar})$ for $\varphi$.
%Let $\left\| \cdot \right\|_{\normmu}$ be a norm induced by an inner product in $\cM(X)$.
Next, we put the assumptions related to the G\^{a}teaux differentials. 
Let $\left\| \cdot \right\|_{\normpsi}$ and $\left\| \cdot \right\|_{\normvar}$ be norms induced by inner products in $\cC(X)$, and let $\left\| \cdot \right\|_{\normmu}$ be a norm induced by an inner product in $\cM(X)$. Note that both the first variable $\psi$ and the third variable $\varphi$ in $\mathcal{K}$ are continuous functions, but the inner product space $(\mathcal{C}(X), \left\| \cdot \right\|_{\normpsi})$ for $\psi$ is different from the inner product space $(\mathcal{C}(X), \left\| \cdot \right\|_{\normvar})$ for $\varphi$.

% sawa: 以下の文章は流れが悪いので、assumptionの中に移動
%Let $\left\| \cdot \right\|_{\normpsi}$ and $\left\| \cdot \right\|_{\normvar}$, and $\left\| \cdot \right\|_{\normmu}$ be norms induced by inner products in $\cC(X)$ and $\cM(X)$, respectively.

%%%%%%%%%%%%%%%%%%%%%%%%%%%%%%%%%%%%%%%%%%
\begin{assumption}\label{exist-argmax-DA-minimax}
%Let $\left\| \cdot \right\|_{\normpsi}$ and $\left\| \cdot \right\|_{\normvar}$, and $\left\| \cdot \right\|_{\normmu}$ be norms induced by inner products in $\cC(X)$ and $\cM(X)$, respectively.
We assume as follows:
%Assume the following:
\begin{itemize}
\item[(i)] 
For each $\psi \in \Spsi$, $\mu \in \Smu$, and $\varphi \in \Svar$, there exist the following arguments of the maximum:
\[
N_{\psi, \mu, \varphi}
= \underset{\nu \in \mathcal{M}(X)}{\mathrm{argmax}}
\left\{ \int \psi d\nu - \mathcal{K}(\cdot, \mu, \varphi)^{\star}(\nu)\right\},
\]
\[
\varPhi_{\psi, \mu, \varphi}
= 
\underset{\phi \in \mathcal{C}(X)}{\mathrm{argmax}}
\left\{ \int \phi d\mu - \mathcal{K}(\psi, \cdot, \varphi)^{\star}(\phi)\right\},
\]
\[
\Lambda_{\psi, \mu, \varphi}
= 
\underset{\lambda \in \mathcal{M}(X)}{\mathrm{argmax}}
\left\{ \int \varphi d\lambda - \mathcal{K}(\psi, \mu, \cdot)^{\star}(\lambda)\right\}.
\]
\item[(ii)]
$N_{\psi, \mu, \varphi}$, $\Phi_{\psi, \mu, \varphi}$, and $\Lambda_{\psi, \mu, \varphi}$ are bounded with respect to 
dual norms $\left\| \cdot \right\|_{\normpsi}^{\star}$, $\left\| \cdot \right\|_{\normmu}^{\star}$, and $\left\| \cdot \right\|_{\normvar}^{\star}$ , that is, there exists $B>0$ such that, for $(\psi, \mu, \varphi) \in \Spsi\times \Smu \times \Svar$,
\begin{equation}
\left\| N_{\psi, \mu, \varphi} \right\|_{\normpsi}^{\star} \leq  B,
\quad
\left\| \Phi_{\psi, \mu, \varphi} \right\|_{\normmu}^{\star}\leq  B,
\quad
\left\| \Lambda_{\psi, \mu, \varphi} \right\|_{\normvar}^{\star}\leq  B.
\label{boundness-DA}
\end{equation}
\end{itemize}
%%%%%%%%%%%%%%%%%%%%%%%%%%%%%%%%%%%%%%%%%%%%%%%%%%%%%%%%%%%%%%%%%%%
\end{assumption}
%%%%%%%%%%%%%%%%%%%%%%%%%%%%%%%%%%%%%%%%%%
Here, $\mathcal{K}(\cdot, \mu, \varphi)^{\star}$, $\mathcal{K}(\psi, \cdot, \varphi)^{\star}$, and $\mathcal{K}(\psi, \mu, \cdot)^{\star}$ are convex conjugates of $\mathcal{K}(\cdot, \mu, \varphi)$, $\mathcal{K}(\psi, \cdot, \varphi)$, and $\mathcal{K}(\psi, \mu, \cdot)$, respectively.
%For the definition of the convex conjugate, see Appendix~\ref{Preliminary}.
%%%%%%%%%%%%%%%%%%%%%%%%%%%%%%%%%%%%%%%%%%
The above assumption guarantees that the existence of G\^{a}teaux differentials of $\mathcal{K}$, and provide the form of their G\^{a}teaux differentials as following lemma. 
%First of all, we provide the form of G\^{a}teaux differentials of $\mathcal{K}$. 
The proof is given by the similar arguments in \citet[Theorem 2]{chu2019probability}.
%%%%%%%%%%%%%%%%%%%%%%%%%%%%%%%%%%%%%%%
\begin{lemma}\label{Gatea-diff-form-DA}
Let Assumption~\ref{exist-argmax-DA-minimax} hold.
Then, for each $\psi \in \Spsi$, $\mu \in \Smu$, and $\varphi \in \Svar$, there exist G\^{a}teaux differentials $d\mathcal{K}(\cdot, \mu, \varphi)_{\psi}$, $d\mathcal{K}(\psi, \cdot, \varphi)_{\mu}$, and $d\mathcal{K}(\psi, \mu, \cdot)_{\varphi}$ of $\mathcal{K}(\cdot, \mu, \varphi)$, $\mathcal{K}(\psi, \cdot, \varphi)$, and $\mathcal{K}(\psi, \mu, \cdot)$ at $\psi \in \Spsi$, $\mu \in \Smu$, and $\varphi \in \Svar$, and they are expressed as follows:
\begin{align*}
    d\mathcal{K}(\cdot, \mu, \varphi)_{\psi}(\eta) 
    &= \int \eta d N_{\psi, \mu, \varphi},\\
    d\mathcal{K}(\psi, \cdot, \varphi)_{\mu}(\chi) &= \int \Phi_{\psi, \mu, \varphi} d \chi,\\
    d\mathcal{K}(\psi, \mu, \cdot)_{\varphi}(\phi) 
    &= \int \phi d\Lambda_{\psi, \mu, \varphi}.
\end{align*}
\end{lemma}

%\begin{proof}
%See Lemma~\ref{Gatea-diff-form-DA-app} in Appendix~\ref{Appendix1}.
%\end{proof}
%%%%%%%%%%%%%%%%%%%%%%%%%%%%%%%%%%%%%%%
In addition to these assumptions, we assume the $L$-smoothness of $\mathcal{K}$ for each variable to show the convergence to a minimax solution.
%In order to show the convergence to a minimax solution, we assume the $L$-smoothness of $K$ for each variable.
%%%%%%%%%%%%%%%%%%%%%%%%%%%%%%%%%%%%%%%%%%%
\begin{assumption}\label{L-smooth-relative}
Let $L>0$, and let $\xi : \mathcal{C}(X) \to \overline{\mathbb{R}}$, $\phi : \mathcal{M}(X) \to \overline{\mathbb{R}}$, and $\eta : \mathcal{C}(X) \to \overline{\mathbb{R}}$ be proper, lower semi-continuous, and convex. Then, we assume as follows:
%Assume the following:
\begin{itemize}
\item[(i)] For each $\mu \in \Smu$ and $\varphi \in \Svar$, $\mathcal{K}(\cdot, \mu, \varphi)$ is $L$-smooth with respect to $\left\| \cdot \right\|_{\normpsi}$ over $\Spsi$.
\item[(ii)] For each $\psi \in \Spsi$ and $\varphi \in \Svar$, $\mathcal{K}(\psi, \cdot, \varphi)$ is $L$-smooth with respect to $\left\| \cdot \right\|_{\normmu}$ over $\Smu$.
\item[(iii)] For each $\psi \in \Spsi$ and $\mu \in \Smu$, $-\mathcal{K}(\psi, \mu, \cdot)$ is $L$-smooth with respect to $\left\| \cdot \right\|_{\normvar}$ over $\Svar$.
\end{itemize}
\end{assumption}
% sawa: 5章との関連説明を追加。機械学習で良く使われるRKHSとMMDとの関連をここで示しておくと見通しが良い（査読者対応の一つ）
This assumption also aligns with practical settings, %reality
e.g., (iii) corresponds to the case where the $f$-divergence~\citep{ali1966general, csiszar1967information} %and 
or integral probability metric (IPM)~\citep{muller1997integral} is utilized as a discrepancy measure.

%%%%%%%%%%%%%%%%%%%%%%%%%%%%%%%%%%%%%%%%%%%
Here, we define the gradient descent for solving minimax optimization problem (\ref{Minimax-K-eq}).
%%%%%%%%%%%%%%%%%%%%%%%%%%%%%%%%%%%%%%%
\begin{definition}\label{mirror-descent}
Let $\psi_{0} \in \Spsi$, $\mu_{0} \in \Smu$ $\varphi_{0} \in \Svar$ be initial guesses.
We define the gradient descent $\{(\psi_{n}, \mu_{n}, \varphi_{n})\}_{n\in \mathbb{N}_{0}} \subset \Spsi \times \Smu \times \Svar$ by
\begin{align*}
&\psi_{n+1}=\underset{\psi \in \Spsi}{\mathrm{argmin}}
\left\{
d\mathcal{K}( \cdot, \mu_{n}, \varphi_{n})_{\psi_n}(\psi-\psi_n) 
+ 
\frac{1}{2 \alpha_{n}}\left\|\psi-\psi_n\right\|^{2}_{\normpsi}
\right\},\\
&\mu_{n+1}=\underset{\mu \in \Smu}{\mathrm{argmin}}
\left\{
d\mathcal{K}(\psi_{n}, \cdot, \varphi_{n})_{\mu_n}(\mu-\mu_n) + \frac{1}{2 \alpha_{n}}\left\|\mu-\mu_n\right\|^{2}_{\normmu}
\right\},\\
&
\varphi_{n+1}=\underset{\varphi \in \Svar}{\mathrm{argmax}}
\left\{
d\mathcal{K}(\psi_{n}, \mu_{n}, \cdot)_{\varphi_n}(\varphi-\varphi_n)-\frac{1}{2 \alpha_{n}}\left\|\varphi-\varphi_n\right\|^{2}_{\normvar} \right\},\\
\end{align*}
where $\alpha_{n}>0$ is the step size of the update rule.
\end{definition}
If subsets $\Spsi$, $\Smu$, and $\Svar$ are subspaces, then update can be expressed as a sum of a previous step and a gradient term, a form that is commonly encountered in the gradient descent algorithm (see e.g., \citet{chong2023introduction}).  
%However, it is not always a subspace, as in the example in Section~\ref{verifi-Convex-concave}.
However, in the general case of subsets $\Spsi$, $\Smu$, and $\Svar$, the argmin and argmax in Definition~\ref{mirror-descent} may not exist.
Therefore, in this paper, the following assumption is established to ensure the existence of the gradient descent of Definition~\ref{mirror-descent}. 
%Moreover, if proximal term is replaced with Bregman divergence, it is called mirror descent...our first main result can be easiliy to extend,,,,

%%%%%%%%%%%%%%%%%%%%%%%%%%%%%%%%%%%%%%%%%%%%%%%%%
\begin{assumption}\label{existence-mirror-descent}
Assume that there exists a sequence $\{(\psi_{n}, \mu_{n},\varphi_{n})\}_{n \in \mathbb{N}_{0}} \subset \Spsi \times \Smu \times \Svar$ defined in Definition~\ref{mirror-descent}.
\end{assumption}
%

\begin{comment}
\begin{assumption}\label{strongly-convex-ass}
Let $\beta >0$. Assume that $\xi$ and $\phi$ are $\beta$-strongly convex with respect to $\left\| \cdot \right\|_{\mathcal{C}(X)}$ and $\left\| \cdot \right\|_{\mathcal{M}(X)}$, respectively.
\end{assumption}
%%%%%%%%%%%%%%%%%%%%%%%%%%%%%%%%%%%%%%%%%%%%%%%%%
For the definition of the strong convexity, see 
Appendix~\ref{Preliminary}.
%%%%%%%%%%%%%%%%%%%%%%%%%%%%%%%%%%%%%%%%%%%%%%%%%%%%%%%%%%%%%%%%%
Examples of the strongly convex function are norms $\xi=\frac{1}{2}\left\| \cdot \right\|_{\mathcal{C}(X)}^{2}$ and $\phi=\frac{1}{2}\left\| \cdot \right\|_{\mathcal{M}(X)}^{2}$ that are induced by inner products.
%
Another example is the Kullback-Leibler divergence, which is strongly convex with respect to the total variation norm.
\end{comment}

% See the following link:
% https://mathoverflow.net/questions/307062/is-kl-divergence-dpq-strongly-convex-over-p-in-infinite-dimension
% "KLがTV distanceに関してstrongly convexである"の参考文献(本，論文)は見つけられていない.
% 別文献でclassically well-known exampleとかかかれて引用なしだったので，ここでも引用なしにしておく

%%%%%%%%%%%%%%%%%%%%%%%%%%%%%%%%%%%%%%%%%%%%%%%%%%%%%%%%%%%%%%%%
Building upon the background established above, we are ready to present our main theorem of this section:
%The following is our main theorem in this section:
%%%%%%%%%%%%%%%%%%%%%%%%%%%%%%%%%%%%%%%%%%%%%%%%%%%%%%%%%%%%%%%%%
\begin{theorem}\label{main-theorem-convergence}
Let Assumptions~\ref{convex-concave-K}, \ref{exist-argmax-DA-minimax}, \ref{L-smooth-relative}, and \ref{existence-mirror-descent} hold, and let $0<\alpha_{n} \leq 1/L$.
Let $\{(\psi_{n}, \mu_{n}, \varphi_{n})\}_{n\in \mathbb{N}_{0}} \subset \Spsi \times \Smu \times \Svar$ be the gradient descent defined in Definition~\ref{mirror-descent}.
Let $(\psi_{\ast}, \mu_{\ast}, \varphi_{\ast})$ be a minimax solution for (\ref{Minimax-K-eq}).
Then, for any $N \in \mathbb{N}$, we have
\begin{equation}
\begin{split}
&
\left|
\mathcal{K}(\widehat{\psi}_{N}, \widehat{\mu}_{N}, \widehat{\varphi}_{N})
- \mathcal{K}(\psi_{\ast}, \mu_{\ast}, \varphi_{\ast})
\right|
\leq 
\left( 
\sum_{n=0}^{N-1}\alpha_{n} \right)^{-1}
\left( \frac{1}{2} C_{s}
+
6B^2
\sum_{n=0}^{N-1}
\alpha_{n}^{2}
\right),
\end{split}
\label{main-theorem-convergence-eq}
\end{equation}
where 
\begin{equation}
C_{s}:=
\sup_{\psi \in \Spsi}\|\psi-\psi_{0}\|_{\normpsi}^2
+
\sup_{\mu \in \Smu}\|\mu-\mu_0\|_{\normmu}^2 
+\sup_{\varphi \in \Svar} \|\varphi-\varphi_0\|_{\normvar}^2
.
\label{const-C}
\end{equation}
Here, $\widehat{\psi}_{N}$, $\widehat{\mu}_{N}$, and $\widehat{\varphi}_{N}$ are weighted averages given by
\begin{equation}
\widehat{\psi}_{N} := \frac{\sum_{n=0}^{N-1}\alpha_{n}\psi_{n}}{\sum_{n=0}^{N-1}\alpha_{n}}, \quad
\widehat{\mu}_{N} := \frac{\sum_{n=0}^{N-1}\alpha_{n}\mu_{n}}{\sum_{n=0}^{N-1}\alpha_{n}} , \quad
\widehat{\varphi}_{N} := \frac{\sum_{n=0}^{N-1}\alpha_{n}\varphi_{n}}{\sum_{n=0}^{N-1}\alpha_{n}}.
\label{averages}
\end{equation}

\end{theorem}
\begin{proof}
See Appendix~\ref{Appendix1}.
\end{proof}
%%%%%%%%%%%%%%%%%%%%%%%%%%%%%%%%%%%%%%%%%%%%%%%%%%%%%%%%%%%%%%%%%%%
We %first 
note that the the constant $C_{s}$ 
is finite due to the compactness of $\Smu$, $\Spsi$, and $\Svar$.
%, and the continuity of the Bregman divergences.
We observe that the upper bounds (\ref{main-theorem-convergence-eq}) with different choices of step sizes $\alpha_n \in (0,1]$ are as follows:
%We observe that the upper bound (\ref{main-theorem-convergence-eq}) with different choices of step sizes $\alpha_n \in (0,1]$ is as follows:
%
\begin{itemize}
\item
If the step sizes are constant, denoted by $\alpha_n=\alpha$, then the right-hand side (RHS) of  (\ref{main-theorem-convergence-eq}) is expressed as
\[
\text{RHS of (\ref{main-theorem-convergence-eq})}
= \frac{C_{s}}{2 \alpha N}
+
6B^2\alpha,
\]
which does not converges to zero as $N\to \infty$.
Therefore, in this case, weighted averages~(\ref{averages}) provides an approximate solution to the minimax problem.
The first term converges to zero as $N\to \infty$ with an order of $\mathcal{O}(1/N)$.
The second term can be reduced as $\alpha \to 0$, despite the first term diverging.
This is a trade-off relationship with respect to the step size $\alpha$. 
A similar observation was made in \citet[Proposition 3.1]{nedic2009subgradient}, which studied the minimax problem in finite dimensional space using subgradient methods.
%%%%%%%%%%%%%%%%%%%%%%%%%%%%%%%%%%%%%%%%%%%%%%%%%%%%%%%%%%%
% sue: 問題のないスケーリング？既知の結果を再発見？？
%%%%%%%%%%%%%%%%%%%%%%%%%%%%%%%%%%%%%%%%%%%%%%%%%%%%%%%%%%%
\item 
If step sizes decay as $\alpha_n= \alpha/\sqrt{n}$ where $\alpha$ is a constant, then the right-hand side of  (\ref{main-theorem-convergence-eq}) is expressed as
\[
\text{RHS of (\ref{main-theorem-convergence-eq})}
=
\frac{C_{s}}{2\alpha \sqrt{N}}
+
\frac{6B^2\alpha}{\sqrt{N}}
(1+\log N)
,
\]
which converges to zero as $N\to \infty$ with an order of $\mathcal{O}(\log N/\sqrt{N})$. Therefore, in this case, weighted averages~(\ref{averages}) provide an exact solution to the minimax problem.

\end{itemize}

\subsection{Nonconvex-concave setting}\label{Minimax analysis-nonconvex}
Unlike the previous section, which considered the convex-concave minimax problem expressed in~\eqref{Minimax-K-eq}, this section considers the nonconvex-concave minimax problem. 
%expressed in~\eqref{Minimax-G-eq}.

Let $Z \subset \mathbb{R}^{d^{\prime}}$ be a compact set, 
and
let $\cC(Z;X)$ be the set of all continuous functions $Z \to X$,
and
let $\Sf \subset \cC(Z;X)$ and $\Spsi, \Svar  \subset \mathcal{C}(X)$ be subspaces.
Then, we consider the following minimax problem: 
\begin{equation}
\min_{\psi \in \Spsic}
\min_{f \in \Sfc}
\max_{\varphi \in \Svarc} 
\mathcal{G}(\psi, f, \varphi),
\label{Minimax-G-eq}
\end{equation}
where $\Sfc \subset \Sf$, $\Spsic \subset \Spsi$, and $\Svarc  \subset \Svar$ are convex subsets, and $\mathcal{G}:\cC(X) \times \cC(Z;X) \times \mathcal{C}(X) \to \overline{\mathbb{R}}$ is supposed to be an objective function of GANs or UDAs.
The typical example of $\mathcal{G}$ is the objective function in (\ref{intro-minimax-eq-nonconvex}), that is,
\[
\mathcal{G}(\psi, f, \varphi)=R(\psi, f_{\sharp}\xi_0) + \int \varphi d(f_{\sharp}\xi_0) - J_{\nu_0}^{\star}(\varphi).
\]

The difference with Section~\ref{Convex-concave setting} is that~\eqref{Minimax-G-eq} does not assume the convexity for $\mathcal{G}(\cdot, f, \varphi)$ and $\mathcal{G}(\psi, \cdot, \varphi)$.
%Unlike Section~\ref{Convex-concave setting}, we do not assume the convexity for $\mathcal{G}(\cdot, f, \varphi)$ and $\mathcal{G}(\psi, \cdot, \varphi)$.
Thus, since there may not exist a Nash equilibrium point for problem (\ref{Minimax-G-eq}) in general, it is difficult to prove that the sequence obtained by some
gradient descent converges to the optimal minimax solution. 

We show that {\it the sequence obtained by a certain gradient descent converges to a stationary point of (\ref{Minimax-G-eq}) under appropriate assumptions.}
%{\it Our goal in this section is to show that the sequence obtained by the gradient descent converges to a stationary point under appropriate assumptions.}

Throughout this section, let $\langle \cdot, \cdot \rangle_{\Spsi}$, $\langle \cdot, \cdot \rangle_{\Sf}$, and $\langle \cdot, \cdot \rangle_{\Svar}$ be inner products in in $\Spsi$, $\Sf$, and $\Svar$, respectively.
We denote $\|\cdot \|_{\Spsi}$, $\|\cdot \|_{\Sf}$, and $\|\cdot \|_{\Svar}$ by norms induced by thier inner products.
%Then, we assume that $\Spsi$, $\Sf$, and $\Svar$ are the Hilbert spaces as follows.
%, which is a widely assumed in machine learning field.
%

First, we put the following assumption for $\Spsi$, $\Sf$, and $\Svar$.
\begin{assumption}\label{closed-subspaces-nonconvex}
Assume that $\Spsi$, $\Sf$, and $\Svar$ are closed subspace with respect to norms $\|\cdot \|_{\Spsi}$, $\|\cdot \|_{\Sf}$, and $\|\cdot \|_{\Svar}$ in $\cC(X)$, $\cC(Z;X)$, and $\cC(X)$, respectively.
\end{assumption}
Under Assumption~\ref{closed-subspaces-nonconvex},  $\Spsi$, $\Sf$, and $\Svar$ are Hilbert spaces equipped with inner products $\langle \cdot, \cdot \rangle_{\Spsi}$, $\langle \cdot, \cdot \rangle_{\Sf}$, and $\langle \cdot, \cdot \rangle_{\Svar}$, respectively.
%and we can denote their inner products by $\langle \cdot, \cdot \rangle_{\Spsi}$, $\langle \cdot, \cdot \rangle_{\Sf}$, and $\langle \cdot, \cdot \rangle_{\Svar}$, respectively.

Next, we put the assumption about the $\beta$-strongly concave.
\begin{assumption}\label{strongly-concave-L-nonconvex}
Let $\beta >0$.
Assume that for each $\psi \in \Spsic$ and $f \in \Sfc$, $\mathcal{G}(\psi, f, \cdot)$ is $\beta$-strongly concave with respect to $\|\cdot \|_{\Svar}$ over $\Svarc$.
\end{assumption}
%
% sawa: 5章との繋がり追加
Note that this assumption is related to the gradient penalties~\citep{gulrajani2017improved}, widely-used as the stabilization techniques in adversarial training, as detailed in Section~\ref{sub:strongly-concave-L-nonconvex}.

Under this assumption, we can define for $\psi \in \Spsic$ and $f \in \Sfc$,% respectively,
\begin{align}\label{def-argmax-function}
&
\Phi(\psi,f):=\underset{\varphi \in \Svarc}{\mathrm{argmax}}\ 
\mathcal{G}(\psi, f, \varphi),
\notag
\\
&
G(\psi,f)
:=\max_{\varphi \in \Svarc}\mathcal{G}(\psi, f, \varphi)
=\mathcal{G}(\psi, f, \Phi(f,\psi)).
\end{align} 
Hereby, the minimax problem (\ref{Minimax-G-eq}) is equivalent to minimization of (\ref{def-argmax-function}) under Assumption~\ref{strongly-concave-L-nonconvex}.

Then, we put the following assumption related to the G\^{a}teaux differentials, which is associated with the spectral normalization~\citep{miyato2018spectral} widely-used as stabilization techniques for GANs.
%the same as gradient penalty.
%
\begin{assumption}\label{existence-Gate-diff-nonconvex}
Assume that, for each $\psi \in \Spsic$, $f \in \Sfc$, and $\varphi \in \Svarc$, there exist G\^{a}teaux differentials $d\mathcal{G}(\cdot, f, \varphi)_{\psi}$, $d\mathcal{G}(\psi, \cdot, \varphi)_{f}$, and $d\mathcal{G}(\psi, f, \cdot)_{\varphi}$ of $\mathcal{G}(\cdot, f, \varphi)$, $\mathcal{G}(\psi, \cdot, \varphi)$, and $\mathcal{G}(\psi, f, \cdot)$ at $\psi \in \Spsic$, $f \in \Sfc$, and $\varphi \in \Svarc$, respectively.
\end{assumption}
Under Assumptions~\ref{closed-subspaces-nonconvex} and \ref{existence-Gate-diff-nonconvex}, G\^{a}teaux differentials $d\mathcal{G}(\cdot, f, \varphi)_{\psi}:\Spsi \to \overline{\mathbb{R}}$, $d\mathcal{G}(\psi, \cdot, \varphi)_{f}:\Smu \to \overline{\mathbb{R}}$, and $d\mathcal{G}(\psi, f, \cdot)_{\varphi}:\Svar \to \overline{\mathbb{R}}$ are identified with some elements in Hilbert spaces $\Spsi$, $\Sf$, and $\Svar$, referred as $\nabla \mathcal{G}(\cdot, f, \varphi)_{\psi} \in \Spsi$, $\nabla \mathcal{G}(\psi, \cdot, \varphi)_{f} \in \Sf$, and $\nabla \mathcal{G}(\psi, f, \cdot)_{\varphi} \in \Svar$, respectively.
Furthermore, by Riesz representation theorem, we have the following:
%
% \begin{align*}
% &
% d\mathcal{G}(\cdot, f, \varphi)_{\psi} = \langle \nabla \mathcal{G}(\cdot, f, \varphi)_{\psi}, \cdot \rangle_{\Spsi}, \quad
% \|d\mathcal{G}(\cdot, f, \varphi)_{\psi} \|_{\Spsi}^{\star} = \|\nabla \mathcal{G}(\cdot, f, \varphi)_{\psi} \|_{\Spsi},\\
% %
% &
% d\mathcal{G}(\psi, \cdot, \varphi)_{f} = \langle \nabla \mathcal{G}(\psi, \cdot, \varphi)_{f}, \cdot \rangle_{\Sf}, \quad
% \|d\mathcal{G}(\psi, \cdot, \varphi)_{f} \|_{\Sf}^{\star} = \|\nabla \mathcal{G}(\psi, \cdot, \varphi)_{f} \|_{\Sf},\\
% %
% &
% d\mathcal{G}(\psi, f, \cdot)_{\varphi} = \langle \nabla \mathcal{G}(\psi, f, \cdot)_{\varphi}, \cdot \rangle_{\Svar}, \quad 
% \|d\mathcal{G}(\psi, f, \cdot)_{\varphi} \|_{\Svar}^{\star} = \|\nabla \mathcal{G}(\psi, f, \cdot)_{\varphi} \|_{\Svar}.
% \end{align*}
\begin{align*}
d\mathcal{G}(\cdot, f, \varphi)_{\psi} &= \langle \nabla \mathcal{G}(\cdot, f, \varphi)_{\psi}, \cdot \rangle_{\Spsi}, 
&
\|d\mathcal{G}(\cdot, f, \varphi)_{\psi} \|_{\Spsi}^{\star} &= \|\nabla \mathcal{G}(\cdot, f, \varphi)_{\psi} \|_{\Spsi},\\
d\mathcal{G}(\psi, \cdot, \varphi)_{f} &= \langle \nabla \mathcal{G}(\psi, \cdot, \varphi)_{f}, \cdot \rangle_{\Sf}, 
&
\|d\mathcal{G}(\psi, \cdot, \varphi)_{f} \|_{\Sf}^{\star} &= \|\nabla \mathcal{G}(\psi, \cdot, \varphi)_{f} \|_{\Sf},\\
d\mathcal{G}(\psi, f, \cdot)_{\varphi} &= \langle \nabla \mathcal{G}(\psi, f, \cdot)_{\varphi}, \cdot \rangle_{\Svar},
&
\|d\mathcal{G}(\psi, f, \cdot)_{\varphi} \|_{\Svar}^{\star} &= \|\nabla \mathcal{G}(\psi, f, \cdot)_{\varphi} \|_{\Svar}.
\end{align*}
In addition to these assumptions, we assume the $L$-smoothness of $\cG$ for each variable $\psi$, $f$, and $\varphi$ to show the convergence to %a minimax solution
a stationary point.
%In order to show the convergence, we assume the following:
\begin{assumption}\label{L-smooth-nonconvex} 
Let $L>0$. Then, we assume the following:
for $\psi, \psi_1, \psi_2 \in \Spsi$, $f, f_1, f_2 \in \Sf$, and $\varphi, \varphi_1, \varphi_2 \in \Svar$,
\begin{equation}\label{L-smooth-eq-nonconvex}
\begin{split}
& 
(a):\ 
\| \nabla \cG (\cdot, f, \varphi)_{\psi_1} - \nabla \cG (\cdot, f, \varphi)_{\psi_2} \|_{\Spsi} \leq L \| \psi_1 - \psi_2 \|_{\Spsi},\\
&
(b):\ 
\| \nabla \cG (\cdot, f_1, \varphi)_{\psi} - \nabla \cG (\cdot, f_2, \varphi)_{\psi} \|_{\Spsi} \leq L \| f_1 - f_2 \|_{\Sf},\\
&
(c):\ 
\| \nabla \cG (\cdot, f, \varphi_1)_{\psi} - \nabla \cG (\cdot, f, \varphi_1)_{\psi} \|_{\Spsi} \leq L \| \varphi_1 - \varphi_2 \|_{\Svar},\\
&
(d):\ 
\| \nabla \cG (\psi, \cdot, \varphi)_{f_1} - \nabla \cG (\psi, \cdot, \varphi)_{f_2} \|_{\Sf} \leq L \| f_1 - f_2 \|_{\Sf},\\
&
(e):\ 
\| \nabla \cG (\psi_1, \cdot, \varphi)_{f} - \nabla \cG (\psi_2, \cdot, \varphi)_{f} \|_{\Sf} \leq L \| \psi_1 - \psi_2 \|_{\Spsi},\\
&
(f):\ 
\| \nabla \cG (\psi, \cdot, \varphi_1)_{f} - \nabla \cG (\psi, \cdot, \varphi_2)_{f} \|_{\Sf} \leq L \| \varphi_1 - \varphi_2 \|_{\Svar},\\
&
(g):\ 
\| \nabla \cG (\psi, f, \cdot)_{\varphi_1} - \nabla \cG (\psi, f, \cdot)_{\varphi_2} \|_{\Svar} \leq L \| \varphi_1 - \varphi_2 \|_{\Svar},\\
&
(h):\ 
\| \nabla \cG (\psi_1, f, \cdot)_{\varphi} - \nabla \cG (\psi_2, f, \cdot)_{\varphi} \|_{\Svar} \leq L \| \psi_1 - \psi_2 \|_{\Spsi},\\
&
(i):\ 
\| \nabla \cG (\psi, f_1, \cdot)_{\varphi} - \nabla \cG (\psi, f_2, \cdot)_{\varphi} \|_{\Svar} \leq L \| f_1 - f_2 \|_{\Sf}.
\end{split}
\end{equation}
\end{assumption}
%
%Note that Assumption~\ref{L-smooth-nonconvex} implicates the $L$-smoothness (Definition~\ref{Relative smoothness for measure and continuous function spaces}) of $\cG$ for each variable $\psi$, $f$, and $\varphi$. 
Here, we define the projected gradient descent for solving minimax optimization problem (\ref{Minimax-G-eq}).
\begin{definition}\label{gradient-descent-nonconvex}
Let $\psi_{0} \in \Spsic$, $f_{0} \in \Sfc$ $\varphi_{0} \in \Svarc$ be initial guesses.
Then, we define the projected gradient descent $\{(\psi_{n}, f_{n}, \varphi_{n})\}_{n\in \mathbb{N}_{0}} \subset \Spsic \times \Sfc \times \Svarc$ by
\begin{align*}
&\widetilde{\psi}_{n+1}=\underset{\psi \in \Spsi}{\mathrm{argmin}}
\left\{
d\mathcal{G}( \cdot, f_{n}, \varphi_{n})_{\psi_n}(\psi-\psi_n) 
+ 
\frac{1}{2 \alpha_{\psi, n}}\left\|\psi-\psi_n\right\|^{2}_{\Spsi}
\right\},\\
&\psi_{n+1}= \cP_{\Spsic}(\widetilde{\psi}_{n+1}),\\
&\widetilde{f}_{n+1}=\underset{f \in \Sf}{\mathrm{argmin}}
\left\{
d\mathcal{G}(\psi_{n}, \cdot, \varphi_{n})_{f_n}(f-f_n) + \frac{1}{2 \alpha_{f,n}}\left\|\mu-\mu_n\right\|^{2}_{\Sf}
\right\},\\
&f_{n+1}= \cP_{\Sfc}(\widetilde{f}_{n+1}),\\
&
\widetilde{\varphi}_{n+1}=\underset{\varphi \in \Svar}{\mathrm{argmax}}
\left\{
d\mathcal{G}(\psi_{n}, f_{n}, \cdot)_{\varphi_n}(\varphi-\varphi_n)-\frac{1}{2 \alpha_{\varphi, n}}\left\|\varphi-\varphi_n\right\|^{2}_{\Svar} \right\},\\
&\varphi_{n+1}= \cP_{\Svarc}(\widetilde{\varphi}_{n+1})
\end{align*}
where $\cP_{\Spsic}$, $\cP_{\Sfc}$, and $\cP_{\Svarc}$ are projection operators on $\Spsic$, $\Sfc$, and $\Svarc$, and
$\alpha_{\psi, n}>0$, $\alpha_{f, n}>0$, and $\alpha_{\varphi, n}>0$ are step sizes.
\end{definition}
Remark that, using Riesz representation theorem under Assumption~\ref{closed-subspaces-nonconvex}, the above update rule is equivalent to the following:
\begin{align*}
&\psi_{n+1}=\cP_{\Spsic}\left( \psi_{n} - \alpha_{\psi, n} \nabla \mathcal{G}(\cdot, f_n, \varphi_n)_{\psi_n} \right),\\
&f_{n+1}=\cP_{\Sfc}\left( f_{n} - \alpha_{f, n} \nabla \mathcal{G}(\psi_n, \cdot, \varphi_n)_{f_n} \right),\\
&
\varphi_{n+1}=\cP_{\Svarc}\left( \varphi_{n} + \alpha_{\varphi, n} \nabla \mathcal{G}(\psi_n, f_{n}, \cdot)_{\varphi_n} \right).
\end{align*}
We also assume small step sizes to show the convergence of the gradient descent algorithm as follows: 
%In order to show the convergence, we need to choose suitable step sizes as follows: 

%sawa: ここの仮定のそれぞれに何か意味付けできますか？
\begin{assumption}\label{step-size-nonconvex} 
Assume that there exists $C_0, C>0$ and $C_{\psi}, C_{f}, \gamma \in (0,1)$ such that for all $n \in \mathbb{N}_0$,
\begin{itemize}
    \item[(i)]
    $C_{0}< \alpha_{\varphi, n} < \min \left( \frac{1}{L}, \frac{1}{\beta} \right)$,
    \item[(ii)]
    $
    L 
    + L^2 L_{\beta}\alpha_{\psi,n}^2 
    + L^3 (1+L)(1+L_{\beta}) \alpha_{\psi,n}^2
    + L^2 L_{\beta}\alpha_{f,n}^2
    + \frac{L^2 (1+L_{\beta})\alpha_{f,n}^2}{2} 
    \leq C
    $,
    \item[(iii)]
    $
    (1-\beta^2 \alpha_{\varphi, n-1}^2) 
    + \frac{2 L^2}{\beta^2} \left(1 + \frac{1}{\beta C_0} \right)(\alpha_{\psi,n-1}^2+\alpha_{f,n-1}^2)
    \leq \gamma
    $,
    \item[(iv)]
    $
    C_{\psi}
    \leq
    1 
    - \frac{L\alpha_{\psi,n}}{2} 
    - L (1+L)(1+L_{\beta})\alpha_{\psi,n} 
    - L_{\beta} \alpha_{\psi,n} 
    - \frac{2 L^2 C \alpha_{\psi,n}}{\beta^2(1-\gamma)} \left(1 + \frac{1}{\beta C_0} \right)
    $,
    \item[(v)]
    $
     C_{f}
    \leq
    1 
    - \frac{L\alpha_{f,n}}{2}
    - \frac{L (1+L_{\beta})\alpha_{f,n}}{2} 
    - L_{\beta}\alpha_{f,n}
    - \frac{2 L^2 C \alpha_{f,n}}{\beta^2(1-\gamma)} \left(1 + \frac{1}{\beta C_0} \right)
    $,
\end{itemize}
where we denote by $L_{\beta}:=L\left( \frac{L}{\beta} + 1 \right)$.
\end{assumption}
These assumptions imposes small step sizes $\alpha_{\psi, n}, \alpha_{f, n}, \alpha_{\varphi, n}$, depending on constants $L$ and $\beta$. 
%This smallness is sometimes assumed in the case of nonconvex-concave minimax problem. See e.g., \citet{huang2021efficient, lin2020gradient}.
Similar assumptions are often made in the context of nonconvex-concave minimax problems, as observed in works such as \citet{huang2021efficient, lin2020gradient}.

%
%%%%%%%%%%%%%%%%%%%%%%%%%%%%%%%%%%%%%%%%%%%%%%%%%%%%%%%%%%%%%%%%
%The following is our main theorem in this section:
Building upon the background established above, we are ready to present our main theorem of this section:
%%%%%%%%%%%%%%%%%%%%%%%%%%%%%%%%%%%%%%%%%%%%%%%%%%%%%%%%%%%%%%%%%
\begin{theorem}\label{main-theorem-convergence-nonconvex}
Let Assumptions~\ref{closed-subspaces-nonconvex}, \ref{strongly-concave-L-nonconvex}, \ref{existence-Gate-diff-nonconvex}, \ref{L-smooth-nonconvex}, and \ref{step-size-nonconvex} hold.
Let $\{(\psi_{n}, f_{n}, \varphi_{n})\}_{n\in \mathbb{N}_{0}} \subset \Spsic \times \Sfc \times \Svarc$ be the projected gradient descent defined in Definition~\ref{gradient-descent-nonconvex}.
Then, for $N \in \mathbb{N}$, we have
\begin{align}
    &\label{main-nonconvex-eq-1}
    \left\| \widehat{\nabla G}_{\psi, N} \right\|_{\Spsi}
    \leq 
    \widehat{C}
    \left(\sum_{n=0}^{N-1}\alpha_{\psi,n}\right)^{-1/2},
    \\
    &\label{main-nonconvex-eq-2}
    \left\| \widehat{\nabla G}_{f, N} \right\|_{\Sf}
    \leq 
    \widehat{C}
    \left(\sum_{n=0}^{N-1}\alpha_{f,n}\right)^{-1/2},
\end{align}
where
\[
\widehat{C}=
\left( 
G(\psi_{0}, f_{0}) - \inf_{(\psi, f) \in \Spsic \times \Sfc} G(\psi, f)
+\frac{C \| \Phi(\psi_0, f_0) -\varphi_0 \|_{\Spsi}^2}{1-\gamma}
\right)^{1/2}.
\]
Here, $\widehat{\nabla G}_{\psi, N}$ and $\widehat{\nabla G}_{f, N}$ are weighted averages given by
\begin{equation}
\widehat{\nabla G}_{\psi, N} := \frac{\sum_{n=0}^{N-1}\alpha_{\psi, n}\nabla G (\cdot, f_n)_{\psi_n}}{\sum_{n=0}^{N-1}\alpha_{\psi, n}}, 
\quad
\widehat{\nabla G}_{f,N} := \frac{\sum_{n=0}^{N-1}\alpha_{f, n}\nabla G(\psi_n, \cdot)_{f_n} }{\sum_{n=0}^{N-1}\alpha_{f,n}}.
%\label{averages}
\end{equation}
\end{theorem}
\begin{proof}
See Appendix~\ref{Appendix1.5}.
\end{proof}
The idea of the proof is to generalize \citet[Theorem 4.4]{lin2020gradient}, which studied the convergence of the nonconvex-concave minimax problem in the finite dimensional setting, to the infinite dimensional function spaces, and to generalize two variables to three variables.
Note that 
if step sizes are chosen as constants satisfying Assumption~\ref{step-size-nonconvex}, then the right-hand sides of (\ref{main-nonconvex-eq-1}) and (\ref{main-nonconvex-eq-2}) are expressed as
\[
\text{RHS of (\ref{main-nonconvex-eq-1}) and (\ref{main-nonconvex-eq-2})}
%\approx 
=O(1/\sqrt{N})
%\frac{\widehat{C}}{\sqrt{N}},
\]
which converges to zero as $N\to \infty$.
In other words, we have proved that the gradient decent defined by Definition~\ref{gradient-descent-nonconvex} converges to a stationary point.
Finally, we note that the order $\mathcal{O}(1/\sqrt{N})$ agrees with the result obtain by \citet[Theorem 4.4]{lin2020gradient}, though we have adopted the infinite dimensional setting.

%with an order of $\mathcal{O}(1/\sqrt{N})$. %Therefore, in this case, weighted averages~(\ref{averages}) provide an exact solution to the minimax problem.

%%%%%%%%%%%%%%%%%%%%%%%%%%%%%%%%%%%%%%%%%%%%%%%%%%%%%%%%%%%%%%%%%
%\section{Verification of Objective Function for GANs and UDAs}
\section{Examples of Relationship Between Objective Functions for GANs and UDAs and Assumptions for Convergence}
\label{Application}
In this section, we confirm that certain objective functions of GANs and UDAs fulfill the conditions for guaranteed convergence described in Section~\ref{sec:Minimax analysis}.
In Section~\ref{verifi-Convex-concave}, we will verify Assumptions~\ref{convex-concave-K}~(i) and \ref{L-smooth-relative} (i) \& (iii) for the problem (\ref{intro-minimax-eq}), and in Section~\ref{verifi-Nonconvex-concave} we will verify Assumptions~\ref{strongly-concave-L-nonconvex} and \ref{existence-Gate-diff-nonconvex} for the problem (\ref{intro-minimax-eq-nonconvex}) because we can immediately confirm that the remaining assumptions hold for our objective function. 

%%%%%%%%%%%%%%%%%%%%%%%%%%%%%%%%%%%%%%%%%%%%%%%%%%
\subsection{Convex-concave setting}\label{verifi-Convex-concave}
In this section, we will verify Assumptions~\ref{convex-concave-K}~(i) and \ref{L-smooth-relative} (i) and (iii) for the convex-concave setting (\ref{intro-minimax-eq}).
%%%%%%%%%%%%%%%%%%%%%%%%%%%%%%%%%%%%%%%%%%%%%%%%%% 
\subsubsection{Assumption~\ref{convex-concave-K} (i) (Joint convexity of $(\psi, \mu) \mapsto R(\psi, \mu)$)}
\label{subsec:joint_convexity}

%\Erasesaw{In this section, we only focus on UDAs, which can be regarded as GANs where the source risk is ignored.}
Assumption~\ref{convex-concave-K} requires the joint convexity of a source risk $R(\psi, \mu)=\int\ell(\psi, \psi_0) d\mu$ with respect to $(\psi, \mu) \in \cC(X)\times \cP(X)$ for the existence of a minimax solution. 
Here, $\mu$ is a marginal distribution of a source domain, $\psi$ is a predictor to be optimized for a task (which may be implemented as a neural network), and $\psi_0$ is the true predictor for a task.
Also, $\ell(\psi, \psi_0)$ denotes a loss function for a task in the source domain.
In general, it is obvious that the source risk does not possess the joint convexity.
Therefore, we need to introduce some regularization terms to the source risk such as

\begin{comment}
{\color{red}
[Reviewer tcXx. 6.In equation 9, does represents the regularizer? Can the author introduce the notation before-
hand? In order to make the assumption holds, the author claim the strongly convex regularizer
should be added. Can the author provide any simple experiment to justify regularizer helps con-
vergence for UDA setting.]
}
% strongly convex regularizerの説明を追加
\end{comment}

\begin{equation}
    R(\psi, \mu)=\int \ell(\psi, \psi_{0})d\mu + V(\psi)+ W(\mu),
\end{equation}
where $V : \cC(X) \to\bR$ and $W : \cP(X) \to\bR$ are regularization terms.
The next proposition gives the sufficient conditions of the joint convexity for the source risk:
\begin{proposition}\label{example-joint-convex}
Let $\left\| \cdot \right\|_{\mathcal{C}(X), 1}$ and $\left\| \cdot \right\|_{\mathcal{C}(X), 2}$ be norms in $\mathcal{C}(X)$.
Let $\rho, \gamma > 0$ with $\gamma \geq \rho$.
Let be a loss function $\ell: \cC(X)\times \cC(X)\to\bR$.
Assume that
    \begin{itemize}
        \item[(i)] $\psi \mapsto \ell(\psi, \psi_0)$ is convex for each $\psi_0 \in \cC(X)$.
        \item[(ii)] $\psi \mapsto \ell(\psi, \psi_{0})$ is $\rho$-Lipschitz with respect to $\left\| \cdot \right\|_{\mathcal{C}(X), 1}$ and $\left\| \cdot \right\|_{\mathcal{C}(X), 2}$ for any $\psi_0\in\cC(X)$, that is,
        \begin{equation*}
            \left\| \ell(\psi_1, \psi_{0}) - \ell(\psi_2, \psi_{0}) \right\|_{\mathcal{C}(X), 1}
            \leq \rho \left\| \psi_1 - \psi_2 \right\|_{\mathcal{C}(X), 2}, 
            \ \psi_1, \psi_2 \in \mathcal{C}(X).
        \end{equation*} 
        \item[(iii)] $V$ and $W$ are $\gamma$-strongly convex with respect to $ \left\| \cdot \right\|_{\mathcal{C}(X), 2}$ and $\left\| \cdot \right\|_{\mathcal{C}(X), 1}^{\star}$, respectively.
    \end{itemize}
    Then, the source risk $R(\psi, \mu)$ is joint convex with respect to $(\psi, \mu)$.
\end{proposition}
\begin{proof}
    See Appendix~\ref{example-joint-convex-app} for the proof.
\end{proof}

Assumption (i) and (ii) are the convexity and Lipschitz continuity for the loss function.
%The example for this loss function is a squared error loss. 
For example, the squared error loss satisfies these assumptions.
The example for (iii) is that $V=\frac{1}{2}\left\| \cdot \right\|_{\mathcal{H}_2}^{2}$ and $W=\frac{1}{2}\left\| \cdot \right\|_{\mathcal{H}_1}^{\star 2}$ where $(\mathcal{H}_i,\left\|\cdot \right\|_{\mathcal{H}_i} )$ is a reproducing kernel Hilbert space (RKHS) with a positive definite kernel $K_i : X \times X \to \mathbb{R}$.
%\citet{bietti2019kernel} used this RKHS norm as regularization for NNs.
Note that the dual norm $\left\| \cdot \right\|_{\mathcal{H}_1}^{\star}$ of the RKHS norm $\left\| \cdot \right\|_{\mathcal{H}_1}$ corresponds to a maximal mean discrepancy~(MMD).
As both $\left\| \cdot \right\|_{\mathcal{H}_2}$ and $\left\| \cdot \right\|_{\mathcal{H}_1}^{\star}$ are norms induced by inner products, %which implies that 
$V$ and $W$ are 1-strongly convex with respect to $\left\| \cdot \right\|_{\mathcal{H}_2}$ and $\left\| \cdot \right\|_{\mathcal{H}_1}^{\star}$, respectively.

\subsubsection{Assumption~\ref{L-smooth-relative} (i) (Smoothness of $\psi \mapsto \int \ell(\psi, \psi_{0})d\mu$)}
Assumption~\ref{L-smooth-relative}(i) demands that the source risk $R(\psi, \mu) = \int\ell(\psi, \psi_0)d\mu$ is $L$-smooth for $\psi\in \cC(X)$.
% because our discussion only restrict the source risk $R$.
%
Let us consider the following general functional $I_{h,\mu}: \cC(X) \to \exR$ for the later convenience:
\begin{equation*}
    I_{h,\mu}(\Psi):= \int h(\Psi(x)) d\mu(x), \ \Psi \in \mathcal{C}(X),
\end{equation*}
where $h: \mathbb{R} \to \mathbb{R}$ and $\mu \in \mathcal{P}(X)$. %are fixed. 
Then, the next lemma guarantees $L$-smoothness of $I_{h,\mu}(\Psi)$.
%%%%%%%%%%%%%%%%%%%%%%%%%%%%%%%%%%%
\begin{lemma}\label{H-smooth-lemma}
Let be $a,b \in [-\infty, \infty]$ and $h \in \cC^1(a, b)$.
Then, we denote
\[
S_{\mathcal{C},a,b}:=\{\psi \in \mathcal{C}(X) \ : \ a \leq \psi(x) \leq b, \ x \in X \}.
\]
Assume that $h:(a,b) \to \mathbb{R}$ is $L$-smooth, %in the sense of the real value,
that is,
\[
D_{h}(s|t)\leq \frac{L}{2}|s-t|^2 , \ s,t \in (a, b),
\]
where $D_{h}(s|t)=h(s) - h(t) -h^{\prime}(t) (s - t)$.
Then, $I_{h,\mu}$ is $L$-smooth with respect to $\| \cdot \|_{L^2(X, \mu)}$ over $S_{\mathcal{C},a,b}$.
\end{lemma}
%%%%%%%%%%%%%%%%%%%%%%%%%%%%%%%%%%%%%%%%%%%%%%%%
\begin{proof}
See Appendix~\ref{H-smooth-lemma-app} for the proof.
\end{proof}
%
%The above lemma provides the sufficient conditions to guarantee the convergence theorem when the functional $I_{h,\mu}$ is applied to the source risk.
%We can show that $\ell(\cdot, \psi_{0})$ is $L$-smooth, which implies that, by Lemma~\ref{H-smooth-lemma}, $\psi \mapsto \int \ell(\psi, \psi_{0})d\mu$ is $L$-smooth with respect to $\| \cdot \|_{L^2(X, \mu)}$.

If loss function $\ell(\cdot, \psi_{0})$ is $L$-smooth (e.g., squared error loss), then  by Lemma~\ref{H-smooth-lemma}, $\psi \mapsto R(\psi, \mu)= \int \ell(\psi, \psi_{0})d\mu$ is $L$-smooth with respect to $\| \cdot \|_{L^2(X, \mu)}$.

%We can also show that $I_{g_{\ell}, \mu}$ is strongly convex with respect to some norm $\left\| \cdot \right\|_{\mathcal{C}(X)}$, which is required in Assumption~\ref{strongly-convex-ass}, see Appendix~\ref{Examples for all} for details.
%
%\end{comment}

\begin{comment}
We note that, since $\mathcal{C}(X) \subset L^{p}(X)$ ($p\geq 1$) due to compactness of $X \subset \mathbb{R}^{d}$, we can introduce $L^p$-norm in $\mathcal{C}(X)$.
Especially, $L^2$-norm is \Add{the norm induced by inner product}, which implies that $D_{\frac{1}{2}\left\| \cdot \right\|_{L^2(X)}^{2}}(\psi|\varphi) =\left\| \psi-\varphi \right\|_{L^2(X)}^{2}$ %(\citep[Example 1]{aubin2022mirror}).
%inner product-induced norm
The function $h$ corresponds to $\ell(\cdot, \psi_{true})$.
If $\ell(\cdot, \psi_{true})$ is smooth in the sense of (\ref{h-smooth-squre}), then $\psi \mapsto \int \ell(\psi, \psi_{true})d\mu + \int \varphi d\mu - J^{\star}(\varphi)$ is smooth relative to $\frac{1}{2} \left\| \cdot \right\|_{L^2(X)}^{2}$.
\end{comment}

%%%%%%%%%%%%%%%%%%%%%%%%%%%%%%%%%%%%%%%%%%%%%%%%%%%%
%%%%%%%%%%%%%%%%%%%%%%%%%%%%%%%%%%%%%%%%%%%%%%%%%%%%
\subsubsection{Assumption~\ref{L-smooth-relative} (iii) (Smoothness of $\varphi \mapsto J^{\star}_{\nu_{0}}(\varphi)$)}

Assumption \ref{L-smooth-relative} (iii) imposes the $L$-smoothness condition on the convex conjugate $J^{\star}_{\nu_{0}}(\varphi)$ of the discrepancy measure $J_{\nu_0}(\mu)$ to ensure convergence for both GAN and UDA.
The representative discrepancy measures are $f$-divergence \citep{ali1966general, csiszar1967information} and integral probability metric (IPM) \citep{muller1997integral}, which respectively unify different divergences between probability measures with various applications such as GAN and UDA. 
The $f$-divergence includes Kullback-Liebler divergence, Jensen-Shannon divergence, and Pearson $\chi^2$ divergence, while the IPM includes Wasserstein-1 distance, Dudley metric, and maximum mean discrepancy. %Among these, the total variation distance is the only common one.

% sawa: relative smoothnessって単語、ここから複数回出てきますけど、使っていいんですか？これ以前では使っていないと思います。
In the subsequent, we provide several examples of $J^{\star}_{\nu_0}$ that satisfy the $L$-smoothness for (A) $f$-divergence and (B) IPM.
%Additionally, we discuss some implications.

%
\paragraph{(A) $f$-divergence}
Let $f : \mathrm{dom}_{f} \subset \mathbb{R}_{+} \to \mathbb{R}$ be a proper, lower semi-continuous and convex function. Then, 
the $f$-divergence $D_f(\mu|\nu)$ between $\mu\in\cP(X)$ and $\nu\in\cP(X)$ is defined as 
\begin{align}
    D_f(\mu|\nu):=
    \begin{cases}
    \int f\left(\frac{d\mu}{d\nu} \right)d\nu
    &\quad \text{if $\mu \ll \nu$} \\
    + \infty
    &\quad \text{otherwise}
    \end{cases},
    \label{$f$-divergence}
\end{align}
where $\mu \ll \nu$ denotes that $\mu$ is absolutely continuous with respect to $\nu$.
%Here, we use the convention $0\cdot(\pm \infty) = 0$.

The $f$-divergence is joint convex with respect to $(\mu, \nu)$ (as the mapping $(p,q) \mapsto qf(p/q)$ is joint convex) and non-negative for all $\mu$ and $\nu$ but not symmetric with respect to $\mu$ and $\nu$ in general.
In our case, we set $J_{f, \nu_0}(\mu) = D_f(\mu|\nu_0)$ with a fixed measure $\nu_0$ which implies a true distribution for GAN and a target distribution for UDA. 
The convergence theorem demands the $L$-smoothness for the convex conjugate $J^{\star}_{f, \nu_0}(\mu)$ of $J_{f, \nu_0}(\mu)$. 
%The convex conjugate $J^\star_{f, \nu_0}(\mu)$ is given in several papers with the purpose of finding the dual representations of $f$-divergence e.g., \cite{nguyen2007estimating, nguyen2010estimating, nowozin2016f}.

%The following lemma states the the convex conjugate of $J_{f, \nu_0}(\mu)$:
The following lemma provide the representation of the convex conjugate $J^{\star}_{f, \nu_0}(\varphi)$ of $J_{f, \nu_0}(\mu)$:
%%%%%%%%%%%%%%%%%%%%%%%%%%%%%%%%%%%%%%%%%%%%%%%%%
\begin{lemma}\label{$f$-divergence-conjugate}
    Assume that $f \in C^{1}(\mathrm{dom}_{f})$, and there exists the inverse $(f^{\prime})^{-1}$ of $f^{\prime}$.
    Then, the convex conjugate $J_{f, \nu_0}^{\star}$ of $J_{f, \nu_0}$ is given by
    $J_{f, \nu_0}^{\star}(\varphi)= \int f^\star\circ \varphi d\nu_0$
    for $\varphi \in S_{\mathcal{C},f}:=\{ \varphi \in \mathcal{C}(X) \ : \ \varphi(x) \in \mathrm{dom}_{(f^{\prime})^{-1}} \}$, where 
    \[
    f^{\star}(s)
    :=\sup_{t} \{ st - f(t)\}
    =s\cdot(f^{\prime})^{-1}(s) - f\circ (f^{\prime})^{-1}(s), \ s \in \mathrm{dom}_{(f^{\prime})^{-1}}.
    \]
\end{lemma}
%%%%%%%%%%%%%%%%%%%%%%%%%%%%%%%%%%%%%%%%%%%%%%%%
\begin{proof}
See Appendix~\ref{$f$-divergence-conjugate-app} for the proof.
\end{proof}

In the context of $f$-divergence, it is sufficient to confirm the smoothness of the convex conjugate $f^\star$ in the sense of the real function.
We can take Jensen-Shannon divergence and Pearson $\chi^2$ divergence as  examples and confirm that $J^\star_{f, 
\nu_0}(\mu)$ satisfies the $L$-smoothness.
%with respect to the squared $L^2$-norm. 
\begin{example}[Jensen–Shannon divergence]
    The Jensen-Shannon divergence is defined as
    \begin{equation}
        D_{\JS}(\mu|\nu)
        := \frac{1}{2}D_{\KL}(\mu |\rho) + \frac{1}{2}D_{\KL}(\nu|\rho),
    \end{equation}
    where $\rho = (\mu + \nu)/2$, and $D_{\KL}(\mu|\nu)$ is the Kullback–Leibler divergence between $\mu$ and $\nu$ defined by
    \[
    D_{\KL}(\mu|\nu)
    = \int \frac{d \mu}{d\nu} \log \frac{d \mu}{d\nu} d \nu.
    \]
    Here, $f(t)$ is represented as
    \begin{equation}
    f_\JS(t):=-\frac{1}{2}(t+1)\log\left(\frac{1+t}{2}\right) + \frac{1}{2}t\log t,  \ t \in (0, \infty).
    \end{equation}
    The convex conjugate $f^\star_\JS(s)$ is 
    \begin{equation}
        f^\star_{\JS}(s)
        = -\frac{1}{2} \log(1-\frac{1}{2}e^{2s}) - \frac{1}{2}\log 2,
        \quad s \in (-\infty, \frac{1}{2}\log 2).
    \end{equation}
    The convex conjugate $f^\star_{\JS}$ is $L$-smooth over $(a,b)$ with some $L>0$, and $a,b \in (-\infty, \frac{1}{2} \log 2)$.
    Therefore, by using Lemma~\ref{H-smooth-lemma}, $J_{f_{\JS}, \nu_{0}}^{\star}$ is $L$-smooth with respect to $\| \cdot \|_{L^2(X, \mu)}$.
\end{example}

%%%%%%%%%%%%%%%%%%%%%%%%%%%%%%%%%%%%%%%%%%%
\begin{example}[Pearson $\chi^2$ divergence]
\label{Pearson divergence-app}
The Pearson $\chi^2$ divergence is defined as
\[
D_{\chi^2}(\mu|\nu):=\int (\frac{d\mu}{d\nu} -1)^2 d\nu.
\]
$f(t)$ is represented as 
\[
f_{\mathrm{P}}(t) :=(t-1)^2, \ t \in \mathbb{R}.
\]
The convex conjugate $f_{\mathrm{P}}^{\star}(s)$ is 
\[
f_{\mathrm{P}}^{\star}(s)=\frac{1}{4}s^2 +s, \ s \in \mathbb{R}.
\]
The convex conjugate $f^\star_{P}$ is $L$-smooth over $(a,b)$ with some $L>0$, and $a,b \in \mathbb{R}$.
Therefore, by using Lemma~\ref{H-smooth-lemma}, $J_{f_{\mathrm{P}}, \nu_{0}}^{\star}$ is $L$-smooth with respect to $\| \cdot \|_{L^2(X, \mu)}$.
\end{example}

\paragraph{(B) Integral Probability Metric (IPM)}
% Let $\cF$ be a %class of functions on $X$.
Let $\cF$ be a class of real-valued bounded measurable functions on $X$. 
The IPM associated with $\cF$ is defined as
\begin{equation}
    d_{\cF}(\mu, \nu) := \sup_{g \in \cF} \left\{\abs{\int g d\mu -\int g d\nu}\right\}
\label{IMP-def-def}
\end{equation}
for all pairs of measures $(\mu, \nu) \in \cP(X)\times\cP(X)$ such that all functions in $\cF$ are absolutely $\mu$- and $\nu$-integrable.
The typical examples are Wasserstein-1 distance for $\cF=\{g \in \mathrm{Lip}(X):\left\|g \right\|_{\Lip} \leq 1 \}$ 
where $\mathrm{Lip}(X)$ is the class of the real-valued Lipschitz functions on $X$ and $\left\| \cdot \right\|_{\Lip}$ is the Lipschitz norm,
and the MMD for $\cF=\{g \in \mathcal{H}:\left\|g \right\|_{\mathcal{H}} \leq 1 \}$ where $(\mathcal{H},\left\|\cdot \right\|_{\mathcal{H}} )$ is an %reproducing kernel Hilbert space 
RKHS with a positive definite kernel $K : X \times X \to \mathbb{R}$.
In our case, we set $J_{\mathrm{IPM}, \nu_0}(\mu) = d_{\cF}(\mu, \nu_{0}) $ with a fixed measure $\nu_0$.
Then, we can obtain the following Lemma.

\begin{lemma}\label{IPMs-conjugate}
Assume that $\cF$ include the zero function.
    Then, the convex conjugate $J_{\IPM, \nu_0}^{\star}$ of $J_{\mathrm{IPM}, \nu_0}$ is given by
    \[
    J_{\IPM, \nu_0}^{\star}(\varphi)= \int \varphi d \nu_{0} + \chi\{ \varphi \in \mathcal{F} \},
    \]
    where the indicator function is give by
    \begin{equation}
    \chi\{ A \}
    := 
    \begin{cases}
        0 &\quad \text{if $A$ is true} \\
        \infty &\quad \text{if $A$ is false}
    \end{cases}.
    \end{equation}
\end{lemma}
\begin{proof}
See Appendix~\ref{IPMs-conjugate-app} for the proof.
\end{proof}

From Lemma~\ref{IPMs-conjugate}, the differential $d(J_{\IPM, \nu_0}^{\star})_{\varphi}$ of $J_{\IPM, \nu_0}^{\star}$ at $\varphi \in \cF$ is given by
\[
d(J_{\IPM, \nu_0}^{\star})_{\varphi}(\lambda)=J_{\IPM, \nu_0}^{\star}(\lambda), \quad \lambda \in \mathcal{C}(X),
\]
which implies that, by Definition~\ref{Bregman divergences over measure and continuous function spaces}
\[
D_{J_{\IPM, \nu_0}^{\star}}(\psi|\varphi)
= 0, \quad \psi,\varphi \in \cF.
\]
%From Lemma~\ref{IPMs-conjugate}, $\varphi \mapsto J_{\IPM, \nu_0}^{\star}(\varphi)$ is linear %functional over for $\varphi \in \mathcal{F}$.
%Since the Bregman divergence for linear mappings is identically zero by Definition~\ref{Bregman divergences over measure and continuous function spaces}, 
Therefore, we obtain the following proposition. 
\begin{proposition}
For any $L>0$, and any norm $\| \cdot \|_{\cC(X)}$ induced by inner products, 
$J_{\IPM, \nu_0}^{\star}$ is $L$-smooth with respect to $\| \cdot \|_{\cC(X)}$ over $\mathcal{F}$.
\end{proposition}
%Therefore, by restricting subset $\Spsi$ in $\mathcal{G}$, we obtain that $\varphi \mapsto J_{\IPM, \nu_0}^{\star}(\varphi)$ is $L$-smooth over $\Svar$.

%%%%%%%%%%%%%%%%%%%%%%%%%%%%%%%%%%%%%%%%%%%%%%%%%%%%
%%%%%%%%%%%%%%%%%%%%%%%%%%%%%%%%%%%%%%%%%%%%%%%%%%%%
\subsubsection{Examples of $\mathcal{K}(\psi, \mu, \varphi)$ simultaneously satisfying all assumptions in Section~\ref{Minimax analysis}}
\label{sec:all-ex}
In this section, we provide an example of our objective function (\ref{intro-minimax-eq}) to simultaneously satisfy all the assumptions (Assumptions~\ref{convex-concave-K} and \ref{L-smooth-relative}) for the convex-concave structure and smoothness.

Let $\psi_{0} \in \mathcal{C}(X)$ be a true predictor, and let $\nu_{0} \in \mathcal{P}(X)$ be a true distribution.
We consider an objective function $\mathcal{K}_{1}:\Spsi \times \Smu \times \Svar \to \overline{\mathbb{R}}$ defined as
\[
\begin{split}
\mathcal{K}_{1}(\psi, \mu, \varphi)
&
:=\frac{1}{2}\int (\psi - \psi_0)^2 d\mu + \frac{\gamma}{2} \left\| \psi \right\|_{\mathcal{H}_{2\sqrt{2}\sigma}}^{2}  + \frac{\gamma}{2} \left\| \mu \right\|_{\mathcal{H}_{\sigma}}^{\star 2}  
+ \int\varphi d\mu - \int k(\varphi) d\nu_{0},
\end{split}
\]
which corresponds to the problem (\ref{intro-minimax-eq}) with $R(\psi, \mu)=\frac{1}{2} \int (\psi- \psi_{0})^2 d\mu + \frac{\gamma}{2} \left\| \psi \right\|_{\mathcal{H}_{2\sqrt{2}\sigma}}^{2}  + \frac{\gamma}{2} \left\| \mu \right\|_{\mathcal{H}_{\sigma}}^{\star 2} $ and $J_{\nu_0}$ is either IPMs or $f$-divergences. 
Here, $\gamma>0$ is a regularization parameter, and $k:(a,b) \to \mathbb{R}$ is a convex and $C^1$-function with some $a,b \in \overline{\mathbb{R}}$, introduced to encompass more general situations including both IPMs and $f$-divergences. 
If the function $k$ takes the form $k(s)=s$, then %the divergence $J$ 
the discrepancy measure $J_{\nu_0}$
corresponds to IPMs.
If the function $k$ takes the form $k(s)=f^{\star}(s)$, then %the divergence $J$
the discrepancy measure $J_{\nu_0}$ corresponds to $f$-divergences.
Also, $(\mathcal{H}_{\sigma},\left\|\cdot \right\|_{\mathcal{H}_{\sigma}})$ is a RKHS with Gaussian kernel $K_{\sigma}(x,y)=(2\pi \sigma^2)^{-d/2}e^{-|x-y|^{2}/2\sigma^2 }$ with variance $\sigma^2$, 

We choose convex subsets $\Spsi$, $\Smu$, and $\Svar$ as
\begin{align*}
\Spsi
%S_{\mathcal{C}^{\infty}}
&:= 
\{ \psi \in \mathcal{C}^{\infty}_{0}(X) : \left\|\partial_{x}^{\alpha} \psi \right\|_{L^{\infty}(\mathbb{R}^d)} \leq C_{b} \ \text{ for all } \alpha \in \mathbb{N}_{0}^{d} \},
\\
\Smu
%S_{\mathcal{M}, \mu_{u}}
&:=\{ \mu \in \mathcal{P}(X) \ : \ \mu(A) \leq \mu_u(A) \text{ for all measurable sets $A$ in $X$}
%\footnote{For two measures $\mu$ and $\nu$, we write $\mu \leq \nu$ if $\mu(A) \leq \nu(A)$ for all measurable sets $A$ in $X$} 
\},\\
\Svar
%S_{\mathcal{C},a,b}^{\prime}
&:=\{\varphi \in \mathcal{F} \ : \ a \leq \varphi(x) \leq b, \ x \in X \},
\end{align*}
where 
$\mathcal{C}^{\infty}_{0}(X)$ is the space of $\mathcal{C}^{\infty}$ functions with compact support in $X$, $\mathcal{F}$ is a subset in $\mathcal{C}(X)$, and $\mu_u \in \cM^+(X)$ is a non-negative measure.
Then, Proposition~\ref{all-ex-prop} is obtained if the following assumption is satisfied.

%%%%%%%%%%%%%%%%%%%%%%%%%%%%%%%%
\medskip
\begin{assumption}
\label{ass-all-ex}
We assume the following:
\begin{itemize}
\item $\sigma<\frac{1}{2}$.
%%%%%%%%%%%%%
\item $\gamma \geq 4 C_{b}^2 C_{\sigma}^d$ where $C_{\sigma}:=\sum_{j \in \mathbb{N}_{0}}
(4\sigma^2)^{j}$.
%%%%%%%%%%%%%
\item $\psi_0 \in \Spsi$ and $\nu_0 \in \Smu$.
%%%%%%%%%%%%%
\item $k:(a,b) \to \mathbb{R}$ is $L_{k}$-smooth in the sense of the real function.
%$k:(a,b) \to \mathbb{R}$ is convex, $C^1$, and $L_{k}$-smooth relative to a convex and $C^1$ function $k_{0}:(a,b) \to \mathbb{R}$.
\end{itemize}
\end{assumption}
\medskip
\begin{proposition}\label{all-ex-prop}
Let Assumption~\ref{ass-all-ex} hold. 
Then, the following statements hold:
\begin{itemize}
%%%%%%%%%%%%%%%%%%%%
\item[(1)] [Assumption~\ref{convex-concave-K} (i)] $(\psi, \mu) \mapsto \mathcal{K}_{1}(\psi, \mu, \varphi)$ is convex.
%%%%%%%%%%%%%%%%%%%%
\item[(2)] [Assumption~\ref{convex-concave-K} (ii)] $\varphi \mapsto \mathcal{K}_{1}(\psi, \mu, \varphi)$ is concave.
%%%%%%%%%%%%%%%%%%%%
\item[(3)] [Assumption~\ref{L-smooth-relative} (i)] $\psi \mapsto \mathcal{K}_{1}(\psi, \mu, \varphi)$ is $1$-smooth with respect to 
$$ \left( \frac{1}{2} \left\| \cdot \right\|_{L^2(X, \mu_u)}^{2} + \frac{\gamma}{2} \left\| \cdot \right\|_{\mathcal{H}_{2\sqrt{2}\sigma}}^{2}\right)^{1/2}.
$$
%%%%%%%%%%%%%%%%%%%%
\item[(4)] [Assumption~\ref{L-smooth-relative} (ii)] $\mu \mapsto \mathcal{K}_{1}(\psi, \mu, \varphi)$ is $1$-smooth with respect to $\left(\frac{\gamma}{2} \right)^{1/2} \left\| \cdot \right\|_{\mathcal{H}_{\sigma}}^{\star}$.
%%%%%%%%%%%%%%%%%%%%
\item[(5)] [Assumption~\ref{L-smooth-relative} (iii)] $\varphi \mapsto -\mathcal{K}_{1}(\psi, \mu, \varphi)$ is $L_{k}$-smooth with respect to $\| \cdot \|_{L^2(X, \mu_u)}$.
%%%%%%%%%%%%%%%%%%%%
\end{itemize}
\end{proposition}
\begin{proof}
See Appendix~\ref{proof-all-ex}.
\end{proof}
% sawa: ここの設定だとMain Theoremを満たすことができると追加。間違っていたら修正してください
Note that we can immediately confirm that the remaining assumptions hold for our objective function.
That is, Theorem~\ref{main-theorem-convergence} is satisfied by the setting of this section.

% 
%%%%%%%%%%%%%%%%%%%%%%%%%%%%%%%%%%%%%%%%%%%%%%%%%%
\subsection{Nonconvex-concave setting}\label{verifi-Nonconvex-concave}
In this section, we will verify Assumptions~\ref{strongly-concave-L-nonconvex} and \ref{existence-Gate-diff-nonconvex} for the nonconvex-concave setting (\ref{intro-minimax-eq-nonconvex}).
\subsubsection{Assumption~\ref{strongly-concave-L-nonconvex} (Strong convexity of $\varphi  \mapsto  J^{\star}_{\nu_{0}}(\varphi)$)}
\label{sub:strongly-concave-L-nonconvex}
Assumption~\ref{strongly-concave-L-nonconvex} requires the strong concavity for the problem (\ref{intro-minimax-eq-nonconvex}) with respect to $\varphi$. 
This requirement is equivalent to the strong convexity of discrepancy measure $J_{\nu_0}$.
We realize this by adding some regularization to the discrepancy measure.

Assume that the discrepancy measure $J_{\nu_0}$ is proper, lower semi-continuous, and convex.
We define inf-convolution $J_{\nu_0} \oplus \mathcal{I}(\mu)$ as
\begin{align*}
J_{\nu_0}\oplus \mathcal{I}(\mu):= \inf_{\xi \in \cM(X)} 
J_{\nu_0}(\xi)+\mathcal{I}(\mu-\xi) 
\end{align*}
where $\mathcal{I} : \cM(X) \to \overline{\mathbb{R}}$ is a regularizer function, which is proper, lower semi-continuous, and convex. 
Then, the following Lemma holds.

\begin{lemma}\label{strong-convex-disc}
Let $\beta >0$ and $\| \cdot \|_{\cM(X)}$ be a norm induced by an inner product in $\cM(X)$.
Then, if $\mathcal{I} : \cM(X) \to \mathbb{R}$ is $(1/\beta)$-smooth with respect to $\| \cdot \|_{\cM(X)}$, then $(J_{\nu_0} \oplus \mathcal{I})^{\star}: \cC(X) \to \mathbb{R}$ is $\beta$-strongly convex with respect to $\| \cdot \|_{\cM(X)}^{\star}$.
\end{lemma}
\begin{proof}
See Appendix~\ref{strong-convex-disc-app} for the proof.
\end{proof}
% sawa: Contributionsの直下の文から引っぱってきました。こんな感じの一言がないと、この節のしまりが悪い（だから何？ってなる）と思います。内容が間違っている場合は修正ください
Thanks to this Lemma, we can attain the strong convexity by introducing the regularizer, such as the squared MMD in the RKHS with the Gaussian kernel, which corresponds to the gradient penalty \citep{gulrajani2017improved}.

%%%%%%%%%%%%%%%%%%%%%%%%%%%%%%%%%%%%%%%%%%%%%%%%%%%%%
%
\subsubsection{Assumption~\ref{existence-Gate-diff-nonconvex} (G\^{a}teaux differentiability with respect to $f$)}

Here, we consider the case when the source risk has the form $R(\psi, \mu)= \int \ell(\psi, \psi_0)d\mu$
where the loss function $\ell(\psi, \psi_0)$ is convex with respect to $\psi$.
Then, the minimax problem~(\ref{intro-minimax-eq-nonconvex}) is translated into 
\begin{equation}
    \min_{\psi \in \mathcal{C}(X)} 
    \min_{f \in \mathcal{C}(Z;X)} 
    \max_{\varphi \in \mathcal{C}(X)} 
    \int \ell(\psi \circ f, \psi_0 \circ f)d\mu_0 
    + \int \varphi \circ f d\mu_0 
    - J_{\nu_0}^{\star}(\varphi).
    % \label{intro-minimax-eq-nonconvex}
\end{equation}
It is obvious that the above objective function is G\^{a}teaux differentiable with respect to $\psi$ and $\varphi$ because the above objective function is  convex and concave for $\psi$ and $\varphi$, respectively. 
As functions $\psi$ and $\varphi$ are composed with $f$, some regularity for $\psi$ and $\varphi$ is required to hold G\^{a}teaux differentiable with respect to $f$. 

Let us consider the following general functional $\mathcal{J}_{h, \xi}:\cC(Z; X) \to \mathbb{R}$ :
\[
\mathcal{J}_{h, \xi}(f):= \int h \circ f d \xi, \quad f \in \cC(Z; X),
\]
where $h \in \cC(X)$ and $\xi \in \cP(Z)$ are fixed.
Then, the following lemma guarantees the G\^{a}teaux differentiability of $\mathcal{J}_{h, \xi}$.
\begin{lemma}\label{Gateaux-diff-f}
Assume that $h \in \mathrm{Lip}(X)$, $f \in \cC(Z;X)$, and $\mu \ll m$ where $m$ is the Lebesgue measure. 
Then, $\mathcal{J}_{h, \xi}$ is G\^{a}teaux differentiable at $f$.
Furthermore, its G\^{a}teaux differential is given by $d(\mathcal{J}_{h, \xi})_{f}$
\begin{align*}
    d(\mathcal{J}_{h, \xi})_{f}(g)
    = 
    \int (\nabla h \circ f) \cdot g d\xi.
\end{align*}
\end{lemma}
\begin{proof}
See Appendix~\ref{Gateaux-diff-f-app} for the proof.
\end{proof}
%
% sawa: 5.4章から引っ張ってきました。ここで説明しないと「このLemmaは何なのか？機械学習の何に対応するのか？」がわからないまま進んでいき、苦痛を感じるかと思います。
Thanks to this Lemma, we can attain the Lipschitzness interpreted as applying the spectral normalization~\citep{miyato2018spectral}, a widely-used stabilization technique for GANs.

%%%%%%%%%%%%%%%%%%%%%%%%%%%%%%%%%%%%%%%%%%%%%%%%%%%%%%%%
%
\subsubsection{Examples of $\mathcal{G}(\psi, f, \varphi)$ simultaneously satisfying all assumptions in Section~\ref{Minimax analysis-nonconvex}}
In this section, we provide an example of our objective function (\ref{intro-minimax-eq-nonconvex}) to simultaneously satisfy all the assumptions (Assumptions~\ref{strongly-concave-L-nonconvex}, \ref{existence-Gate-diff-nonconvex} and \ref{L-smooth-nonconvex}).

Let $\psi_{0} \in \mathcal{C}(X)$ be a true predictor, $\nu_{0} \in \mathcal{P}(X)$ be a true distribution, and $\beta>0$. 
%Let $\beta>0$, and let $\psi_{0} \in \mathcal{C}(X)$ be a true predictor, and let $\nu_{0} \in \mathcal{P}(X)$ be a true distribution. 
We then consider an objective function $\mathcal{G}_1:\Spsi \times \Sf \times \Svar \to \overline{\mathbb{R}}$ defined as
\[
\begin{split}
\mathcal{G}_{1}(\psi, f, \varphi)
&
:=
\frac{1}{2}\int (\psi \circ f - \psi_{0} \circ f)^2 d\xi_0 
+ 
\int \varphi \circ f d\xi_{0} 
- \int k(\varphi) d\nu_{0} 
- \frac{\beta}{2} \| \varphi \|_{\mathcal{H}_{\sigma}}^{2},
\end{split}
\]
which corresponds to the problem (\ref{intro-minimax-eq-nonconvex}) with $R(\psi, \mu)=\int (\psi- \psi_{0})^2 d\mu$, and discrepancy measure $J_{\nu_0}$ is replaced with inf-convolution $J_{\nu_0} \oplus (\frac{1}{2 \beta} \| \cdot \|_{\mathcal{H}_{\sigma}^{ \star 2}})$ where $J_{\nu_0}$ is either IPMs or $f$-divergences. 
Here, $(\mathcal{H}_{\sigma},\left\|\cdot \right\|_{\mathcal{H}_{\sigma}})$ is a RKHS with Gaussian kernel $K_{\sigma}(x,y)=(2\pi \sigma^2)^{-d/2}e^{-|x-y|^{2}/2\sigma^2 }$ with variance $\sigma^2$.
In the same way of Section~\ref{sec:all-ex}, we introduce $k:(a,b) \to \mathbb{R}$, which is a convex and $C^1$-function with some $a,b \in \overline{\mathbb{R}}$, to encompass more general situations including both IPMs and $f$-divergences. 
%$(\mathcal{H}_{\sigma},\left\|\cdot \right\|_{\mathcal{H}_{\sigma}})$ is a RKHS with Gaussian kernel $K_{\sigma}(x,y)=(2\pi \sigma^2)^{-d/2}e^{-|x-y|^{2}/2\sigma^2 }$ with variance $\sigma^2$.

We choose norms $\| \cdot \|_{\Spsi}$, 
$\| \cdot \|_{\Sf}$, and 
$\| \cdot \|_{\Svar}$ as 
\begin{align}
\| \cdot \|_{\Spsi}:= \|\cdot \|_{H^{1}(X)},
\quad
\| \cdot \|_{\Sf}:= \|\cdot \|_{L^{2}(Z;X, \xi_0)},
\quad
\| \cdot \|_{\Svar}:= \|\cdot \|_{\mathcal{H}_{\sigma}},
\end{align}
and subset $\Spsi$, $\Spsic$, $\Sf$, $\Sfc$, $\Svar$, and $\Svarc$ as
\begin{align*}
\Spsi
&:= 
\overline{
\left\{ \psi \in H^{1}(X) : \psi \text{ and } \nabla \psi \text{ are Lipschitz continuous} 
\right\}
}^{\| \cdot \|_{\Spsi}}, 
\\
\Spsic
&:=
\left\{ 
\psi \in \Spsi : \mathrm{Lip}(\psi), \mathrm{Lip}(\nabla \psi) \leq C_1, 
\quad 
\sup_{x\in X}|\psi(x)|, \sup_{x\in X}\left|\nabla \psi(x)\right| \leq C_2
\right\},
\\
\Sf
&:=\overline{
\left\{ f \in \cC(Z;X) : \|f\|_{L^2(Z;X, \xi_0)}<\infty, \ f_{\sharp}\xi_0 \ll m, \
\sup_{x}\left| \frac{d(f_{\sharp}\xi_0)}{dm} \right| < \infty \right\}
}^{\| \cdot \|_{\Sf}},
\\
\Sfc
&:=\left\{f \in \Sf : \sup_{x}\left| \frac{d(f_{\sharp}\xi_0)}{dm} \right| \leq C_3 \right\},
\\
\Svar
&:= 
\overline{
\left\{ \varphi \in \cC_{0}^{\infty}(X) \cap \mathcal{H}_{\sigma} \cap \mathcal{F} : \varphi \text{ and } \nabla \varphi \text{ are Lipschitz continuous} 
\right\}
}^{\| \cdot \|_{\Svar}}, 
\\
\Svarc
&:=
\left\{ 
\varphi \in \Svar  : \mathrm{Lip}(\nabla \varphi) \leq C_4, 
\ a \leq \varphi(x) \leq b, \ x \in X
\right\},
\end{align*}
with some constants $C_1,C_2,C_3,C_4>0$, where $\mathrm{Lip}(\psi)$ is the Lipschitz constant for function $\psi$, and $\mathcal{F}$ is a subset in $\mathcal{C}(X)$.
Then, Proposition~\ref{all-ex-nonconvex-prop} is obtained if the following assumption is satisfied.
%%%%%%%%%%%%%%%%%
%%%%%%%%%%%%%%%%%%%%%%%%%%%%%%%%
\medskip
\begin{assumption}
\label{ass-all-ex-nonconvex}
We assume the following:
\begin{itemize}
%%%%%%%%%%%%%
\item $\psi_0 \in \Spsic$.
%%%%%%%%%%%%%
\item The derivative $k^{\prime}$ of $k$ is $L_{k}$-Lipschitz continuous.
%%%%%%%%%%%%%
\item $\nu \ll m$, and $\sup_{x \in X} \left|\frac{d \nu_0}{dm}(x) \right| < \infty$.
\end{itemize}
\end{assumption}
\medskip
\begin{proposition}\label{all-ex-nonconvex-prop}
Let Assumption~\ref{ass-all-ex-nonconvex} hold. 
Then, the following statements hold:
\begin{itemize}
%%%%%%%%%%%%%%%%%%%%
\item[(1)] [Assumption~\ref{strongly-concave-L-nonconvex}] $\varphi \mapsto \mathcal{G}_{1}(\psi, f, \varphi)$ is $\beta$-strongly concave with respect to $\|\cdot \|_{\Svar}$.
%%%%%%%%%%%%%%%%%%%%
\item[(2)] [Assumption~\ref{existence-Gate-diff-nonconvex}] $\mathcal{G}_{1}(\psi, f, \varphi)$ is G\^{a}teaux differentiable for each variable.
%%%%%%%%%%%%%%%%%%%%
\item[(3)] [Assumption~\ref{L-smooth-nonconvex}] $\mathcal{G}_{1}(\psi, f, \varphi)$ satisfies the condition~(\ref{L-smooth-eq-nonconvex}).
%%%%%%%%%%%%%%%%%%%%
\end{itemize}
\end{proposition}
\begin{proof}
See Appendix~\ref{all-ex-nonconvex-prop-app} for the proof.
\end{proof}
%
% sawa: ここの設定だとMain Theoremを満たすことができると追加。間違っていたら修正してください
As well as Section~\ref{sec:all-ex}, the remaining assumptions hold for our objective function.
Therefore, Theorem~\ref{main-theorem-convergence-nonconvex} is satisfied by the setting of this section.

%%%%%%%%%%%%%%%%%%%%%%%%%%%%%%%%%%%%%%%%%%%%%%%%%%
%
%
\subsection{Interpretations of our analysis}
Throughout Sections~\ref{verifi-Convex-concave} and \ref{verifi-Nonconvex-concave}, %previous sections, 
we have verified that certain objective functions for ideal settings of GANs and UDAs satisfy the sufficient conditions %necessary 
for the convergences discussed in Section~\ref{sec:Minimax analysis}. 
Both objective functions for GANs and UDAs involve the discrepancy measure, and its convex conjugate need to be strongly convex and $L$-smooth. 

An example for achieving strong convexity is through the inf-convolution with a discrepancy measure $J_{\nu_0}$ and a regularizer such as the squared MMD $\| \cdot \|_{\mathcal{H}_{\sigma}}^{\star 2}$ in the RKHS $\mathcal{H}_{\sigma}$ with Gaussian kernel $K_{\sigma}(x,y)=(2\pi \sigma^2)^{-d/2}e^{-|x-y|^{2}/2\sigma^2 }$ having variance $\sigma^2$ (Lemma~\ref{strong-convex-disc}). 
The convex conjugate of this inf-convolution can be expressed as
\[
(J_{\nu_0} \oplus \| \cdot \|_{\mathcal{H}_{\sigma}}^{\star 2})^{\star}(\varphi)
= J_{\nu_0}^{\star}(\varphi) +  \| \varphi \|_{\mathcal{H}_{\sigma} }^2.
\]
Also, the RKHS norm $\|\varphi \|_{\mathcal{H}_{\sigma}}$ in this equation is represented as (\citet[Proposition 14]{chu2020smoothness}) 
\[
\left\| \varphi \right\|_{\mathcal{H}_{\sigma}}^2
=
\sum_{k=0}^{\infty}(\frac{1}{2}\sigma^2)^{k} 
\sum_{|\alpha|=k}
\frac{1}{\alpha!}
\left\| \partial^{\alpha}_{x} \varphi \right\|_{L^2}^2,
\]
and minimizing this RKHS norm involves constraining the gradient of discriminator to be small.
This can be interpreted as applying gradient penalties \citep{gulrajani2017improved}, common stabilization techniques in adversarial training, to penalize gradients with large norm values.
Note that the gradient penalty \citep{gulrajani2017improved} is a regularization technique to add the gradient norm 
$
\mathbb{E}_{x \sim \mathbb{P}}[|\nabla \varphi(x) - 1 |^2]
$
to the discriminator's loss function.

On the other hand, when considering the discrepancy measure as IPMs, the convex conjugate of IPMs is given by
\[
J_{\IPM, \nu_0}^{\star}(\varphi)= \int \varphi d \nu_{0} + \chi\{ \varphi \in \mathcal{F} \},
\]
which is $L$-smoothness for $\varphi \in \mathcal{F}$ (Lemma~\ref{IPMs-conjugate}). 
The function class $\mathcal{F}$ should be the subset of Lipschitz continuous function spaces $\mathrm{Lip}(X)$ due to the G\^{a}teaux differentiability of objective functions in the nonconvex-concave problem (\ref{Minimax-G-eq}) (Lemma~\ref{Gateaux-diff-f}). 
The restriction of $\mathcal{F} \subset \mathrm{Lip}(X)$ can be interpreted as applying the spectral normalization \citep{miyato2018spectral}, widely-used stabilization technique, to enforce the discriminator to be Lipschitz continuous. 
The spectral normalization \citep{miyato2018spectral} is a normalization technique for weights of neural networks so that the Lipschitz norm $\|\varphi \|_{\mathrm{Lip}}$ of the discriminator is bounded above by 1.

\section{Conclusion} 
% and Future Work}

We provided the rigorous framework for the convergence analysis of the minimax problem in the infinite-dimensional spaces of continuous functions and probability measures.
We discussed GANs and UDAs comprehensively and interpreted the assumptions for the convergences as stabilization techniques.
\section*{Acknowledgments}
The research was jointly funded by AISIN and AISIN SOFTWARE.
%
%%%%%%%%%%%%%%%%%%%%%%%%%%%%%%%%%%%%%%%%%%%%%%%%%%%%%%%%%%%%%%%%%
%\bibliographystyle{plain}
% \bibliographystyle{plainnat}
%\bibliographystyle{abbrvnat}
%\bibliography{ref.bib}
\bibliography{jmlr.bbl}
%%%%%%%%%%%%%%%%%%%%%%%%%%%%%%%%%%%%%%%%%%%%%%%%%%%%%%%%%%
%
%
%
%
%
%%%%%%%%%%%%%%%%%%%%%%%%%%%%%%%%%%%%%%%%%%%%%%%%%%%%%%%
% Appendix
%%%%%%%%%%%%%%%%%%%%%%%%%%%%%%%%%%%%%%%%%%%%%%%%%%%%%%%%
\newpage
%\onecolumn
\appendix
\part*{Appendix}

\section{Proof of Theorem \ref{main-theorem-convergence}}\label{Appendix1}

Before the proof of Theorem \ref{main-theorem-convergence}, we review the three-point inequality.

\paragraph{Three-point inequality}
The three-point inequality is a key ingredient for the proof of Theorem \ref{main-theorem-convergence}, which  was first introduced by \cite{chen1993convergence}.
We introduce the three-point inequality in the space of measures and in continuous function spaces without the proof. See \citet{aubin2022mirror} for the proof.
\begin{lemma}[Three-point inequality for the space of measures]
    \label{Three-point-inequality-measure}
    Let $S_\cM \subset \mathcal{M}(X)$,
    and let $G:\mathcal{M}(X) \to \overline{\mathbb{R}}$ be a proper, lower semi-continuous, and convex function.
    Let $\alpha_{\cM}>0$.
    For a given $\mu\in\cM(X)$, let
    \begin{equation*}
        \overline{\nu}
        := \underset{\nu \in S_\cM}{\mathrm{argmin}} 
        \left\{ G(\nu) + \frac{1}{2 \alpha_{\cM}} \|\nu-\mu\|^2_{\cM(X)} \right\},
    \end{equation*}
    where $\|\cdot \|_{\cM(X)}$ is a norm induced by inner products in $\cM(X)$.
    Then, 
    \begin{equation*}
        G(\nu) + \frac{1}{2 \alpha_{\cM}} \|\nu-\mu\|^2_{\cM(X)} 
        \geq G(\overline{\nu}) + \frac{1}{2 \alpha_{\cM}} \|\overline{\nu}-\mu\|^2_{\cM(X)} + \frac{1}{2 \alpha_{\cM}} \|\nu-\overline{\nu}\|^2_{\cM(X)}
        \quad\text{for all}~ \nu\in S_\cM.
    %\label{TPI-measure}
    \end{equation*}
\end{lemma}
%
%%%%%%%%%%%%%%%%%%%%%%%%%%%%%%%%%%%%%%%%%%%%%%%%%%%%%%%%%%%%%%%%%
%%%%%%%%%%%%%%%%%%%%%%%%%%%%%%%%%%%%%%%%%%%%%%%%%%%%%%%%%%%%%%%%%
\medskip
\begin{lemma}[Three-point inequality for continuous function space]
    \label{Three-point-inequality-function}
    Let $S_\cC \subset \mathcal{C}(X)$, 
    and let $F:\mathcal{C}(X) \to \overline{\mathbb{R}}$ be a proper, lower semi-continuous, and convex function.
    Let $\alpha_{\cC}>0$.
    For a given $f\in\cC(X)$, let
    \begin{equation*}
        \overline{g}:=
        \underset{g \in S_\cC }{\mathrm{argmin}}
        \left\{ F(g) + \frac{1}{2 \alpha_{\cC}} \|g-f\|_{\cC(X)}^2 \right\}.
    \end{equation*}
    where $\|\cdot \|_{\cC(X)}$ is a norm induced by inner products in $\cC(X)$.
    Then,
    \begin{equation*}
        F(g) + \frac{1}{2 \alpha_{\cC}} \|g-f\|_{\cC(X)}^2 \geq F(\overline{g}) + \frac{1}{2 \alpha_{\cC}} \|\overline{g}-f\|_{\cC(X)}^2 + \frac{1}{2 \alpha_{\cC}} \|g-\overline{g} \|_{\cC(X)}^2
        \quad\text{for all}~ g\in S_\cC.
    % \label{TPI-function}
    \end{equation*}
    %Furthermore, if $\overline{\psi} \in \mathrm{int}(\Spsi)$, then (\ref{TPI-function}) holds for all $\psi \in \mathcal{C}(X)$.
\end{lemma}

\begin{comment}
\begin{lemma}[Three-point inequality for the continuous function space]\label{Three-point-inequality-function}
Let $\varphi \in \mathcal{C}(X)$, and let $\Spsi \subset \mathcal{C}(X)$ be a subset in $\mathcal{C}(X)$. 
Let $F:\mathcal{C}(X) \to \overline{\mathbb{R}}$ be a convex function.
Let
\[
\overline{\psi}:=
\underset{\psi \in \Spsi}{\mathrm{argmin}} \left\{ F(\psi) + D_{\eta}(\psi|\varphi) \right\}.
\]
Then, for $\psi \in \Spsi$,
\begin{equation}
F(\psi) + D_{\xi}(\psi|\varphi) \geq F(\overline{\psi}) + D_{\xi}(\overline{\psi}|\varphi) + D_{\xi}(\psi|\overline{\psi}).
\label{TPI-function}
\end{equation}
%Furthermore, if $\overline{\psi} \in \mathrm{int}(\Spsi)$, then (\ref{TPI-function}) holds for all $\psi \in \mathcal{C}(X)$.
\end{lemma}
\end{comment}
%
%
%%%%%%%%%%%%%%%%%%%%%%%%%%%%%%%%%%%%%%%%%%%%%%%
% Main Theorem
%%%%%%%%%%%%%%%%%%%%%%%%%%%%%%%%%%%%%%%%%%%%%%%
\medskip
The proof of Theorem \ref{main-theorem-convergence} is essentially based on the $L$-smoothness of $\cK(\psi, \mu, \varphi)$ for each variables and the three-point inequality in Lemma~\ref{Three-point-inequality-measure} and \ref{Three-point-inequality-function} associated with the update rules of the gradient descent in Definition~\ref{mirror-descent}.
%Our proof is similar with the proof of \cite{lu2018relatively} and \cite{aubin2022mirror} in part but is generalized to minimax problems.

\begin{proof}[Proof of Theorem \ref{main-theorem-convergence}]
First, we evaluate the lower bound of $\cK(\psi_n, \mu_n, \varphi_{n+1})$.
The following holds for any $\varphi\in \Svar$:
\begin{align*}
    &\mathcal{K}(\psi_{n}, \mu_{n}, \varphi_{n+1})\\
    &\geq \mathcal{K}(\psi_{n}, \mu_{n}, \varphi_{n}) 
    + d\mathcal{K}(\psi_{n}, \mu_{n}, \cdot)_{\varphi_{n}}(\varphi_{n+1} - \varphi_{n}) 
    - \frac{L}{2} \| \varphi_{n+1}-\varphi_n\|^{2}_{\normvar} \\
    &\geq \mathcal{K}(\psi_{n}, \mu_{n}, \varphi_{n}) 
    + d\mathcal{K}(\psi_{n}, \mu_{n}, \cdot)_{\varphi_{n}}(\varphi_{n+1}-\varphi_{n}) 
    - \frac{1}{2 \alpha_{n}}\| \varphi_{n+1}-\varphi_n\|^{2}_{\normvar} \\
    &\geq \mathcal{K}(\psi_{n}, \mu_{n}, \varphi_{n}) 
    + d\mathcal{K}(\psi_{n}, \mu_{n}, \cdot)_{\varphi_{n}}(\varphi-\varphi_{n})
    - \frac{1}{2\alpha_{n}}( \| \varphi-\varphi_n\|^{2}_{\normvar} - \| \varphi-\varphi_{n+1}\|^{2}_{\normvar} )\\
    &\geq \mathcal{K}(\psi_{n}, \mu_{n}, \varphi)
    - \frac{1}{2\alpha_{n}}( \| \varphi-\varphi_n\|^{2}_{\normvar} - \| \varphi-\varphi_{n+1}\|^{2}_{\normvar} ),
\end{align*}
where the first inequality follows from the $L$-smoothness of $\varphi\mapsto \cK(\psi, \mu, \varphi)$ 
%in Definition \ref{Relative smoothness for measure and continuous function spaces}
, and the second inequality follows from $0 < \alpha_n \le 1/L$, and the last inequality results from the concavity of $\varphi\to\cK(\psi, \mu, \varphi)$ for each $\psi$ and $\mu$. 
Also, the third inequality follows from the three-point inequality in Lemma~\ref{Three-point-inequality-function} 
with 
$
    F(\varphi)
    = -d\mathcal{K}(\psi_{n}, \mu_{n}, \cdot)_{\varphi_{n}}(\varphi - \varphi_{n})
$
and $f=\varphi_n$:
\begin{multline*}
    d\mathcal{K}(\psi_{n}, \mu_{n}, \cdot)_{\varphi_{n}}(\varphi_{n+1}-\varphi_{n}) 
    - \frac{1}{2\alpha_{n}}\|\varphi_{n+1}-\varphi_n\|_{\normvar}^{2}
    \\
    \leq 
    d\mathcal{K}(\psi_{n}, \mu_{n}, \cdot)_{\varphi_{n}}(\varphi-\varphi_{n}) 
    - \frac{1}{2\alpha_{n}}(\|\varphi-\varphi_n\|_{\normvar}^{2}-\|\varphi-\varphi_{n+1}\|_{\normvar}^{2}),
    \quad \varphi \in \Svar.
    %\label{convergence-objective-DA-minimax-TPI-function-n}
\end{multline*}
Furthermore, by using the the concavity of $ \varphi \mapsto \mathcal{K}(\psi_{n}, \mu_n, \varphi)$ for $\cK(\psi_n, \mu_{n}, \varphi_{n+1})$ at the first line, we have
\begin{align}
    &
    -\alpha_{n}\mathcal{K}(\psi_{n}, \mu_{n}, \varphi_{n}) 
    + \alpha_{n}\mathcal{K}(\psi_{n}, \mu_{n}, \varphi)
    \nonumber
    \\
    &
    \leq 
    \frac{1}{2}(\| \varphi-\varphi_n\|^{2}_{\normvar} - \| \varphi-\varphi_{n+1}\|^{2}_{\normvar})
    +
    \alpha_n d\cK(\psi_n, \mu_n, \cdot)_{\varphi_n}(\varphi_{n+1}-\varphi_n)   
    .
    \label{DA-esti-K-var-n}
\end{align}
By taking the summation of (\ref{DA-esti-K-var-n}) from $n=0$ to $N-1$, we obtain
\begin{align}
    -\sum_{n=0}^{N-1}
    \alpha_{n}
    \mathcal{K}(\psi_{n}, \mu_{n}, \varphi_{n}) 
    +
    \left(
    \sum_{n=0}^{N-1}\alpha_{n}\right) \mathcal{K}(\widehat{\psi}_{N}, \widehat{\mu}_{N}, \varphi)
    \leq \frac{1}{2 } \|\varphi-\varphi_0\|_{\normvar}^2+\gatSumvar.
    \label{phi_upper_bound}
\end{align}
Here, we used 
$
%\begin{equation}
    \sum_{n=0}^{N-1} \alpha_{n} \mathcal{K}(\psi_n, \mu_n, \varphi)
    \geq \sum_{n=0}^{N-1} \alpha_{n} \mathcal{K}(\widehat{\psi}_{N}, \widehat{\mu}_{N}, \varphi)
%\end{equation}
$ by Jensen's inequality
, where the weighted sums $(\hpsi_N, \hmu_N, \hvphi_N)$ are defined by (\ref{averages}).
Here, we introduced $\gatSumvar$ defined as
\begin{equation}
    \gatSumvar
    :=
    \sum_{n=0}^{N-1}
    \alpha_n
    d\cK(\psi_n, \mu_n, \cdot)_{\varphi_n}(\varphi_{n+1}-\varphi_n) 
    \ge 0.
    \label{sum-of-gateaux-differentials-var}
\end{equation}
The non-negativity follows from the update rule defined by Definition \ref{mirror-descent}.
When $\varphi=\widehat{\varphi}_{N}$ in (\ref{phi_upper_bound}), we have
\begin{equation}
    -\sum_{n=0}^{N-1}
    \alpha_{n}
    \mathcal{K}(\psi_{n}, \mu_{n}, \varphi_{n}) 
    + \left(
    \sum_{n=0}^{N-1}\alpha_{n}\right)
    \mathcal{K}(\widehat{\psi}_{N}, \widehat{\mu}_{N}, \widehat{\varphi}_{N})
    \leq
    \frac{1}{2 } \|\widehat{\varphi}_{N}-\varphi_0\|_{\normvar}^2+\gatSumvar.
    \label{estimate-new-DA-1-4}
\end{equation}
When $\varphi=\varphi_\ast$ in (\ref{phi_upper_bound}), we have
\begin{equation}
    - 
    \sum_{n=0}^{N-1}
    \alpha_{n}\mathcal{K}(\psi_{n}, \mu_{n}, \varphi_{n})
    +
    \left(
    \sum_{n=0}^{N-1}\alpha_{n}\right)
    \mathcal{K}(\psi_{\ast}, \mu_{\ast}, \varphi_{\ast}) 
    \leq
   \frac{1}{2 } \|\varphi_{\ast}-\varphi_0\|_{\normvar}^2+\gatSumvar,
    \label{estimate-new-DA-1-3}
\end{equation}
where $(\psi_{\ast}, \mu_{\ast}, \varphi_{\ast})$ is a saddle point defined at \eqref{eq:nash-minimax} satisfying
$
    \mathcal{K}(\psi_{\ast}, \mu_{\ast}, \varphi_{\ast}) 
    \leq \mathcal{K}(\hpsi_N, \hmu_N, \varphi_{\ast})
$.

\medskip
Second, we evaluate the upper bound of $\cK(\psi_n, \mu_{n+1}, \varphi_{n})$.
The following holds for any $(\psi, \mu) \in S_\cM\times S_\cC$:
\begin{align*}
    & \mathcal{K}(\psi_{n}, \mu_{n+1}, \varphi_{n})\\
    & \leq \mathcal{K}(\psi_{n}, \mu_{n}, \varphi_{n}) 
    + d\mathcal{K}(\psi_{n}, \cdot, \varphi_{n})_{\mu_n}(\mu_{n+1}-\mu_{n})
    + \frac{L}{2} \|\mu_{n+1}-\mu_{n}\|_{\normmu}^2 \\
    & \leq \mathcal{K}(\psi_{n}, \mu_{n}, \varphi_{n}) 
    + d\mathcal{K}(\psi_{n}, \cdot, \varphi_{n})_{\mu_n}(\mu_{n+1}-\mu_{n})
    + \frac{1}{2 \alpha_{ n}}\| \mu_{n+1}-\mu_{n} \|_{\normmu}^2 \\
    & 
    \quad 
    - d\mathcal{K}(\cdot, \mu_{n}, \varphi_{n})_{\psi_n}(\psi_{n+1}-\psi_{n})
    \\
    & \quad
    + d\mathcal{K}(\cdot, \mu_{n}, \varphi_{n})_{\psi_n}(\psi_{n+1}-\psi_{n})
    + \frac{1}{2 \alpha_{n}} \|\psi_{n+1}-\psi_{n}\|_{\normpsi}^2 \\
    &
    \leq \mathcal{K}(\psi_{n}, \mu_{n}, \varphi_{n})
    + d\mathcal{K}( \cdot, \mu_{n}, \varphi_{n})_{\psi_n}(\psi-\psi_{n})  
    + d\mathcal{K}(\psi_{n}, \cdot, \varphi_{n})_{\mu_n}(\mu-\mu_{n})\\
    &\quad 
    - d\mathcal{K}( \cdot, \mu_{n}, \varphi_{n})_{\psi_n}(\psi_{n+1}-\psi_{n})\\
    &\quad 
    + \frac{1}{2\alpha_{n}}
    \left(
          \| \psi-\psi_{n}\|_{\normpsi}^2 
        - \| \psi-\psi_{n+1}\|_{\normpsi}^2 
        + \|\mu-\mu_{n}\|_{\normmu}^2
        - \|\mu-\mu_{n+1}\|_{\normmu}^2
    \right) \\
    &\leq \cK(\psi, \mu, \varphi_{n})
    - d\mathcal{K}( \cdot, \mu_{n}, \varphi_{n})_{\psi_n}(\psi_{n+1}-\psi_{n})\\
    &\quad
    + \frac{1}{2 \alpha_{n}}
    \left(
          \| \psi-\psi_{n}\|_{\normpsi}^2
        - \| \psi-\psi_{n+1}\|_{\normpsi}^2 
        + \|\mu-\mu_{n}\|_{\normmu}^2
        - \|\mu-\mu_{n+1}\|_{\normmu}^2
    \right),
\end{align*}
where the first inequality follows from the $L$-smoothness of $\mu\mapsto\cK(\psi, \mu, \varphi)$ for each $\psi$ and $\varphi$, the second inequality follows from $0<\alpha_{n}\le 1/L$, and the last inequality is the result of the joint convexity of $(\psi, \mu) \mapsto \cK(\psi, \mu, \varphi)$ for any $\varphi$. 
Also, the third inequality results from the three-point inequality in
Lemma \ref{Three-point-inequality-measure} with 
$
    G(\nu)=d\mathcal{K}( \psi_n, \cdot, \varphi_{n})_{\mu_n}(\nu - \mu_n)
$ and $\mu = \mu_n$,
\begin{equation*}
    \begin{split}
    & d\mathcal{K}(\psi_{n}, \cdot, \varphi_{n})_{\mu_n}
    (\mu_{n+1}-\mu_{n})
    +\frac{1}{2 \alpha_{n}}\|\mu_{n+1}-\mu_{n}\|_{\normmu}^2
    \\
    &
    \leq 
    d\mathcal{K}(\psi_{n}, \cdot, \varphi_{n})_{\mu_n}(\mu -\mu_{n}) + \frac{1}{2 \alpha_{n}}(\|\mu-\mu_{n}\|_{\normmu}^2 -\|\mu-\mu_{n+1}\|_{\normmu}^2),
    \quad \mu\in \Smu,
    \label{convergence-objective-DA-minimax-TPI-psi-n}
    \end{split}
\end{equation*}
and Lemma \ref{Three-point-inequality-function} with
$
    F(\psi)=d\mathcal{K}( \cdot, \mu_{n}, \varphi_{n})_{\psi_n}(\psi - \psi_{n})
$
and $f=\psi_n$
\begin{equation*}
    \begin{split}
    & d\mathcal{K}( \cdot, \mu_{n}, \varphi_{n})_{\psi_n}
    (\psi_{n+1}-\psi_{n}) 
    +
    \frac{1}{2 \alpha_{n}}\|\psi_{n+1}-\psi_{n}\|_{\normpsi}^2 
    \\
    &
    \leq 
    d\mathcal{K}( \cdot, \mu_{n}, \varphi_{n})_{\psi_n}(\psi-\psi_{n}) 
    +
    \frac{1}{2 \alpha_{n}}\left(\|\psi-\psi_{n}\|_{\normpsi}^2  -\|\psi -\psi_{n+1}\|_{\normpsi}^2  \right),
    \quad \mu\in \Spsi.
    \end{split}
    \label{convergence-objective-DA-minimax-TPI-psi-n-aaa}
\end{equation*}
Furthermore, by using the the convexity of $ \mu \mapsto \mathcal{K}(\psi_{n}, \mu, \varphi_{n})$ for $\cK(\psi_n, \mu_{n+1}, \varphi_{n})$ at the first line, we have
\begin{equation}
    \begin{split}
    & \alpha_{n} \mathcal{K}(\psi_{n}, \mu_{n}, \varphi_{n}) 
    - \alpha_{n} \mathcal{K}(\psi, \mu, \varphi_{n})
    \\
    &\hspace{1cm}
    \leq 
    \frac{1}{2}
    \left(
          \| \psi-\psi_{n}\|_{\normpsi}^2 
        - \| \psi-\psi_{n+1}\|_{\normpsi}^2 
        + \|\mu-\mu_{n}\|_{\normmu}^2
        - \|\mu-\mu_{n+1}\|_{\normmu}^2
    \right)\\
    &\hspace{2cm}\quad
    -\alpha_{n} d\mathcal{K}( \cdot, \mu_{n}, \varphi_{n})_{\psi_n}(\psi_{n+1}-\psi_{n}) 
    -\alpha_{n} d\mathcal{K}( \psi_n, \cdot, \varphi_{n})_{\mu_n}(\mu_{n+1}-\mu_{n}).
    \end{split}
    \label{estimate-n-DA-1}
\end{equation}

We perform the summation of (\ref{estimate-n-DA-1}) over the interval $n=0$ to $N-1$:
\begin{equation}
    \sum_{n=0}^{N-1}
    \alpha_{n}
    \mathcal{K}(\psi_{n}, \mu_{n}, \varphi_{n}) 
    - 
    \left(\sum_{n=0}^{N-1}\alpha_{n}\right)
    \mathcal{K}(\psi, \mu, \widehat{\varphi}_{N})
    \leq 
    \frac{1}{2} \left(
        \| \psi-\psi_{0}\|_{\normpsi}^2 +
        \| \mu-\mu_{0}\|_{\normmu}^2
    \right)
    + \gatSum,
    \label{psi-mu-upper}
\end{equation}
where we used 
$
    \sum_{n=0}^{N-1}\alpha_{n}\mathcal{K}(\psi, \mu, \varphi_{n})
    \leq \left(\sum_{n=0}^{N-1}\alpha_{n}\right)
    \mathcal{K}(\psi, \mu, \widehat{\varphi}_{N})
$ following from the Jensen's inequality and
we introduced $\gatSum$ defined as
\begin{equation}
    \gatSum
    :=
    - \sum_{n=0}^{N-1}
    \alpha_n
    \left(
          d\mathcal{K}( \cdot, \mu_{n}, \varphi_{n})_{\psi_n}(\psi_{n+1}-\psi_{n}) 
        + d\mathcal{K}( \psi_n, \cdot, \varphi_{n})_{\mu_n}(\mu_{n+1} - \mu_{n})
    \right)
    \ge 0.
    \label{sum-of-gateaux-differentials}
\end{equation}
% sawa: ここ、なんでnon-negativeになるんでしたっけ？
The non-negativity follows from the update rule defined by Definition \ref{mirror-descent}.
By substituting $\psi=\widehat{\psi}_{N}$ and $\mu=\widehat{\mu}_{N}$ into (\ref{psi-mu-upper}), we have
\begin{equation}
    \begin{split}
        \sum_{n=0}^{N-1}
        \alpha_{n}
        \mathcal{K}(\psi_{n}, \mu_{n}, \varphi_{n+1}) 
        &
        - 
        \left(\sum_{n=0}^{N-1}\alpha_{n}\right)
        \mathcal{K}(\widehat{\psi}_{N}, \widehat{\mu}_{N}, \widehat{\varphi}_{N})
        \\
        &
        \leq \frac{1}{2} \left(
            \|\widehat{\psi}_{N}-\psi_{0}\|_{\normpsi}^2
            + \|\widehat{\mu}_{N}-\mu_{0}\|_{\normmu}^2
        \right)
        + \gatSum.
    \end{split}
    \label{estimate-new-DA-1-2}
\end{equation}
Also, by taking $(\psi, \mu) = (\psi_\ast, \mu_\ast)$ in (\ref{psi-mu-upper}) which is a saddle point for the minimax solution for minimax problem $\mathcal{K}(\psi, \mu, \varphi)$ on $(\Spsi \times \Smu) \times \Svar$ such that $(\psi_\ast, \mu_\ast, \varphi_\ast)$ satisfies the 
$
    \mathcal{K}(\psi_{\ast}, \mu_{\ast}, \hvphi_N)
    \leq \mathcal{K}(\psi_{\ast}, \mu_{\ast}, \varphi_{\ast})
$,
we have
\begin{equation}
\begin{split}
    \sum_{n=0}^{N-1}\alpha_{n} 
    \mathcal{K}(\psi_{n}, \mu_{n}, \varphi_{n+1}) 
    &
    - \left(\sum_{n=0}^{N-1}\alpha_{n}\right)
    \mathcal{K}(\psi_{\ast}, \mu_{\ast}, \varphi_{\ast})
    \\
    &
    \leq \frac{1}{2}
    \left(
        \|\psi_{\ast}-\psi_{0}\|_{\normpsi}^2
        +\| \mu_{\ast}-\mu_{0}\|_{\normmu}^2
    \right)
    + \gatSum.
\label{estimate-new-DA-1-1}
\end{split}
\end{equation}

\medskip
Third, let combine all the results we have obtained.
Summing up (\ref{estimate-new-DA-1-3}) and (\ref{estimate-new-DA-1-2}) yields
\begin{equation*}
    \begin{split}
    \left(\sum_{n=0}^{N-1}\alpha_{n}\right)
    &
    (
        \mathcal{K}(\psi_{\ast}, \mu_{\ast}, \varphi_{\ast})
        - \mathcal{K}(\widehat{\psi}_{N}, \widehat{\mu}_{N}, \widehat{\varphi}_{N})
    )\\
    &\leq \frac{1}{2}
    \left(
        \|\widehat{\psi}_{N}-\psi_{0}\|_{\normpsi}^2
        + \|\widehat{\mu}_{N}-\mu_{0}\|_{\normmu}^2
        + \|\varphi_\ast-\varphi_0\|_{\normvar}^2
    \right)
    +\gatSumvar + \gatSum.
    \end{split}
\end{equation*}
Also, summing up (\ref{estimate-new-DA-1-4}) and (\ref{estimate-new-DA-1-1}) gives
\begin{equation*}
    \begin{split}
    \left(\sum_{n=0}^{N-1}\alpha_{n}\right)
    &
    (
        \mathcal{K}(\widehat{\psi}_{N}, \widehat{\mu}_{N}, \widehat{\varphi}_{N})
        - \mathcal{K}(\psi_{\ast}, \mu_{\ast}, \varphi_{\ast})
    )\\
    &\leq \frac{1}{2}\left(
          \|\psi_{\ast}-\psi_{0}\|_{\normpsi}^2
        + \|\mu_{\ast}-\mu_{0}\|_{\normmu}
        + \|\widehat{\varphi}_N-\varphi_0\|_{\normvar}^2
    \right)
    +\gatSumvar + \gatSum.
    \end{split}
\end{equation*}
Therefore, we obtain
\begin{equation}
    \begin{split}
    &\left|
    \mathcal{K}(\widehat{\psi}_{N}, \widehat{\mu}_{N}, \widehat{\varphi}_{N})
    - \mathcal{K}(\psi_{\ast}, \mu_{\ast}, \varphi_{\ast})
    \right|
    \leq 
    \left( 
        \sum_{n=0}^{N-1}\alpha_{n} 
    \right)^{-1}
    \left(
    \frac{1}{2}C_s +\gatSumvar + \gatSum
    \right)
    ,
    \end{split}
\label{save-point-1}
\end{equation}
where $C_s>0$ is a finite constant defined in (\ref{const-C}).
This means that the value of the object function at $(\hpsi_N, \hmu_N, \hvphi_N)$ approximately converges to a saddle point under the gradient descent update rule, if the sums $\gatSumvar + \gatSum$ of the G\^ateaux differentials in (\ref{sum-of-gateaux-differentials-var}) and (\ref{sum-of-gateaux-differentials}) are finite. 

\medskip
%
%
\begin{comment}
Finally, we prove that the sums $\gatSumvar + \gatSum$ of G\^ateaux differentials defined in (\ref{sum-of-gateaux-differentials-var}) and (\ref{sum-of-gateaux-differentials}) are bounded from above by the norms 
$
    \left\| \mu_{n+1}-\mu_{n} \right\|_{\normmu}
$,
$
    \left\|\psi_{n+1}-\psi_{n} \right\|_{\normpsi}
$ and 
$
    \left\|\varphi_{n+1}-\varphi_{n} \right\|_{\normvar}
$
under Assumptions~\ref{exist-argmax-DA-minimax}.
Actually, each G\^ateaux differential 
%in (\ref{sum-of-gateaux-differentials})
is bounded from above 
%under Lemma \ref{Gatea-diff-form-DA-app} associated with Assumption~\ref{exist-argmax-DA-minimax} 
as follows:
\end{comment}
%
%
Finally, we prove that the term $\gatSumvar + \gatSum$ defined in (\ref{sum-of-gateaux-differentials-var}) and (\ref{sum-of-gateaux-differentials}) are bounded from above by the norms 
$
    \left\| \mu_{n+1}-\mu_{n} \right\|_{\normmu}
$,
$
    \left\|\psi_{n+1}-\psi_{n} \right\|_{\normpsi}
$ and 
$
    \left\|\varphi_{n+1}-\varphi_{n} \right\|_{\normvar}
$
under Assumptions~\ref{exist-argmax-DA-minimax}.
Actually, each G\^ateaux differential is bounded from above as follows:
\begin{equation}
    \begin{split}
    -d\mathcal{K}( \cdot, \mu_{n}, \varphi_{n})_{\psi_n}(\psi_{n+1}-\psi_{n}) 
    &= -\int (\psi_{n+1}-\psi_{n}) dN_{\psi_n, \mu_{n}, \varphi_{n}}
    \\
    &
    \leq \left\|\psi_{n+1}-\psi_{n} \right\|_{\normpsi} \underbrace{ \left\|N_{\psi_n, \mu_{n}, \varphi_{n}} \right\|_{\normpsi}^{\star}
    }_{\leq B},
    \end{split}
    \label{psi-gateaux-upper-bound}
\end{equation}
and
\begin{equation}
    \begin{split}
        -d\mathcal{K}(\psi_{n}, \cdot, \varphi_{n})_{\mu_n}(\mu_{n+1}-\mu_{n})
    &= -\int \Phi_{\psi_{n}, \mu_{n}, \varphi_{n}} d(\mu_{n+1}-\mu_{n})
    \\
    &
    \leq \underbrace{\left\|\Phi_{\psi_{n}, \mu_{n}, \varphi_{n}} \right\|_{\normmu}^{\star}
    }_{\leq B}
    \left\|\mu_{n+1}-\mu_{n} \right\|_{\normmu},
    \end{split}
\end{equation}
and
\begin{equation}
    \begin{split}
        d\mathcal{K}(\psi_{n}, \mu_n, \cdot)_{\varphi_{n}}(\varphi_{n+1}-\varphi_{n})
    &= \int \varphi_{n+1}-\varphi_{n} d \Lambda_{\psi_{n}, \mu_{n}, \varphi_{n}}
    \\
    &
    \leq 
    \left\|\varphi_{n+1}-\varphi_{n} \right\|_{\normvar}
    \underbrace{\left\|\Lambda_{\psi_{n}, \mu_{n}, \varphi_{n}} \right\|_{\normvar}^{\star}
    }_{\leq B}
    .
    \end{split}
\end{equation}
Moreover, by taking into account that the gradient decent scheme in Definition~\ref{mirror-descent} with $\psi=\psi_n$ implies
$
    d\mathcal{K}( \cdot, \mu_{n}, \varphi_{n})_{\psi_n}(\psi_{n+1}-\psi_n) + \frac{1}{2 \alpha_n }\|\psi_{n+1}-\psi_n\|_{\normpsi}^2 \leq 0
$, 
we can obtain that
\begin{align*}
    \frac{1}{2 \alpha_n} \left\| \psi_{n+1}-\psi_n \right\|^{2}_{\mathcal{C}(X)} 
    &\leq -d\mathcal{K}( \cdot, \mu_{n}, \varphi_{n})_{\psi_n}(\psi_{n+1}-\psi_{n}) \\
    &\leq B \left\| \psi_{n+1}-\psi_n \right\|_{\normpsi},
\end{align*}
which is equivalent to
\begin{equation}
    \left\| \psi_{n+1}-\psi_n \right\|_{\normpsi} 
    \leq
    2 B \alpha_{n}.
    \label{save-point-2}
\end{equation}
By the similar argument, we obtain
\begin{equation}
    \left\| \mu_{n+1}-\mu_n \right\|_{\normmu} 
    \leq
    2B \alpha_{n}, \ \ 
    \left\| \varphi_{n+1}-\varphi_n \right\|_{\normvar} 
    \leq
    2B \alpha_{n}.
    \label{save-point-3}
\end{equation}
By combining (\ref{save-point-1}) with \eqref{psi-gateaux-upper-bound} - \eqref{save-point-3}, we conclude \eqref{main-theorem-convergence-eq}:
\begin{align*}
    &\left|
    \mathcal{K}(\widehat{\psi}_{N}, \widehat{\mu}_{N}, \widehat{\varphi}_{N})
    - \mathcal{K}(\psi_{\ast}, \mu_{\ast}, \varphi_{\ast})
    \right|
    \leq 
    \left( 
    \sum_{n=0}^{N-1}\alpha_{n} \right)^{-1}
    \left(
    \frac{1}{2} C_s
    + 6 B^2 \sum_{n=0}^{N-1} \alpha_{n}^2
    \right).
\end{align*}

\end{proof}

%
%
%%%%%%%%%%%%%%%%%%%%%%%%%%%%%%%%%%%%%%%%%%%%%%%%%%%%%%%%%%%%%%%%%
\section{Proof of Theorem \ref{main-theorem-convergence-nonconvex}}\label{Appendix1.5}
%%%%%%%%%%%%%%%%%%%%%%%%%%%%%%%%%%%%%%%%%%%%%%%%%%%%%%%%%%%%%%%%%
Before the proof of the main result, we will show two lemmas used in the proof of Theorem \ref{main-theorem-convergence-nonconvex}. 

\begin{lemma}\label{B-lemma-1}
Let Assumptions~\ref{closed-subspaces-nonconvex}, \ref{strongly-concave-L-nonconvex}, \ref{existence-Gate-diff-nonconvex}, and \ref{L-smooth-nonconvex} hold.
Then, we have the following:
\begin{itemize}
    \item[(i)]
    $
    \| \Phi(\psi_1, f) - \Phi(\psi_2, f) \|_{\Svar} 
    \leq 
    \frac{L}{\beta} \|\psi_1 - \psi_2 \|_{\Spsi}
    $ for $\psi_1, \psi_2 \in \Spsic$, $f \in \Sfc$.
    \item[(ii)]
    $
    \| \Phi(\psi, f_1) - \Phi(\psi, f_2) \|_{\Svar} 
    \leq 
    \frac{L}{\beta} \|f_1 - f_2 \|_{\Sf}
    $ for $\psi \in \Spsic$, $f_1, f_2 \in \Sfc$.
    \item[(iii)]
    $
    \psi \mapsto G(\psi,f)
    $ 
    is $L\left( \frac{L}{\beta} + 1 \right)$-smooth with respect to $\| \cdot \|_{\Spsi}$ over $\Spsic$ for each $f \in \Sfc$.
    \item[(iv)]
    $
    f \mapsto G(\psi,f)
    $ 
    is $L\left( \frac{L}{\beta} + 1 \right)$-smooth with respect to $\| \cdot \|_{\Sf}$ over $\Sfc$ for each $\psi \in \Spsic$.
\end{itemize}
\end{lemma}
\begin{proof}
The proof is a generalization of \citet[Lemma 4.3]{lin2020gradient} to infinite dimensional function spaces with two variable. 

By the optimality, we have
\begin{align*}
    &
    d\cG(\psi_1, f, \cdot)_{\Phi(\psi_1, f)} (\Phi(\psi_2, f)-\Phi(\psi_1, f)) \leq 0,
    \\
    &
    d\cG(\psi_2, f, \cdot)_{\Phi(\psi_2, f)} (\Phi(\psi_1, f)-\Phi(\psi_2, f)) \leq 0,
\end{align*}
which implies that
\begin{align}\label{B-eq-1}
    [d\cG(\psi_1, f, \cdot)_{\Phi(\psi_1, f)}
    -
    d\cG(\psi_2, f, \cdot)_{\Phi(\psi_2, f)}
    ](\Phi(\psi_2, f)-\Phi(\psi_1, f)) \leq 0.
\end{align}
With Assumption~\ref{strongly-concave-L-nonconvex} and (\ref{B-eq-1}), we estimate that
\begin{align*}
    &
    \beta \| \Phi(\psi_1, f) - \Phi(\psi_2, f) \|_{\Svar} 
    \\
    &
    \leq
    [d\cG(\psi_1, f, \cdot)_{\Phi(\psi_1, f)}
    -
    d\cG(\psi_1, f, \cdot)_{\Phi(\psi_2, f)}
    ](\Phi(\psi_2, f)-\Phi(\psi_1, f)) 
    \\
    &
    \leq
    [d\cG(\psi_2, f, \cdot)_{\Phi(\psi_2, f)}
    -
    d\cG(\psi_1, f, \cdot)_{\Phi(\psi_2, f)}
    ](\Phi(\psi_2, f)-\Phi(\psi_1, f)) 
    \\
    &
    \leq 
    \| \nabla \cG(\psi_2, f, \cdot)_{\Phi(\psi_2, f)}
    -
    \nabla \cG(\psi_1, f, \cdot)_{\Phi(\psi_2, f)} \|_{\Svar}
    \| \Phi(\psi_1, f) - \Phi(\psi_2, f) \|_{\Svar}
    \\
    &
    \leq 
    L 
    \| \psi_1 - \psi_2 \|_{\Svar}
    \| \Phi(\psi_1, f) - \Phi(\psi_2, f) \|_{\Svar},
\end{align*}
where last inequality results from Assumption~\ref{L-smooth-nonconvex}. 
Hence, we obtain (i). 
(ii) is given by the same arguments of (i).

By the envelop theorem \citep{milgrom2002envelope}, the G\^{a}teaux differential $dG(\cdot, f)_\psi$ of $G(\cdot, f)$ at $\psi \in \Spsi$ is represented as 
$
dG(\cdot, f)_\psi = d\cG(\cdot, f, \Phi(\psi, f))_\psi 
$.
Using this, we estimate that
\begin{align*}
    &
    \| \nabla G(\cdot, f)_{\psi_1}
    - \nabla G(\cdot, f)_{\psi_2}
    \|
    \\
    &
    =
    \| \nabla \cG(\cdot, f, \Phi(\psi_1, f) )_{\psi_1}
    -
     \nabla \cG(\cdot, f, \Phi(\psi_2, f) )_{\psi_2}  \|
     \\
     &
     \leq
     \| \nabla \cG(\cdot, f, \Phi(\psi_1, f) )_{\psi_1}
     -
     \nabla \cG(\cdot, f, \Phi(\psi_2, f) )_{\psi_1}  \|
     \\
     &
     \hspace{5mm}
     +
     \| \nabla \cG(\cdot, f, \Phi(\psi_2, f) )_{\psi_1}
    -
     \nabla \cG(\cdot, f, \Phi(\psi_2, f) )_{\psi_2}  \|
     \\
     &
     \leq
     L
     \left(
     \| \Phi(\psi_1, f)
     -
     \Phi(\psi_2, f)
     \|_{\Svar}
     + 
     \|\psi_1 - \psi_2\|_{\Spsi}
     \right)
     \\
     &
     \leq
     L
     \left( \frac{L}{\beta} + 1 \right)\|\psi_1 - \psi_2\|_{\Spsi},
\end{align*}
where last inequality follows from (i).
Using the above estimate, we further estimate that
\begin{equation}\label{B-eq-2}
\begin{split}
    &
    G(\psi_1, f) - G(\psi_2, f) -dG (\cdot, f)_{\psi_2}(\psi_1 - \psi_2)
    \\
    &
    \leq
    \int_{0}^{1} \frac{d}{d\epsilon}
    G(\psi_2 + \epsilon (\psi_1 - \psi_2), f) -  dG (\cdot, f)_{\psi_2}(\psi_1 - \psi_2) d \epsilon
    \\
    &
    \leq
    \int_{0}^{1} 
    dG(\cdot, f)_{\psi_2 + \epsilon (\psi_1 - \psi_2)}(\psi_1 - \psi_2) -  dG (\cdot, f)_{\psi_2}(\psi_1 - \psi_2) d \epsilon
    \\
    &
    \leq
    \int_{0}^{1}
    L\left( \frac{L}{\beta} + 1 \right)
    \epsilon
    \| \psi_1 - \psi_2 \|_{\Spsi}^{2} d \epsilon
    \\
    &
    \leq
    \frac{1}{2} L\left( \frac{L}{\beta} + 1 \right) \| \psi_1 - \psi_2 \|_{\Spsi}^{2}.
\end{split}
\end{equation}
Hence, we obtain (iii). 
(iv) is given by the same arguments.
\end{proof}
\begin{lemma}\label{B-lemma-2}
Let Assumptions~\ref{closed-subspaces-nonconvex}, \ref{strongly-concave-L-nonconvex}, \ref{existence-Gate-diff-nonconvex}, and \ref{L-smooth-nonconvex} hold.
Let $\eta>0$ and $\varphi \in \Svarc$, and we denote by 
\[
\varphi_{+}:= \cP_{\Svarc} \left( \varphi + \eta \nabla \cG (\psi, f, \cdot)_{\varphi} \right)
\]
Then, it holds that for $\psi \in \Spsic$, $f \in \Sfc$, and $\phi \in \Svarc$,
\begin{align*}
    -\cG(\psi, f, \varphi_{+}) + \cG(\psi, f, \phi)
    \leq 
    \frac{1}{\eta}\langle \varphi_{+} - \varphi, \phi -\varphi \rangle_{\Svar}
    + \left(\frac{L}{2}-\frac{1}{\eta} \right) \| \varphi_+ - \varphi \|_{\Svar}^2
    - 
    \frac{\beta}{2} \| \varphi - \phi \|_{\Svar}^2.
\end{align*}
\end{lemma}
\begin{proof}
The proof is generalized from the finite dimensional case \citep[Lemma 3.6]{bubeck2015convex}. 

By a property of the projection $\cP_{\Svarc}$ (see \citet[Lemma 3.1.4]{nesterov2003introductory}), we have
\begin{align}\label{B-eq-3}
    \langle \varphi_{+} - \left( \varphi + \eta \nabla \cG (\psi, f, \cdot)_{\varphi} \right), \varphi_{+} - \phi \rangle_{\Svar} \leq 0.
\end{align}
By Assumption~\ref{L-smooth-nonconvex} and same argument in (\ref{B-eq-2}), we can show that 
\begin{align*}
    -\cG(\psi, f, \varphi_{+}) 
    \leq
    -\cG(\psi, f, \varphi) 
    + d \cG (\psi, f, \cdot)_{\varphi}(\varphi_+ -\varphi) + 
    \frac{L}{2} \| \varphi_+ - \varphi \|_{\Svar}^2,
\end{align*}
which implies that with Assumption~\ref{strongly-concave-L-nonconvex} and (\ref{B-eq-3})
\begin{align*}
    &
    -\cG(\psi, f, \varphi_{+}) + \cG(\psi, f, \phi)
    \\
    &
    \leq 
    -\cG(\psi, f, \varphi_{+}) + \cG(\psi, f, \varphi) 
    -\cG(\psi, f, \varphi) + \cG(\psi, f, \phi)
    \\
    &
    \leq 
    - d \cG (\psi, f, \cdot)_{\varphi}(\varphi_+ -\varphi) + 
    \frac{L}{2} \| \varphi_+ - \varphi \|_{\Svar}^2
    + d \cG (\psi, f, \cdot)_{\varphi}(\varphi-\phi) - 
    \frac{\beta}{2} \| \varphi - \phi \|_{\Svar}^2,
    \\
    &
    \leq 
    -\langle \nabla \cG (\psi, f, \cdot)_{\varphi}, \varphi_{+} - \phi \rangle_{\Svar}
    + \frac{L}{2} \| \varphi_+ - \varphi \|_{\Svar}^2
    - 
    \frac{\beta}{2} \| \varphi - \phi \|_{\Svar}^2,
    \\
    &
    \leq 
    -\frac{1}{\eta}\langle \varphi_{+} - \varphi, \varphi_{+} - \phi \rangle_{\Svar}
    + \frac{L}{2} \| \varphi_+ - \varphi \|_{\Svar}^2
    - 
    \frac{\beta}{2} \| \varphi - \phi \|_{\Svar}^2,
    \\
    &
    \leq 
    \frac{1}{\eta}\langle \varphi_{+} - \varphi, \phi -\varphi \rangle_{\Svar}
    + \left(\frac{L}{2}-\frac{1}{\eta} \right) \| \varphi_+ - \varphi \|_{\Svar}^2
    - 
    \frac{\beta}{2} \| \varphi - \phi \|_{\Svar}^2.
\end{align*}
\end{proof}
\begin{proof}[Proof of Theorem \ref{main-theorem-convergence-nonconvex}]
The proof is a generalization of \citet[Theorem 4.4]{lin2020gradient} to infinite dimensional function spaces with three variables. 

We denote by $L_{\beta}=L\left( \frac{L}{\beta} + 1 \right)$.
First, we estimate the upper bound of $G(\psi_{n+1}, f_{n+1})$.
By Lemma~\ref{B-lemma-1}, we have
\begin{equation}\label{B-eq-4}
\begin{split}
    &
    G(\psi_{n+1}, f_{n+1})
    \leq 
    G(\psi_{n+1}, f_{n})
    + d G (\psi_{n+1},\cdot)_{f_n}(f_{n+1}-f_{n})
    + \frac{L_{\beta}}{2} \| f_{n+1} - f_{n} \|_{\Sf}^2.
\end{split}
\end{equation}
By a property of the projection $\cP_{\Sf}$ (see \citet[Lemma 3.1.5]{nesterov2003introductory} and Young's inequality, we have
\begin{equation}\label{B-eq-5}
\begin{split}
    &
    \| f_{n+1} - f_{n} \|_{\Sf}^2
    \\
    &
    \leq
    \alpha_{f,n}^2
    \| \nabla \cG (\psi_{n},\cdot, \varphi_n )_{f_n} \|_{\Sf}^2 
    \\
    &
    \leq
    2 \alpha_{f,n}^2 \| \nabla \cG (\psi_{n},\cdot, \Phi(\psi_{n}, f_n) )_{f_n} - \nabla \cG (\psi_{n},\cdot, \varphi_n )_{f_n} \|_{\Sf}^2 
    + 
    2 \alpha_{f,n}^2 \| \nabla \cG (\psi_{n},\cdot, \Phi(\psi_{n}, f_n) )_{f_n} \|_{\Sf}^2
    \\
    &
    \leq
    2 L^2 \alpha_{f,n}^2 \| \Phi(\psi_{n}, f_n)  -\varphi_n \|_{\Svar}^2 
    + 
    2 \alpha_{f,n}^2 \| \nabla G (\psi_{n},\cdot )_{f_n} \|_{\Sf}^2.
\end{split}
\end{equation}
We estimate that
\begin{align}\label{B-eq-6}
    &
    d G (\psi_{n+1},\cdot)_{f_n}(f_{n+1}-f_{n})
    \notag
    \\
    &
    =
    \langle 
    \nabla \cG (\psi_{n+1},\cdot, \Phi(\psi_{n+1}, f_n) )_{f_n}, 
    f_{n+1} - f_{n} \rangle_{\Svar}
    \notag
    \\
    &
    =
    - 
    \langle 
    \nabla \cG (\psi_{n+1},\cdot, \Phi(\psi_{n+1}, f_n) )_{f_n}, 
    \alpha_{f,n}
    \nabla \cG (\psi_{n},\cdot, \Phi(\psi_{n}, f_n) )_{f_n} \rangle_{\Svar}
    \notag
    \\
    &
    \hspace{0.4cm}
    +
    \langle 
    \nabla \cG (\psi_{n+1},\cdot, \Phi(\psi_{n+1}, f_n) )_{f_n}, 
    f_{n+1} - f_{n} + 
    \alpha_{f,n}
    \nabla \cG (\psi_{n},\cdot, \Phi(\psi_{n}, f_n) )_{f_n} \rangle_{\Svar},
\end{align}
and by Lemma~\ref{B-lemma-1} and Assumption~\ref{L-smooth-nonconvex},
\begin{align}\label{B-eq-7}
    &
    \|\nabla \cG (\psi_{n+1},\cdot, \Phi(\psi_{n+1}, f_n) )_{f_n}\|_{\Sf}
    \notag
    \\
    &
    \leq
    \|\nabla \cG (\psi_{n+1},\cdot, \Phi(\psi_{n+1}, f_n) )_{f_n}
    -
    \nabla \cG (\psi_{n},\cdot, \Phi(\psi_{n+1}, f_n) )_{f_n}
    \|_{\Sf}
    \notag
    \\
    &
    \hspace{0.4cm}
    +
    \|\nabla \cG (\psi_{n},\cdot, \Phi(\psi_{n+1}, f_n) )_{f_n}
    -
    \nabla \cG (\psi_{n},\cdot, \Phi(\psi_{n}, f_n) )_{f_n}
    \|_{\Sf}
    \notag
    \\
    &
    \hspace{0.4cm}
    +
    \|\nabla \cG (\psi_{n},\cdot, \Phi(\psi_{n}, f_n) )_{f_n}
    \|_{\Sf},
    \notag
    \\
    &
    \leq
    L (1+L_{\beta})
    \|\psi_{n+1}-\psi_{n}
    \|_{\Sf}
    +
    \|\nabla G (\psi_{n},\cdot)_{f_n}
    \|_{\Sf},
\end{align}
and by a property of the projection $\cP_{\Svarc}$ (see \citet[Lemma 3.1.5]{nesterov2003introductory})
\begin{align}\label{B-eq-8}
    &
    \|f_{n+1} - f_{n} + 
    \alpha_{f,n}
    \nabla \cG (\psi_{n},\cdot, \Phi(\psi_{n}, f_n) )_{f_n} \|_{\Sf}
    \notag
    \\
    &
    \leq
    \alpha_{f,n}
    \|
    \nabla \cG (\psi_{n},\cdot, \varphi_n )_{f_n}
    -
    \nabla \cG (\psi_{n},\cdot, \Phi(\psi_{n}, f_n) )_{f_n} \|_{\Sf}
    \notag
    \\
    &
    \leq
    L \alpha_{f,n}
    \|
    \varphi_n 
    -
    \Phi(\psi_{n}, f_n) \|_{\Svar}.
\end{align}
Combining  (\ref{B-eq-6}) with (\ref{B-eq-7}) and (\ref{B-eq-8}), we futhre estimate that
\begin{equation}\label{B-eq-10}
\begin{split}
    &
    d G (\psi_{n+1},\cdot)_{f_n}(f_{n+1}-f_{n})
    \\
    &
    \leq
    L (1+L_{\beta}) \alpha_{f,n} \|\psi_{n+1}-\psi_{n}\|_{\Spsi} \|\nabla G (\psi_{n},\cdot )_{f_n} \|_{\Sf}
    \\
    &
    \hspace{0.4cm}
    - \alpha_{f,n}\|\nabla G (\psi_{n},\cdot)_{f_n} \|_{\Sf}^2
    \\
    &
    \hspace{0.4cm}
    +
    L^2 (1+L_{\beta}) \alpha_{f,n} \|\psi_{n+1}-\psi_{n}\|_{\Spsi} \|
    \varphi_n 
    -
    \Phi(\psi_{n}, f_n) \|_{\Svar},
    \\
    &
    \hspace{0.4cm}
    +
    L \alpha_{f,n}
    \|\nabla G (\psi_{n},\cdot)_{f_n} \|_{\Sf}
    \|
    \varphi_n 
    -
    \Phi(\psi_{n}, f_n) \|_{\Svar}
    \\
    &\leq
    \frac{L(1+L_{\beta})}{2} \|\psi_{n+1}-\psi_{n}\|_{\Spsi}^2
    +
    \frac{L(1+L_{\beta})\alpha_{f,n}^2}{2}
    \|\nabla G (\psi_{n},\cdot )_{f_n} \|_{\Sf}^2
    \\
    &
    \hspace{0.4cm}
    - \alpha_{f,n}\|\nabla G (\psi_{n},\cdot)_{f_n} \|_{\Sf}^2
    \\
    &
    \hspace{0.4cm}
    +
    \frac{L^2 (1+L_{\beta})}{2} \|\psi_{n+1}-\psi_{n}\|_{\Spsi}^2 
    +
    \frac{L^2 (1+L_{\beta}) \alpha_{f,n}^2}{2} 
    \|
    \varphi_n 
    -
    \Phi(\psi_{n}, f_n) \|_{\Svar}^2,
    \\
    &
    \hspace{0.4cm}
    +
    \frac{L \alpha_{f,n}^2}{2}
    \|\nabla G (\psi_{n},\cdot)_{f_n} \|_{\Sf}^2
    +
    \frac{L}{2}
    \|
    \varphi_n 
    -
    \Phi(\psi_{n}, f_n) \|_{\Svar}^2,
\end{split}
\end{equation}
where we have employed Young's inequality for last inequality.
By the same way with (\ref{B-eq-5}), we have
\begin{equation}\label{B-eq-11}
\begin{split}
    &
    \|\psi_{n+1}-\psi_{n}\|_{\Spsi}^2 
    \\
    &
    \leq
    \alpha_{\psi, n}^2
    \| 
    \nabla \cG (\cdot, f_n, \varphi_n )_{\psi_n} \|_{\Spsi}^2 
    \\
    &
    \leq
    2 \alpha_{\psi, n}^2 \| \nabla \cG (\cdot, f_n, \Phi(\psi_{n}, f_n) )_{\psi_n} - \nabla \cG (\cdot, f_{n}, \varphi_n )_{\psi_n} \|_{\Spsi}^2 
    + 
    2 \alpha_{\psi, n}^2 \| \nabla \cG (\cdot,f_n, \Phi(\psi_{n}, f_n) )_{\psi_n} \|_{\Spsi}^2
    \\
    &
    \leq
    2 L^2 \alpha_{\psi, n}^2 \| \Phi(\psi_{n}, f_n) -  \varphi_n \|_{\Svar}^2 
    + 
    2 \alpha_{\psi, n}^2 \| \nabla G (\cdot,f_n )_{\psi_n} \|_{\Spsi}^2.
\end{split}
\end{equation}
With (\ref{B-eq-4}), (\ref{B-eq-5}), (\ref{B-eq-10}), and (\ref{B-eq-11}), we obtain that 
\begin{equation}\label{B-eq-12}
\begin{split}
    &
    G(\psi_{n+1}, f_{n+1})
    \\
    &
    \leq 
    G(\psi_{n+1}, f_{n})
    \\
    &
    \hspace{0.4cm}
    +
    \alpha_{f,n}
    \left\{
    -1 + \frac{L (1+L_{\beta})\alpha_{f,n}}{2} + \frac{L\alpha_{f,n}}{2}
    +
    L_{\beta}\alpha_{f,n}
    \right\}\| \nabla G (\psi_{n},\cdot )_{f_n} \|_{\Sf}^2 
    \\
    &
    \hspace{0.4cm}
    +
    \frac{L (1+L)(1+L_{\beta})}{2} 
     \|\psi_{n+1}-\psi_{n}\|_{\Spsi}^2 
    \\
    &
    \hspace{0.4cm}
    +
    \left\{
    \frac{L}{2} + \frac{L^2 (1+L_{\beta})\alpha_{f,n}^2}{2} + 
    L^2 L_{\beta}\alpha_{f,n}^2
    \right\} 
    \| \Phi(\psi_{n}, f_n)  -\varphi_n \|_{\Svar}^2
    \\
    &
    \leq 
    G(\psi_{n+1}, f_{n})
    \\
    &
    \hspace{0.4cm}
    +
    \alpha_{f,n}
    \left\{
    -1 + \frac{L (1+L_{\beta})\alpha_{f,n}}{2} + \frac{L\alpha_{f,n}}{2}
    +
    L_{\beta}\alpha_{f,n}
    \right\}\| \nabla G (\psi_{n},\cdot )_{f_n} \|_{\Sf}^2 
    \\
    &
    \hspace{0.4cm}
    +
    L (1+L)(1+L_{\beta}) \alpha_{\psi,n}^2 \| \nabla G (\cdot,f_n )_{\psi_n} \|_{\Spsi}^2
    \\
    &
    \hspace{0.4cm}
    +
    \left\{
    \frac{L}{2} + \frac{L^2 (1+L_{\beta})\alpha_{f,n}^2}{2} + 
    L^2 L_{\beta}\alpha_{f,n}^2
    + 
    L^3 (1+L)(1+L_{\beta}) \alpha_{\psi,n}^2
    \right\} 
    \| \Phi(\psi_{n}, f_n)  -\varphi_n \|_{\Svar}^2.
\end{split}
\end{equation}

Second, we estimate the upper bound of $G(\psi_{n+1}, f_{n})$.
By Lemma~\ref{B-lemma-1}, we have
\begin{equation}\label{B-eq-13}
\begin{split}
    &
    G(\psi_{n+1}, f_{n})
    \leq 
    G(\psi_{n}, f_{n})
    + d G (\cdot, f_{n})_{\psi_n}(\psi_{n+1}-\psi_{n})
    + \frac{L_{\beta}}{2} \| \psi_{n+1} - \psi_{n} \|_{\Spsi}^2.
\end{split}
\end{equation}
By the same way with (\ref{B-eq-12}), we estimate that
\begin{equation}\label{B-eq-14}
\begin{split}
    &
    d G (\cdot, f_{n})_{\psi_n}(\psi_{n+1}-\psi_{n})
    \\
    &
    =
    \langle 
    \nabla \cG (\cdot, f_n, \Phi(\psi_{n}, f_n) )_{\psi_n}, 
    \psi_{n+1} - \psi_{n} \rangle_{\Spsi}
    \\
    &
    =
    - \alpha_{\psi, n} \|\nabla \cG (\cdot, f_n, \Phi(\psi_{n}, f_n) )_{\psi_n}\|^{2}_{\Spsi}
    \\
    &
    \hspace{0.4cm}
    +
    \langle 
    \nabla \cG (\cdot, f_n, \Phi(\psi_{n}, f_n) )_{\psi_n}, 
    \psi_{n+1} - \psi_{n} + 
    \alpha_{\psi,n}
    \nabla \cG (\cdot, f_n, \Phi(\psi_{n}, f_n) )_{\psi_n} \rangle_{\Svar}
    \\
    &
    \leq
    - \alpha_{\psi, n} \|\nabla G (\cdot, f_n)_{\psi_n}\|^{2}_{\Spsi}
    \\
    &
    \hspace{0.4cm}
    +
    L
    \|\nabla G (\cdot, f_n)_{\psi_n}\|_{\Spsi} 
    \| \varphi_{n} - \Phi(\psi_n, f_n ) \|_{\Svar}
    \\
    &
    \leq
    \left\{-1 + \frac{L\alpha_{\psi,n}}{2} \right\} \alpha_{\psi, n} \|\nabla G (\cdot, f_n)_{\psi_n}\|^{2}_{\Spsi}
    +
    \frac{L}{2}
    \| \varphi_{n} - \Phi(\psi_n, f_n ) \|_{\Svar}.
\end{split}
\end{equation}
Thus, by combining (\ref{B-eq-11}), (\ref{B-eq-12}), (\ref{B-eq-13}), and (\ref{B-eq-14}), we get
\begin{equation}\label{B-eq-15}
\begin{split}
    &
    G(\psi_{n+1}, f_{n+1})
    \leq 
    G(\psi_{n}, f_{n})
    \\
    &
    \hspace{0.4cm}
    +
    \left\{ -1 + \frac{L\alpha_{\psi,n}}{2} + L (1+L)(1+L_{\beta})\alpha_{\psi,n} + L_{\beta} \alpha_{\psi,n} \right\}
    \alpha_{\psi,n} 
    \| \nabla G (\cdot,f_n )_{\psi_n} \|_{\Spsi}^2
    \\
    &
    \hspace{0.4cm}
    +
    \left\{
    -1 
    + \frac{L\alpha_{f,n}}{2}
    + \frac{L (1+L_{\beta})\alpha_{f,n}}{2} 
    + L_{\beta}\alpha_{f,n}
    \right\}
    \alpha_{f,n}
    \| \nabla G (\psi_{n},\cdot )_{f_n} \|_{\Sf}^2 
    \\
    &
    \hspace{0.4cm}
    +
    \underbrace{
    \left\{
    L 
    + L^2 L_{\beta}\alpha_{\psi,n}^2 
    + L^3 (1+L)(1+L_{\beta}) \alpha_{\psi,n}^2
    + L^2 L_{\beta}\alpha_{f,n}^2
    + \frac{L^2 (1+L_{\beta})\alpha_{f,n}^2}{2} 
    \right\} 
    }_{\text{Assumption~\ref{step-size-nonconvex} (ii)} \leq C}
    \\
    &
    \hspace{9cm}
    \times
    \| \Phi(\psi_{n}, f_n)  -\varphi_n \|_{\Svar}^2.
\end{split}
\end{equation}

Third, we estimate $\| \Phi(\psi_{n}, f_n)  -\varphi_n \|_{\Svar}^2=:\delta_n$.
Using Lemma~\ref{B-lemma-2} and Assumption~\ref{step-size-nonconvex} (i), we evaluate that
\begin{align*}
    &
    \| \varphi_{n} - \Phi(f_{n-1}, \psi_{n-1}) \|_{\Svar}^2
    \\
    &
    \leq
    \| \varphi_{n-1} - \Phi(f_{n-1}, \psi_{n-1}) \|_{\Svar}^2
    + 2 \langle 
    \varphi_{n-1} - \Phi(f_{n-1}, \psi_{n-1}), 
    \varphi_{n} - \varphi_{n-1} \rangle_{\Svar}
    + \| \varphi_{n} - \varphi_{n-1} \|_{\Svar}^2
    \\
    &
    \leq
    (1-\beta \alpha_{\varphi, n-1}) \| \varphi_{n} - \Phi(f_{n-1}, \psi_{n-1}) \|_{\Svar}^2
    +
    (-1 + L \alpha_{\varphi, n-1}) \| \varphi_{n} - \varphi_{n-1} \|_{\Svar}^2
    \\
    &
    \leq
    (1-\beta \alpha_{\varphi, n-1}) \delta_{n-1},
\end{align*}
which implies that by using Young's inequality, Lemma~\ref{B-lemma-1}, (\ref{B-eq-5}), and (\ref{B-eq-11}), we have
\begin{equation}
\begin{split}
    &
    \delta_n = \| \Phi(\psi_{n}, f_n)  -\varphi_n \|_{\Svar}^2
    \\
    &
    \leq 
    (1+\beta \alpha_{\varphi, n-1})\| \Phi(\psi_{n-1}, f_{n-1})  - \varphi_n \|_{\Svar}^2
    \\
    &
    \hspace{2mm}
    +
    2 \left(1 + \frac{1}{\beta \alpha_{\varphi, n-1}} \right)
    \left(
    \| \Phi(\psi_{n-1}, f_{n})  - \Phi(\psi_{n}, f_n) \|_{\Svar}^2
    +
    \| \Phi(\psi_{n-1}, f_{n})  - \Phi(\psi_{n-1}, f_{n-1}) \|_{\Svar}^2
    \right)
    \\
    &
    \leq 
    (1-\beta^2 \alpha_{\varphi, n-1}^2) \delta_{n-1} 
    +\frac{2 L^2}{\beta^2} \left(1 + \frac{1}{\beta C_0} \right)
    \left(
    \| \psi_{n-1} - \psi_{n} \|_{\Svar}^2 
    + \| f_{n-1} - f_{n} \|_{\Sf}^2 
    \right)
    \\ 
    &
    \leq 
    \left\{
    (1-\beta^2 \alpha_{\varphi, n-1}^2) 
    + \frac{2 L^2}{\beta^2} \left(1 + \frac{1}{\beta C_0} \right)(\alpha_{\psi,n-1}^2+\alpha_{f,n-1}^2)
    \right\} \delta_{n-1}
    \\
    &
    \hspace{0.4cm}
    + \frac{2 L^2}{\beta^2} \left(1 + \frac{1}{\beta C_0} \right)
    \left(
    \alpha_{\psi,n-1}^2 \| \nabla G (\cdot,f_{n-1} )_{\psi_{n-1}} \|_{\Spsi}^2
    +
    \alpha_{f,n-1}^2 \| \nabla G (\psi_{n-1},\cdot )_{f_{n-1}} \|_{\Sf}^2
    \right)
    \\
    &
    \leq 
    \gamma \delta_{n-1}
    + \frac{2 L^2}{\beta^2} \left(1 + \frac{1}{\beta C_0} \right)
    \left(
    \alpha_{\psi,n-1}^2 \| \nabla G (\cdot,f_{n-1} )_{\psi_{n-1}} \|_{\Spsi}^2
    +
    \alpha_{f,n-1}^2 \| \nabla G (\psi_{n-1},\cdot )_{f_{n-1}} \|_{\Sf}^2
    \right),
\end{split}
\end{equation}
where we have employed Assumption~\ref{step-size-nonconvex} (iii) for last inequality.
Then, we have
\begin{align*}
    \delta_n
    &
    \leq 
    \gamma^n \delta_0
    +
    \frac{2 L^2}{\beta^2} \left(1 + \frac{1}{\beta C_0} \right)
    \sum_{i=0}^{n}
    \left(
    \alpha_{\psi,i}^2 \gamma^{n-i} \| \nabla G (\cdot,f_{i} )_{\psi_{i}} \|_{\Spsi}^2
    +
    \alpha_{f,i}^2 \gamma^{n-i} \| \nabla G (\psi_{i}, \cdot)_{f_{i}} \|_{\Sf}^2
    \right).
\end{align*}
By this and (\ref{B-eq-15}), we have
\begin{align*}
    &
    \left\{ 1 
    - \frac{L\alpha_{\psi,n}}{2} 
    - L (1+L)(1+L_{\beta})\alpha_{\psi,n} 
    - L_{\beta} \alpha_{\psi,n} \right\}\alpha_{\psi,n} 
    \| \nabla G (\cdot,f_n )_{\psi_n} \|_{\Spsi}^2
    \\
    &
    +
    \left\{
    1 
    - \frac{L\alpha_{f,n}}{2}
    - \frac{L (1+L_{\beta})\alpha_{f,n}}{2} 
    - L_{\beta}\alpha_{f,n}
    \right\}
    \alpha_{f,n}
    \| \nabla G (\psi_{n},\cdot )_{f_n} \|_{\Sf}^2 
    \\
    &
    \leq 
    G(\psi_{n}, f_{n}) - G(\psi_{n+1}, f_{n+1})
    +
    C \gamma^n \delta_0
    \\
    &
    +
    \frac{2 L^2 C}{\beta^2} \left(1 + \frac{1}{\beta C_0} \right)
    \sum_{i=0}^{n}
    \left(
    \alpha_{\psi,i}^2 \gamma^{n-i} \| \nabla G (\cdot,f_{i} )_{\psi_{i}} \|_{\Spsi}^2
    +
    \alpha_{f,i}^2 \gamma^{n-i} \| \nabla G (\psi_{i}, \cdot)_{f_{i}} \|_{\Sf}^2
    \right),
\end{align*}
and taking the summation over the interval $n=0$ to $N-1$,
\begin{align*}
    &
    \sum_{n=0}^{N-1}
    \left\{ 1 
    - \frac{L\alpha_{\psi,n}}{2} 
    - L (1+L)(1+L_{\beta})\alpha_{\psi,n} 
    - L_{\beta} \alpha_{\psi,n} \right\}\alpha_{\psi,n} 
    \| \nabla G (\cdot,f_n )_{\psi_n} \|_{\Spsi}^2
    \\
    &
    +
    \sum_{n=0}^{N-1}
    \left\{
    1 
    - \frac{L\alpha_{f,n}}{2}
    - \frac{L (1+L_{\beta})\alpha_{f,n}}{2} 
    - L_{\beta}\alpha_{f,n}
    \right\}
    \alpha_{f,n}
    \| \nabla G (\psi_{n},\cdot )_{f_n} \|_{\Sf}^2 
    \\
    &
    \leq 
    G(\psi_{0}, f_{0}) - G(\psi_{N}, f_{N})
    +
    C \delta_0 \sum_{n=0}^{N-1}\gamma^n 
    \\
    &
    +
    \frac{2 L^2 C}{\beta^2} \left(1 + \frac{1}{\beta C_0} \right)
    \sum_{n=0}^{N-1}
    \sum_{i=0}^{n}
    \left(
    \alpha_{\psi,i}^2 \gamma^{n-i} \| \nabla G (\cdot,f_{i} )_{\psi_{i}} \|_{\Spsi}^2
    +
    \alpha_{f,i}^2 \gamma^{n-i} \| \nabla G (\psi_{i}, \cdot)_{f_{i}} \|_{\Sf}^2
    \right),
    \\
    &
    \leq 
    G(\psi_{0}, f_{0}) - \inf_{\psi, f} G(\psi, f)
    +
    C \delta_0 \sum_{n=0}^{\infty}\gamma^n 
    \\
    &
    +
    \frac{2 L^2 C}{\beta^2} \left(1 + \frac{1}{\beta C_0} \right)
    \left(\sum_{i=0}^{\infty}\gamma^i \right)
    \sum_{n=0}^{N-1}
    \left(
    \alpha_{\psi,n}^2 \| \nabla G (\cdot,f_{n} )_{\psi_{n}} \|_{\Spsi}^2
    +
    \alpha_{f,n}^2 \| \nabla G (\psi_{n}, \cdot)_{f_{n}} \|_{\Sf}^2
    \right),
\end{align*}
which is equivalent to 
\begin{align*}
    &
    \sum_{n=0}^{N-1}
    \underbrace{
    \left\{ 1 
    - \frac{L\alpha_{\psi,n}}{2} 
    - L (1+L)(1+L_{\beta})\alpha_{\psi,n} 
    - L_{\beta} \alpha_{\psi,n} 
    - \frac{2 L^2 C \alpha_{\psi,n}}{\beta^2(1-\gamma)} \left(1 + \frac{1}{\beta C_0} \right)
    \right\}
    }_{\text{Assumption~\ref{step-size-nonconvex} (iv) } \geq C_{\psi}>0 }
    \\
    &
    \hspace{10cm}
    \times
    \alpha_{\psi,n} 
    \| \nabla G (\cdot,f_n )_{\psi_n} \|_{\Spsi}^2
    \\
    &
    +
    \sum_{n=0}^{N-1}
    \underbrace{
    \left\{
    1 
    - \frac{L\alpha_{f,n}}{2}
    - \frac{L (1+L_{\beta})\alpha_{f,n}}{2} 
    - L_{\beta}\alpha_{f,n}
    - \frac{2 L^2 C \alpha_{f,n}}{\beta^2(1-\gamma)} \left(1 + \frac{1}{\beta C_0} \right)
    \right\}
    }_{\text{Assumption~\ref{step-size-nonconvex} (v) } \geq C_{f}>0}
    \\
    &
    \hspace{10cm} \times
    \alpha_{f,n}
    \| \nabla G (\psi_{n},\cdot )_{f_n} \|_{\Sf}^2 
    \\
    &
    \leq 
    G(\psi_{0}, f_{0}) - \inf_{\psi, f} G(\psi, f)
    +
    \frac{C \delta_0}{1-\gamma}.
\end{align*}
Finally, we estimate that 
\begin{align*}
    &
    \left\| \widehat{\nabla G}_{\psi, N} \right\|_{\Spsi}
    =
    \left\|
    \frac{\sum_{n=0}^{N-1}\alpha_{\psi,n}\nabla G(f_{n}, \cdot)_{\psi_n} }{\sum_{n=0}^{N-1}\alpha_{\psi,n}}
    \right\|_{\Spsi}
    \\
    &
    \leq
    \frac{\sum_{n=0}^{N-1}\alpha_{\psi,n}\left\|\nabla G(f_{n}, \cdot)_{\psi_n}\right\|_{\Spsi}}
    {\sum_{n=0}^{N-1}\alpha_{\psi,n}}
    \leq
    \frac{\left(\sum_{n=0}^{N-1}\alpha_{\psi,n}\left\|\nabla G(f_{n}, \cdot)_{\psi_n}\right\|^2_{\Spsi}\right)^{1/2}}
    {\left(\sum_{n=0}^{N-1}\alpha_{\psi,n}\right)^{1/2}}.
\end{align*}
By the same way, we estimate that
\begin{align*}
    \left\| \widehat{\nabla G}_{f, N} \right\|_{\Sf}
    \leq
    \frac{\left(\sum_{n=0}^{N-1}\alpha_{f,n}\left\|\nabla G(\psi_n, \cdot)_{f_n}\right\|^2_{\Sf}\right)^{1/2}}
    {\left(\sum_{n=0}^{N-1}\alpha_{f,n}\right)^{1/2}}.
\end{align*}
Therefore, we conclude Theorem~\ref{main-theorem-convergence-nonconvex}.

\end{proof}

%%%%%%%%%%%%%%%%%%%%%%%%%%%%%%%%%%%%%%%%%%%%%%
\section{Proofs in Section~\ref{verifi-Convex-concave}}
\label{Appendix2}
%%%%%%%%%%%%%%%%%%%%%%%%%%%%%%%%%%%%%%%%%%%%%%%%%%%%%%%%%%%

\subsection{Proof of Proposition~\ref{example-joint-convex}}
\label{example-joint-convex-app}
\begin{proof}
For $\alpha \in [0,1]$, $\psi_1, \psi_2 \in \mathcal{C}(X)$, and $\mu_1, \mu_2 \in \mathcal{M}(X)$,
\[
\begin{split}
&
R(\alpha \psi_1 + (1-\alpha)\psi_2, \alpha \mu_1 + (1-\alpha)\mu_2)
\\
&
\leq \alpha^2 \int \ell(\psi_{1}, \psi_{0}) d\mu_1 
+
(1-\alpha)^2 
\int \ell(\psi_{2}, \psi_{0}) d\mu_2
\\
&
+\alpha (1-\alpha) 
\left(
\int \ell(\psi_1, \psi_{0}) d\mu_2 + \int \ell(\psi_2, \psi_{0}) d\mu_1
\right)
\\
&
+\alpha V(\psi_1) + (1-\alpha)V(\psi_2) - \frac{\alpha (1-\alpha)\gamma}{2} \left\| \psi_1 -\psi_2 \right\|_{\mathcal{C}(X),2}^2
\\
&
+\alpha W(\mu_1) + (1-\alpha)W(\mu_2) - \frac{\alpha (1-\alpha)\gamma}{2} \left\| \mu_1 -\mu_2 \right\|_{\mathcal{C}(X), 1}^{\ast 2}
\\
&
= \alpha 
\underbrace{\left(\int \ell(\psi_1, \psi_{0}) d\mu_1 + V(\psi_1) + W(\mu_1) \right)}_{=R(\psi_1, \mu_1)}
+ (1-\alpha) \underbrace{\left(\int  \ell(\psi_2, \psi_{0}) d\mu_2 + V(\mu_2) + W(\mu_2) \right)}_{=R(\psi_2, \mu_2)}
\\&
+\alpha (1-\alpha) 
\underbrace{
\left( 
- \int ( \ell(\psi_1, \psi_{0}) -  \ell(\psi_2, \psi_{0}) )d(\mu_1-\mu_2) - \frac{\gamma}{2} \left\| \psi_1 -\psi_2 \right\|_{\mathcal{C}(X),2}^2
- \frac{\gamma}{2} \left\| \mu_1 -\mu_2 \right\|_{\mathcal{C}(X),1}^{\ast 2}
\right)
}_{=(\ast)},
\end{split}
\]
and $(\ast)$ is non-positive because we have
\[
\begin{split}
(\ast)&
\leq
\underbrace{
\left\| \ell(\psi_1, \psi_{0}) - \ell(\psi_2, \psi_{0}) \right\|_{\mathcal{C}(X),1}
}_{\leq \rho \left\| \psi_1 - \psi_2 \right\|_{\mathcal{C}(X),2}
\leq \gamma \left\| \psi_1 - \psi_2 \right\|_{\mathcal{C}(X),2}
}
\left\| \mu_1 -\mu_2 \right\|_{\mathcal{C}(X),1}^{\star}
\\
&\quad
-  \frac{\gamma}{2} \left\| \psi_1 -\psi_2 \right\|_{\mathcal{C}(X),2}^2
- \frac{\gamma}{2} \left\| \mu_1 -\mu_2 \right\|_{\mathcal{C}(X),1}^{\ast 2}
\\
&
\leq 
-\frac{\gamma}{2}\left( \left\| \psi_1 - \psi_2 \right\|_{\mathcal{C}(X),2}-\left\| \mu_1 -\mu_2 \right\|_{\mathcal{C}(X),1}^{\star} \right)^{2} \leq 0.
\end{split}
\]
\end{proof}

\subsection{Proof of Lemma~\ref{H-smooth-lemma}}
\label{H-smooth-lemma-app}
%%%%%%%%%%%%%%%%%%%%%%%%%%%%%%%%%%%%%%%%%%%%
\begin{proof}
For $\psi, \varphi \in S_{\mathcal{C},a,b}$,
\[
\begin{split}
D_{I_{h,\mu}}(\psi|\varphi)
&=
I_{h,\mu}(\psi) - I_{h,\mu}(\varphi) -d(I_{h,\mu})_{\varphi}(\psi-\varphi)
\\
&
= 
\int 
\left(h(\psi)-h(\varphi)-h^{\prime}(\varphi)(\psi-\varphi)
\right) d\mu 
\\
&
\leq 
\frac{L}{2}
\int 
\left| \psi - \varphi \right|^2
d\mu 
= \| \psi - \varphi \|^{2}_{L^2(X, \mu)} .
\end{split}
\]
\end{proof}
%%%%%%%%%%%%%%%%%%%%%%%%%%%%%%%%%%%%%%%%

\subsection{Proof of Lemma~\ref{$f$-divergence-conjugate}}
\label{$f$-divergence-conjugate-app}
%%%%%%%%%%%%%%%%%%%%%%%%%%%%%%%%%%%%%%%%%%%%%%%%
\begin{proof}
By the definition of $f$-divergence~(\ref{$f$-divergence}), we have
\[
\begin{split}
J_{f}^{\star}(\varphi)
=\sup_{\mu \in \mathcal{M}(X)} \int \varphi d\mu - J_{f}(\mu)
=\sup_{\mu \ll \nu_{0}} \int \varphi d\mu - \int f\left(\frac{d\mu}{d\nu_{0}} \right)d\nu_{0}.
\end{split}
\]
We solve a concave maximization problem for $\mu \mapsto \int \varphi d\mu - \int f\left(\frac{d\mu}{d\nu_{0}} \right)d\nu_{0}$.
We consider
\[
\frac{d}{d\epsilon}\left(
\int \varphi d(\mu+\epsilon \chi) - \int f\left(\frac{d\mu+\epsilon \chi}{d\nu_{0}} \right)d\nu_{0}
\right) \Biggr|_{\epsilon = 0}=0,
\]
which is equivalent to
\[
\int \varphi d \chi - \int f^{\prime}\left(\frac{d\mu}{d\nu_{0}} \right) \frac{d\chi}{d\nu_{0}} d\nu_{0} = 
\int \left( \varphi - f^{\prime}\left(\frac{d\mu}{d\nu_{0}} \right) \right) d\chi =0,
\]
for all $\chi$.
Then, the optimal $\mu$ satisfies
\[
\varphi = f^{\prime}\left(\frac{d\mu}{d\nu} \right).
\]
By the assumption, $f^{\prime}$ is invertible, and  $\varphi \in S_{\mathcal{C},f}$.
Substituting $\frac{d\mu}{d\nu_{0}}=(f^{\prime})^{-1}(\varphi)$ into
\[
J_{f}^{\star}(\varphi)
=\sup_{\mu} \int \varphi\frac{d\mu}{d\nu_{0}}d\nu_{0}  - \int f\left(\frac{d\mu}{d\nu_{0}} \right)d\nu_{0},
\]
then, we obtain that
\[
J_{f}^{\star}(\varphi)
=
\int \left\{\varphi \cdot (f^{\prime})^{-1}(\varphi) - f\circ (f^{\prime})^{-1}(\varphi) \right\} d\nu_{0}.
\]
\end{proof}

\subsection{Proof of Lemma~\ref{IPMs-conjugate}}
\label{IPMs-conjugate-app}
\begin{proof}
By the definition of $\IPM$~(\ref{IMP-def-def}), we have
\[
J_{\IPM, \nu_0}(\mu)
=\sup_{\varphi \in \mathcal{C}(X)}
\int \varphi d\mu -\int \varphi d\nu_{0} - \chi\{ \varphi \in \cF \}. 
\]
By the Fenchel-Moreau theorem, we obtain that
\[
J_{\IPM, \nu_0}^{\star}(\varphi)= \int \varphi d \nu_{0} + \chi\{ \varphi \in \cF \}.
\]
\end{proof}
%%%%%%%%%%%%%%%%%%%%%%%%%%%%%%%%%%%%%

\subsection{Proof of Proposition~\ref{all-ex-prop}}
\label{proof-all-ex}
\begin{proof}
(2) holds due to the convexity of $k(\cdot)$.
(4) follows from the linearity of $\mu \mapsto \mathcal{K}_1(\psi, \mu, \varphi)$ and the norm $\frac{\gamma}{2} \left\| \cdot \right\|_{\mathcal{H}_{\sigma}}^{\star}$ induced by inner products.
%(6) holds by using the convexity of $\left\| \cdot \right\|_{L^2(X, \mu_u)}$ and the strong convexity of $\left\| \cdot \right\|_{\mathcal{H}_{2\sqrt{2}\sigma}}$.
%(7) holds because the norm $\left\| \cdot \right\|_{\mathcal{H}_{\sigma}}$ is induced by inner products.

For (3), it holds that
\[
\begin{split}
D_{\mathcal{K}_1(\cdot, \mu, \varphi)}(\psi_1|\psi_2)
&
=
\frac{1}{2} \int (\psi_1 - \psi_2)^2 d\mu
+
\frac{\gamma}{2} \left\|\psi_1-\psi_2 \right\|_{\mathcal{H}_{2\sqrt{2}\sigma }}^{2},
\end{split}
\] 
and by $\mu \in \Smu$, we have
\[
\begin{split}
\frac{1}{2} \int (\psi_1 - \psi_2)^2 d\mu
&
\leq 
\frac{1}{2} \int (\psi_1 - \psi_2)^2 d\mu_{u}
=
\frac{1}{2} \left\| \psi_1 - \psi_2 \right\|_{L^2(X, \mu_u)}^{2}.
%=
%D_{\frac{1}{2} \left\| \cdot \right\|_{L^2(X, \mu_u)}^{2}}(\psi_1 | \psi_2).
\end{split}
\]
%
% \medskip
For (5), we estimate by using the $L_k$-smoothness of $k:(a,b) \to \mathbb{R}$ and $\nu_0 \in \Smu$
\[
\begin{split}
&
D_{-\mathcal{K}_1(\psi, \mu, \cdot)}(\varphi_1|\varphi_2)
=
D_{\int k(\cdot) d\nu_0 }(\varphi_1|\varphi_2)
\\
&
=
\int \left\{ k(\varphi_1)- k(\varphi_2) - k^{\prime}(\varphi_2) (\varphi_1 - \varphi_2) \right\} d\nu_0
\\
&
\leq
\frac{L_k}{2} \int |\varphi_1 - \varphi_2|^2 d\mu_{u}
\\
&
=
\frac{L_k}{2} \|\varphi_1 - \varphi_2 \|_{L^2(X, \mu_u)}^2.
\end{split}
\] 
%
%\medskip
Finally, we will prove (1).
To apply Proposition~\ref{example-joint-convex} as $\frac{1}{2}\ell(\cdot, \psi_0)=(\cdot - \psi_0)^2$, $V(\psi)=\frac{\gamma}{2}\left\|\psi\right\|^2_{\mathcal{H}_{2\sqrt{2}\sigma}}$, $W(\mu)=\frac{\gamma}{2}\left\|\mu \right\|^{\star 2}_{\mathcal{H}_{\sigma}}$ , we verify assumptions~(i), (ii), and (iii) in Proposition~\ref{example-joint-convex}.
(i) holds due to the convexity of $t \mapsto (t-s)^{2}$.
(iii) holds because  norms $ \left\| \cdot \right\|_{\mathcal{H}_{2\sqrt{2}\sigma}}$ and $\left\| \cdot \right\|_{\mathcal{H}_{\sigma}}^{\star}$ are induced by inner products.
We prove (ii) as followings:

By \citet[Proposition 14]{chu2020smoothness}, the RKSH norm $\left\|f \right\|_{\mathcal{H}_{\sigma}}$ is represented as 
\begin{equation}
\left\| f \right\|_{\mathcal{H}_{\sigma}}^2
=
\sum_{k=0}^{\infty}(\frac{1}{2}\sigma^2)^{k} 
\sum_{|\alpha|=k}
\frac{1}{\alpha!}
\left\| \partial^{\alpha}_{x} f \right\|_{L^2(\mathbb{R}^d)}^2,
\label{estimate-1-all-ex}
\end{equation}
for $f \in \mathcal{H_{\sigma}}$.
%{\color{red}
%Here, we employ the multi-index notation where the $d$-dimensional multi-index $\alpha=(\alpha_1,...,\alpha_d) \in \mathbb{N}_{0}^{d}$,  sum of components $|\alpha|=\alpha_1 + \cdots + \alpha_d$, factorial $\alpha ! = \alpha_1 ! \cdots \alpha_d !$, and partial derivative $\partial_{x}^{\alpha}=\partial_{x_1}^{\alpha_1}\cdots \partial_{x_d}^{\alpha_d}$.
%}
%{\color{blue}
Here, we employ the multi-index notation with $d$-dimensional multi-index $\alpha=(\alpha_1,...,\alpha_d) \in \mathbb{N}_{0}^{d}$ where
the sum of its components denotes the $|\alpha|=\alpha_1 + \cdots + \alpha_d$.
Additionally, we define the factorial of the multi-index as $\alpha ! = \alpha_1 ! \cdots \alpha_d !$, and the partial derivative as $\partial_{x}^{\alpha}=\partial_{x_1}^{\alpha_1}\cdots \partial_{x_d}^{\alpha_d}$.
%}
\begin{comment}
, binomial coefficient 
$\begin{pmatrix}
\alpha \\
\beta
\end{pmatrix}
=\frac{\alpha!}{\beta!(\alpha-\beta)!}
$   
\end{comment}

We estimate that
\begin{equation}
\begin{split}
&
\left\|\partial^{\alpha}_{x}
\left( (\psi_1 + \psi_2 - 2 \psi_0 )
(\psi_1 - \psi_2 ) \right) \right\|_{L^2(\mathbb{R}^d)}^{2}
\\
&
=
\left\|
\sum_{\beta \leq \alpha}
\begin{pmatrix}
\alpha \\
\beta
\end{pmatrix}
\partial^{\alpha-\beta}_{x}
(\psi_1 + \psi_2 - 2 \psi_0 )
\partial^{\beta}_{x}
(\psi_1 -\psi_2)
\right\|_{L^2(\mathbb{R}^d)}^{2}
\\
&
\leq 
\left(
\sum_{\beta \leq \alpha}
\begin{pmatrix}
\alpha \\
\beta
\end{pmatrix}
\underbrace{
\left\|
\partial^{\alpha-\beta}_{x}
(\psi_1 + \psi_2 - 2 \psi_0 )
\right\|_{L^\infty(\mathbb{R}^d)}
}_{\leq 4 C_{b}}
\left\|
\partial^{\beta}_{x}
(\psi_1 - \psi_2 ) 
\right\|_{L^2(\mathbb{R}^d)}
\right)^{2}
\\
&
\leq
(4 C_{b})^2 
\underbrace{
\left(
\sum_{\beta \leq \alpha}
\begin{pmatrix}
\alpha \\
\beta
\end{pmatrix}
\right)^{2}
}_{= (2^k)^2}
\times 
\underbrace{
\left(
\sum_{\beta \leq \alpha}
\left\|
\partial^{\beta}_{x}
(\psi_1 - \psi_2 ) 
\right\|_{L^2(\mathbb{R}^d)}
\right)^{2}
}_{
%\leq 
%\left(
%\sum_{\beta \leq \alpha}
%\right)
%\left(
%\sum_{\beta \leq \alpha}
%\left\|
%\partial^{\beta}_{x}
%(\psi_1 - \psi_2 ) 
%\right\|_{L^2(\mathbb{R}^d)}^{2}
%\right)
\leq 
2^k
\sum_{\beta \leq \alpha}
\left\|
\partial^{\beta}_{x}
(\psi_1 - \psi_2 ) 
\right\|_{L^2(\mathbb{R}^d)}^{2}
}
\\
&
\leq
16 C_{b}^2 
8^k
\sum_{\beta \leq \alpha}
\left\|
\partial^{\beta}_{x}
(\psi_1 - \psi_2 ) 
\right\|_{L^2(\mathbb{R}^d)}^{2},
\end{split}
\label{estimate-2-all-ex}
\end{equation}
where the first equality employs the Leibniz formula, the second inequality utilizes the Cauchy–Schwarz inequality and the result of $\psi_1, \psi_2,\psi_0 \in \Spsi$, and the third inequality makes use of the Cauchy–Schwarz inequality and multi-binomial theorem.

By using (\ref{estimate-1-all-ex}) and (\ref{estimate-2-all-ex}), we further estimate that
\begin{align*}
&
\left\| \frac{1}{2} (\psi_1 - \psi_0 )^2
- \frac{1}{2} (\psi_2 - \psi_0 )^2
\right\|_{\mathcal{H}_{\sigma}}^2\\
&=
\frac{1}{4}
\left\| 
(\psi_1 + \psi_2 - 2 \psi_0 )
(\psi_1 - \psi_2 )
\right\|_{\mathcal{H}_{\sigma}}^2
\\
&
=
4 C_{b}^2 
\sum_{k=0}^{\infty}(4\sigma^2)^{k} 
\sum_{|\alpha|=k}
\frac{1}{\alpha!}
\sum_{\beta \leq \alpha}
\left\|
\partial^{\beta}_{x}
(\psi_1 - \psi_2 ) 
\right\|_{L^2(\mathbb{R}^d)}^{2}
\\
&
=
4 C_{b}^2 
\sum_{k=0}^{\infty}
\sum_{|\alpha|=k}
\left[
\sum_{\beta \geq \alpha}
(4\sigma^2)^{|\beta|} 
\frac{1}{\beta!}
\right]
\left\|
\partial^{\alpha}_{x}
(\psi_1 - \psi_2 ) 
\right\|_{L^2(\mathbb{R}^d)}^{2}
\\
&
\leq
4 C_{b}^2 
\sum_{k=0}^{\infty}
(4 \sigma^2)^{k} 
\sum_{|\alpha|=k}
\frac{1}{\alpha!}
\left[
\sum_{\beta \geq \alpha}
(4\sigma^2)^{|\beta|-k} 
\right]
\left\|
\partial^{\alpha}_{x}
(\psi_1 - \psi_2)
\right\|_{L^2(\mathbb{R}^d)}^{2}
\\
&
\leq
4 C_{b}^2 
\sum_{k=0}^{\infty}
(\frac{1}{2} (2 \sqrt{2} \sigma)^2 )^{k} 
\sum_{|\alpha|=k}
\frac{1}{\alpha!}
\underbrace{
\left[
\sum_{\beta_d \geq \alpha_d}
\cdots 
\sum_{\beta_1 \geq \alpha_1}
(4\sigma^2)^{\beta_d-\alpha_d} 
\cdots 
(4\sigma^2)^{\beta_1-\alpha_1} 
\right]
}_{=
\left[
\sum_{j \geq k}
(4\sigma^2)^{j-k} 
\right]^d
=C_{\sigma}^d
}
\left\|
\partial^{\alpha}_{x}
(\psi_1 - \psi_2)
\right\|_{L^2(\mathbb{R}^d)}^{2}
\\
&
\leq
4 C_{b}^2 C_{\sigma}^d
\sum_{k=0}^{\infty}
(\frac{1}{2} (2 \sqrt{2} \sigma)^2 )^{k} 
\sum_{|\alpha|=k}
\frac{1}{\alpha!}
\left\|
\partial^{\alpha}_{x}
(\psi_1 - \psi_2)
\right\|_{L^2(\mathbb{R}^d)}^{2}
\\
&
=4 C_{b}^2 C_{\sigma}^d 
\left\| (\psi_1 - \psi_0 )
\right\|_{\mathcal{H}_{2\sqrt{2}\sigma}}^{2},
\end{align*}
where $C_{\sigma}=\sum_{j \in \mathbb{N}_{0}}
(4\sigma^2)^{j} < \infty$, which implies that
$
    \psi \mapsto \frac{1}{2}(\psi-\psi_0)^2
$
is $4 C_{b}^2 C_{\sigma}^d$-Lipschitz with respect to 
$
    \left\| \cdot\right\|_{\mathcal{H}_{\sigma}}
$ 
and 
$
    \left\| \cdot \right\|_{\mathcal{H}_{2\sqrt{2}\sigma}}
$.
Thus, by the assumption of $\gamma \geq 4 C_{b}^2 C_{\sigma}^d$ and applying Proposition~\ref{example-joint-convex} to our setting, we conclude that $(\psi, \mu) \mapsto \mathcal{K}_{1}(\psi, \mu, \varphi)$ is convex.
\end{proof}

\section{Proofs in Section~\ref{verifi-Nonconvex-concave}}
\label{Appendix3}
\subsection{Proof of Lemma~\ref{strong-convex-disc}}
\label{strong-convex-disc-app}
\begin{proof}
Since $\mathcal{I}: \cM(X) \to \mathbb{R}$ is $(1/\beta)$-smooth with respect to $\| \cdot \|_{\cM(X)}$, the convex conjugate $\mathcal{I}^{\star}: \cC(X) \to \mathbb{R}$ is $\beta$-strongly convex with respect to $\| \cdot \|_{\cM(X)}^{\star}$.
By this and
\[
(J_{\nu_0} \oplus \mathcal{I})^{\star}=J_{\nu_0}^{\star} + \mathcal{I}^{\star},
\]
then $(J_{\nu_0} \oplus \mathcal{I})^{\star}$ is $\beta$-strongly convex with respect to $\| \cdot \|_{\cM(X)}^{\star}$ due to the fact that strong convexity is preserved by adding convex functions.
\end{proof}
%%%%%%%%%%%%%%%%%%%%%%%%%%%%%%%%%
%
\subsection{Proof of Lemma~\ref{Gateaux-diff-f}}
\label{Gateaux-diff-f-app}
\begin{proof}
As $h$ is Lipschitz continuous, $h$ is absolutely continuous.
Thus, the derivative $\nabla h$ of $h$ is defined a.e. in $X$ with respect to the Lebesgue measure $m$. 
By the assumption $\mu<<m$, the derivative $\nabla h$ is also defined a.e. in $X$ with respect to probability measure $\xi$, which implies that
\begin{align*}
    \frac{d}{d\epsilon}
    \mathcal{J}_{h,\xi}(f+\epsilon g)
    \Bigr|_{\epsilon=0}
    = 
    \int_Z \nabla h(f(z)) \cdot g(z) \mu(dz).
\end{align*}
\end{proof}
%%%%%%%%%%%%%%%%%%%%%%%%%%%%%%
%
%
%
%
%
%
\subsection{Proof of Proposition~\ref{all-ex-nonconvex-prop}}
\label{all-ex-nonconvex-prop-app}
\begin{proof}%[Proof of Proposition~\ref{all-ex-nonconvex-prop}]
(1) is given by Lemma~\ref{strong-convex-disc}.
(2) holds from Lemma~\ref{Gateaux-diff-f}, and G\^{a}teaux differentials are given by

\begin{align*}
    &
    d \cG_1(\cdot, f, \varphi)_{\psi} (\eta)
    = 2 \int (\psi \circ f - \psi_0 \circ f) \eta \circ f d\xi_0,
    \\
    &
    d \cG_1(\psi, \cdot, \varphi)_{f} (g)
    =
    2 \int (\psi \circ f - \psi_0 \circ f) \{(\nabla \psi - \nabla \psi_0) \circ f \}\cdot g d\xi_0
    +
    \int \{ \nabla \varphi \circ f \} \cdot g d \xi_0,
    \\
    &
    d \cG_1(\psi, f, \cdot)_{\varphi} (\phi)
    =
    \int \phi \circ f d \xi_0 
    - \int k^{\prime}(\varphi) \cdot \phi d \nu_0 
    - \beta \langle \phi , \varphi \rangle_{\mathcal{H}_{\sigma}}.
\end{align*}

We will confirm Assumption~\ref{L-smooth-nonconvex} as follows:

(a): for $\psi_1, \psi_2 \in \Spsic$, $f\in \Spsic$, $\varphi \in \Svarc$, and $\eta \in \Spsi$ with $\| \eta \|_{H^{1}(X)}\leq 1$, 
\begin{align*}
    d \cG_1(\cdot, f, \varphi)_{\psi_1} (\eta)
    -
    d \cG_1(\cdot, f, \varphi)_{\psi_2} (\eta)
    &
    =2 \int (\psi_1 \circ f - \psi_2 \circ f)  \eta \circ f d\xi_0
    \\
    &
    =2 \int (\psi_1 - \psi_2) \eta d(f_{\sharp}\xi_0)
    \\
    &
    \leq 2 C_3 \|\psi_1 - \psi_2\|_{L^2(X)}\|\eta\|_{L^2(X)}
    \\
    &
    \leq 2 C_3 \|\psi_1 - \psi_2\|_{H^1(X)},
\end{align*}
\begin{align*}
    \Rightarrow
    \|d \cG_1(\cdot, f, \varphi)_{\psi_1} 
    -
    d \cG_1(\cdot, f, \varphi)_{\psi} \|_{H^{1}(X)}^{\star}
    \leq 2 C_3 \|\psi_1 - \psi_2\|_{H^1(X)}.
\end{align*}

(b): for $\psi \in \Spsic$, $f_1, f_2 \in \Sfc$, $\varphi \in \Svarc$, and $\eta \in \Spsi$ with $\| \eta \|_{H^{1}(X)}\leq 1$, 
\begin{align*}
    &
    d \cG_1(\cdot, f_1, \varphi)_{\psi} (\eta)
    -
    d \cG_1(\cdot, f_2, \varphi)_{\psi} (\eta)
    \\
    &
    =2 \int (\psi \circ f_1 - \psi_0 \circ f_1)  \eta \circ f_1 d\xi_0
    -
    2 \int (\psi \circ f_2 - \psi_0 \circ f_2)  \eta \circ f_2 d\xi_0
    \\
    &
    \leq 
    2 \int |\psi \circ f_1 - \psi_0 \circ f_1| |\eta \circ f_1 -\eta \circ f_2 | d\xi_0
    \\
    &
    + 
    2 \int |\psi \circ f_1 - \psi \circ f_2| |\eta \circ f_2 | d\xi_0
    + 
    2 \int |\psi_0 \circ f_1 - \psi_0 \circ f_2| |\eta \circ f_2 | d\xi_0
    \\
    &
    \leq 
    4 C_2 \mathrm{Lip}(\eta) \int | f_1 - f_2 | d\xi_0
    +
    4 C_1 \int | f_1 - f_2 | |\eta \circ f_2| d\xi_0
    \\
    &
    \leq 
    4 C_2 \|\nabla \eta \|_{L^{\infty}(X)} |Z|^{1/2} \| f_1 - f_2 \|_{L^2(Z;X, \xi_0)}
    +
    4 C_1 C_3^{1/2} \|\eta\|_{L^2(X)} \| f_1 - f_2 \|_{L^2(Z;X, \xi_0)}
    \\
    &
    \leq 
    \left(4 C_2 \widetilde{C}_{d, X} |Z|^{1/2}+ 4 C_1 C_3^{1/2} \right) \| f_1 - f_2 \|_{L^2(Z;X, \xi_0)},
\end{align*}
\begin{align*}
    \Rightarrow
    \|d \cG_1(\cdot, f_1, \varphi)_{\psi} 
    -
    d \cG_1(\cdot, f_2, \varphi)_{\psi} \|_{H^{1}(X)}^{\star}
    \leq \left(4 C_2 \widetilde{C}_{d, X} |Z|^{1/2}+ 4 C_1 C_3^{1/2} \right)
    \| f_1 - f_2 \|_{L^2(Z;X, \xi_0)},
\end{align*}
where the last inequality results from Lemma~\ref{Lip-L2-lemma} where $\widetilde{C}_{d, X}>0$ is some constant depending on $d$ and $X$.

(c): for $\psi \in \Spsic$, $f\in \Spsic$, $\varphi_1, \varphi_2 \in \Svarc$, and $\eta \in \Spsi$ with $\| \eta \|_{H^{1}(X)}\leq 1$, 
\begin{align*}
    d \cG_1(\cdot, f, \varphi_1)_{\psi} (\eta)
    -
    d \cG_1(\cdot, f, \varphi_2)_{\psi} (\eta)
    &
    =0
\end{align*}
\begin{align*}
    \Rightarrow
    \|d \cG_1(\cdot, f, \varphi_1)_{\psi} 
    -
    d \cG_1(\cdot, f, \varphi_2)_{\psi} \|_{H^{1}(X)}^{\star}
    \leq c \|\varphi_1 - \varphi_2\|_{\mathcal{H}_{\sigma}},
\end{align*}
for any $c>0$.

(d): for $\psi \in \Spsic$, $f_1, f_2 \in \Sfc$, $\varphi \in \Svarc$, and $g \in \Sf$ with $\| g \|_{L^2(Z;X, \xi_0)}\leq 1$, 
\begin{align*}
    &
    d \cG_1(\psi, \cdot, \varphi)_{f_1} (g)
    -
    d \cG_1(\psi, \cdot, \varphi)_{f_2} (g)
    \\
    &
    =
    2 \int (\psi \circ f_1 - \psi_0 \circ f_1) \{(\nabla \psi - \nabla \psi_0) \circ f_1 \}\cdot g d\xi_0
    + \int \{ \nabla \varphi \circ f_1 \} \cdot g d \xi_0
    \\
    &
    - 2 \int (\psi \circ f_2 - \psi_0 \circ f_2) \{(\nabla \psi - \nabla \psi_0) \circ f_2 \}\cdot g d\xi_0
    - \int \{ \nabla \varphi \circ f_2 \} \cdot g d \xi_0
    \\
    &
    \leq 
    4 C_2 
    \int \left| (\nabla \psi - \nabla \psi_0) \circ f_1 -(\nabla \psi - \nabla \psi_0) \circ f_2  \right| |g| d\xi_0
    \\
    &
    + 
    4 C_2 
    \int \left\{ |\psi \circ f_1 - \psi \circ f_2| + |\psi_0 \circ f_1 - \psi_0 \circ f_2| \right\} |g| d\xi_0
    + \int |\nabla \varphi\circ f_1 - \nabla \varphi\circ f_2||g| d \xi_0,
    \\
    &
    \leq 
    (16 C_1 C_2 +  C_4) 
    \int \left| f_1 - f_2  \right| |g| d\xi_0
    \\
    &
    \leq 
    (16 C_1 C_2 +  C_4)  \| f_1 - f_2 \|_{L^2(Z;X, \xi_0)},
\end{align*}
\begin{align*}
    \Rightarrow
    \|d \cG_1(\psi, \cdot, \varphi)_{f_1}
    -
    d \cG_1(\psi, \cdot, \varphi)_{f_2} \|_{L^2(Z;X, \xi_0)}^{\star}
    \leq (16 C_1 C_2 +  C_4)
    \| f_1 - f_2 \|_{L^2(Z;X, \xi_0)}.
\end{align*}

(e): for $\psi_1, \psi_2 \in \Spsic$, $f \in \Sfc$, $\varphi \in \Svarc$, and $g \in \Sf$ with $\| g \|_{L^2(Z;X, \xi_0)}\leq 1$, 
\begin{align*}
    &
    d \cG_1(\psi_1, \cdot, \varphi)_{f} (g)
    -
    d \cG_1(\psi_2, \cdot, \varphi)_{f} (g)
    \\
    &
    =
    2 \int (\psi_1 \circ f - \psi_0 \circ f) \{(\nabla \psi_1 - \nabla \psi_0) \circ f \}\cdot g d\xi_0
    \\
    &
    - 2 \int (\psi_2 \circ f - \psi_0 \circ f) \{(\nabla \psi_2 - \nabla \psi_0) \circ f \}\cdot g d\xi_0
    \\
    &
    \leq 
    2 
    \int \left| \psi_1 \circ f - \psi_0 \circ f \right| 
    \left| \nabla \psi_1 \circ f - \nabla \psi_2 \circ f \right| |g| d\xi_0
    \\
    &
    + 
    2 
    \int \left| \psi_1 \circ f - \psi_2 \circ f \right| 
    \left| \nabla \psi_2 \circ f - \nabla \psi_0 \circ f \right| |g| d\xi_0
    \\
    &
    \leq 
    4 C_2
    \int  
    \left| \nabla \psi_1 \circ f - \nabla \psi_2 \circ f \right| |g| d\xi_0
    + 
    4 C_2 
    \int \left| \psi_1 \circ f - \psi_2 \circ f \right| 
    |g| d\xi_0
    \\
    &
    \leq 
    4 C_2 C_3^{1/2}
    \| \psi_1 - \psi_2 \|_{H^1(X)},
\end{align*}
\begin{align*}
    \Rightarrow
    \|d \cG_1(\psi_1, \cdot, \varphi)_{f} (g)
    -
    d \cG_1(\psi_2, \cdot, \varphi)_{f} (g) \|_{L^2(Z;X, \xi_0)}^{\star}
    \leq 4 C_2 C_3^{1/2}
    \| \psi_1 - \psi_2 \|_{H^1(X)}.
\end{align*}

(f): for $\psi \in \Spsic$, $f \in \Sfc$, $\varphi_1, \varphi_2 \in \Svarc$, and $g \in \Sf$ with $\| g \|_{L^2(Z;X, \xi_0)}\leq 1$, 
\begin{align*}
    &
    d \cG_1(\psi, \cdot, \varphi_1)_{f} (g)
    -
    d \cG_1(\psi, \cdot, \varphi_2)_{f} (g)
    \\
    &
    = \int (\nabla \varphi_1 \circ f) \cdot g - (\nabla \varphi_2 \circ f) \cdot g d \xi_0 
    \\
    &
    \leq 
    C_3^{1/2} \| \nabla \varphi_1 - \nabla \varphi_2 \|_{L^2(X)} \|g\|_{L^2(Z;X, \xi_0)}
    \\
    &
    \leq 
    C_3^{1/2} \widetilde{C}_{\sigma} \| \varphi_1 - \varphi_2 \|_{\mathcal{H}_{\sigma}} 
\end{align*}
\begin{align*}
    \Rightarrow
    \|d \cG_1(\psi, \cdot, \varphi_1)_{f} 
    -
    d \cG_1(\psi, \cdot, \varphi_2)_{f}
    \|_{L^2(Z;X, \xi_0)}^{\star}
    \leq C_3^{1/2} \widetilde{C}_{\sigma}
    \| \varphi_1 - \varphi_2 \|_{\mathcal{H}_{\sigma}},
\end{align*}
where the last inequality follows from (\ref{estimate-1-all-ex}) where $\widetilde{C}_{\sigma}>0$ is some constant depending on $\sigma$.

(g): for $\psi \in \Spsic$, $f\in \Spsic$, $\varphi_1, \varphi_2 \in \Svarc$, and $\phi \in \Svar$ with $\| \phi \|_{\mathcal{H}_{\sigma}}\leq 1$, 
\begin{align*}
    &
    d \cG_1(\psi, f, \cdot)_{\varphi_1} (\phi)
    -
    d \cG_1(\psi, f, \cdot)_{\varphi_2} (\phi)
    \\
    &
    = 
    \int \left\{ k^{\prime} (\varphi_2) - k^{\prime} (\varphi_1) \right\} \cdot \phi d \nu_0  
    +\beta \langle \phi, \varphi_2 - \varphi_1  \rangle_{\mathcal{H}_{\sigma}}
    \\
    &
    \leq
    L_k  \sup_{x \in X} \left|\frac{d \nu_0}{dm}(x) \right|^{1/2} 
    \| \varphi_1 -\varphi_2 \|_{L^2(X)}
    +\beta \|\varphi_1 - \varphi_2\|_{\mathcal{H}_{\sigma}}
    \\
    &
    \leq
    \left(
    L_k  \sup_{x \in X} \left|\frac{d \nu_0}{dm}(x) \right|^{1/2} 
    +\beta 
    \right)
    \|\varphi_1 - \varphi_2\|_{\mathcal{H}_{\sigma}},
\end{align*}
\begin{align*}
    \Rightarrow
    \|d \cG_1(\psi, f, \cdot)_{\varphi_1} 
    -
    d \cG_1(\psi, f, \cdot)_{\varphi_2} \|_{\mathcal{H}_{\sigma}}^{\star}
    \leq \left(
    L_k  \sup_{x \in X} \left|\frac{d \nu_0}{dm} \right|^{1/2} 
    +\beta 
    \right) \|\varphi_1 - \varphi_2\|_{\mathcal{H}_{\sigma}}.
\end{align*}

(h): for $\psi_1, \psi_2 \in \Spsic$, $f\in \Spsic$, $\varphi \in \Svarc$, and $\phi \in \Svar$ with $\| \phi \|_{\mathcal{H}_{\sigma}}\leq 1$, 
\begin{align*}
    d \cG_1(\psi_1, f, \cdot)_{\varphi} (\phi)
    - 
    d \cG_1(\psi_2, f, \cdot)_{\varphi} (\phi)
    & 
    =0,
\end{align*}
\begin{align*}
    \Rightarrow
    \|d \cG_1(\psi_1, f, \cdot)_{\varphi}
    - 
    d \cG_1(\psi_2, f, \cdot)_{\varphi} \|_{H^{1}(X)}^{\star}
    \leq c \|\psi_1 - \psi_2\|_{H^{1}(X)},
\end{align*}
for any $c>0$.

(i): for $\psi \in \Spsic$, $f_1, f_2 \in \Spsic$, $\varphi \in \Svarc$, and $\phi \in \Svar$ with $\| \phi \|_{\mathcal{H}_{\sigma}}\leq 1$, 
\begin{align*}
    d \cG_1(\psi, f_1, \cdot)_{\varphi} (\phi)
    -
    d \cG_1(\psi, f_2, \cdot)_{\varphi} (\phi)
    &
    = \int \phi \circ f_1 - \phi \circ f_2 d \xi_0
    \\
    &
    \leq 
    \mathrm{Lip}(\phi) \int | f_1 - f_2 | d \xi_0
    \\
    &
    \leq
    \widetilde{C}_{\sigma, d, X} \xi_0(Z)^{1/2} \|f_1 -f_2\|_{L^2(Z; X, \xi_0)}
\end{align*}
\begin{align*}
    \Rightarrow
    \|d \cG_1(\psi, f_1, \cdot)_{\varphi}
    -
    d \cG_1(\psi, f_2, \cdot)_{\varphi} \|_{\mathcal{H}_{\sigma}}^{\star}
    \leq \widetilde{C}_{\sigma, d, X} \xi_0(Z)^{1/2} \|f_1 -f_2\|_{L^2(Z; X, \xi_0)},
\end{align*}
where the last inequality results from Lemma~\ref{Lip-L2-lemma} where $\widetilde{C}_{\sigma, d, X}>0$ is some constant depending on $\sigma$, $d$, and $X$.
\end{proof}
We have employed the following fundamental Lemma in the proof of Proposition~\ref{all-ex-nonconvex-prop}.
\begin{lemma}\label{Lip-L2-lemma}
Let $\Omega \subset \mathbb{R}^d$ be a compact set, and let $f:\Omega \to \mathbb{R}$ be Lipschitz continuous.
Then, we have 
\begin{align*}
    \|f\|_{L^{\infty}(\Omega)}
    \leq
    \max
    \left\{
    \frac{(d+1)|\Omega|^{1/2}}{2^d}\|f\|_{L^2(\Omega)},    \left(\frac{(d+1)d^{d/2}|\Omega|^{1/2}}{2^d}\right)^{\frac{1}{d+1}}
   \|f\|_{L^{2}(\Omega)}^{\frac{1}{d+1}}
    \right\}.
\end{align*}
\end{lemma}
\begin{proof}
Assume that $f:\Omega \to \mathbb{R}$ be $L$-Lipschitz with $L\geq 1$. 
We denote by
$$
a:= \mathrm{argmax}_{x \in \Omega}f(x),
$$
and 
$$
M:=\|f\|_{L^{\infty}(\Omega)}=\max_{x \in \Omega}f(x).
$$
Without loss of generality, we can assume that $M>0$.

As $f:\Omega \to \mathbb{R}$ is $L$-Lipschitz, that is,
\begin{align*}
    |f(x)-f(a)|\leq L|x-a|, \ x \in \Omega,
\end{align*}
we estimate that
\begin{align*}
    L \int_{0}^{M/L}1_{
    \{|x-a|\leq v\}}dv
    &
    \leq
    L\left( \frac{M}{L} - |x-a| \right)
    \\
    &
    \leq 
    -L|x-a|+M
    \leq f(x),
\end{align*}
which implies  that
\begin{align*}
    \int_{\Omega}|f(x)|dx
    &
    \geq 
    L \int_{\Omega}\int_{0}^{M/L}1_{
    \{|x-a|\leq v\}}dvdx
    \\
    &
    =
    L \int_{0}^{M/L}
    \int_{\Omega}1_{
    \{|x-a|\leq v\}}dx
    dv
    \\
    &
    \geq
    L \int_{0}^{M/L}
    \prod_{i=1}^{d}
    \int_{-R+a_i}^{R+a_i}
    1_{
    \{|x_i-a_i|\leq v/\sqrt{d}\}}dx
    dv
    \\
    &
    =L \int_{0}^{M/L} Q_{R,d}(v)dv,
\end{align*}
where 
\begin{align*}
   Q_{R,d}(v): = \left\{
   \begin{array}{ll}
   \left( \frac{2v}{\sqrt{d}} \right)^d & \frac{v}{\sqrt{d}} < R \\
   \left( 2R \right)^d & \frac{v}{\sqrt{d}} \geq R. 
   \end{array}
   \right.
\end{align*}
Here, $R>1$ is chosen large enough such that
\begin{align*}
    \Omega \subset \prod_{i=1}^{d}[-R+a_i, R+a_i].
\end{align*}
By direct computation, we can show that
\begin{align*}
    &
    \int_{\Omega}|f(x)|dx
    \\
    &
    \geq
    \left\{
   \begin{array}{ll}
   2^d R^d \left(M - RL \sqrt{d} \frac{d}{d+1}\right) & \frac{M}{L} >\sqrt{d} R \\
   2^d R^d \frac{\sqrt{d}}{d+1} LR^{d+1} & \frac{M}{L} \leq \sqrt{d} R
   \end{array}
   \right.
   \\
    &
    \geq
    \left\{
   \begin{array}{ll}
   \frac{2^d}{d+1}\|f\|_{L^{\infty}(\Omega)} & \frac{M}{L} >\sqrt{d} R \\
   \frac{2^d }{(d+1)d^{d/2}}
   \|f\|_{L^{\infty}(\Omega)}^{d+1}
   & \frac{M}{L} \leq \sqrt{d} R.
   \end{array}
   \right.
\end{align*}
Therefore, we conclude that
\begin{align*}
    &
    \|f\|_{L^{\infty}(\Omega)}
    \\
    &
    \leq
    \max
    \left\{
    \frac{d+1}{2^d}\|f\|_{L^1(\Omega)},
    \left(\frac{(d+1)d^{d/2}}{2^d}\right)^{\frac{1}{d+1}}
   \|f\|_{L^{1}(\Omega)}^{\frac{1}{d+1}}
    \right\}
    \\
    &
    \leq
    \max
    \left\{
    \frac{(d+1)|\Omega|^{1/2}}{2^d}\|f\|_{L^2(\Omega)},
    \left(\frac{(d+1)d^{d/2}|\Omega|^{1/2}}{2^d}\right)^{\frac{1}{d+1}}
   \|f\|_{L^{2}(\Omega)}^{\frac{1}{d+1}}
    \right\}.
\end{align*}
\end{proof}

\end{document}